\documentclass{article}
\usepackage[preprint]{preprint}
\usepackage{times}
\usepackage[T1]{fontenc}
\usepackage[hyphens]{url}
\usepackage{graphicx}
\usepackage{caption}
\usepackage{tabularx}
\usepackage{placeins}
\usepackage{hyperref}
\hypersetup{hidelinks}

\usepackage{amsmath}
\usepackage{amssymb}
\usepackage{amsthm}
\usepackage{thm-restate}
\usepackage{booktabs}
\usepackage{array}
\usepackage{listings}
\usepackage{tikz}
\usetikzlibrary{arrows.meta}
\newsavebox{\OptABAExampleDAG}

\newif\ifOptABAIncludeSupplement
\ifdefined\OptABAMainOnly
  \OptABAIncludeSupplementfalse
\else
  \OptABAIncludeSupplementtrue
\fi

\newcommand{\OptABASuppListingsRef}[2]{%
  \ifOptABAIncludeSupplement Listings~\ref{#1}--\ref{#2} in Appendix~\ref{supp:caba-core-encoding}\else the appendix section ``Causal ABA Encodings''\fi}
\newcommand{\OptABAIfSupplement}[1]{\ifOptABAIncludeSupplement#1\fi}
\newcommand{\OptABAIfSupplementElse}[2]{\ifOptABAIncludeSupplement#1\else#2\fi}

\lstdefinestyle{numberedasp}{
    numbers=left,
    numberstyle=\tiny,
    stepnumber=1,
    numbersep=4pt,
    xleftmargin=1.5em,
    captionpos=b,
}
\lstnewenvironment{aspencoding}[1][]
    {\lstset{style=numberedasp,#1}}
    {}

\newtheorem{example}{Example}

\newtheorem{definition}{Definition}

\newtheorem{supplemma}{Lemma}

\newtheorem{suppproposition}{Proposition}

\newcommand{\indep}{\!\perp \!\!\!\! \perp\!}
\newcommand{\notindep}{\not\!\perp \!\!\!\! \perp\!}
\newcommand{\given}{ \mid }
\newcommand{\enforce}{\mathsf{enforce}}
\newcommand{\adorned}{\mathsf{opt}}

\newcommand{\ASP}[1]{\mathrm{#1}}
\newcommand{\aspP}{\ASP{P}}
\newcommand{\aspO}{\ASP{O}}
\newcommand{\aspS}{\ASP{S}}

\newcommand{\aspR}{\ASP{R}}
\newcommand{\aspW}{\ASP{W}}

\title{Optimal Correction Sets for Provable\\Redress of Contestable Causal Discovery}
\author{Fabrizio Russo$^{1}$ \quad Giuseppe Mazzotta$^{2}$ \quad Carmine Dodaro$^{2}$\\
\textbf{Francesco Ricca$^{2}$ \quad Francesca Toni$^{1}$}\\[0.35em]
\normalfont $^{1}$Imperial College London, London, United Kingdom\\
\normalfont $^{2}$University of Calabria, Rende, Italy\\[0.25em]
\normalfont\texttt{\{fabrizio,ft\}@imperial.ac.uk}\\
\normalfont\texttt{\{giuseppe.mazzotta,carmine.dodaro,francesco.ricca\}@unical.it}}
\hypersetup{pdftitle={Optimal Correction Sets for Provable Redress of Contestable Causal Discovery},
  pdfauthor={Fabrizio Russo, Giuseppe Mazzotta, Carmine Dodaro, Francesco Ricca, Francesca Toni}}
\begin{document}

\maketitle

\begin{abstract}
Constraint-based causal discovery relies on conditional-independence (CI) tests that may conflict in finite samples, leaving no causal graph that satisfies them all. Causal Assumption-based Argumentation (ABA) makes these tests explicit as defeasible commitments, and offers correspondence guarantees with a subset of them, but its existing repair procedure releases facts heuristically and offers no formal response to a user's contestation of a release. We argue that contestable causal discovery requires an inspectable repair object, a justification for it, and a predictable revision procedure for such contestations. We therefore formulate CI-fact repair as a minimum-cost correction problem and implement it in Answer Set Programming (ASP). Building on Causal ABA, we provide a compact and efficient ASP encoding of causal consistency, with soundness and completeness guarantees, and a two-way projection correspondence between optimal stable models and minimum-cost correction sets. These results establish fidelity to retained facts, subset-minimality and cost optimality of the repair, conflict-based justifications for released facts, and formally guaranteed redress under user contestations. An empirical evaluation supports the formal account: OptABA-PC retains more correct facts, reduces ambiguity in the compatible outputs, and improves graph reconstruction over Causal ABA on common benchmarks.
\end{abstract}

\section{Introduction}

Causal discovery supports scientific understanding and decision making by turning observational and/or interventional data into hypotheses about the underlying causal mechanisms. It is a central ingredient in tasks such as explanation, prediction under interventions, and policy evaluation~\citep{pearl2009causality}. A wide range of approaches exists, including constraint-based methods relying on independence testing, score-based methods that optimise a goodness-of-fit criterion over graphs, and functional causal model approaches that leverage assumptions on the data-generating mechanism; see e.g.~\citep{glymour2019review}.

Constraint-based algorithms such as PC~\citep{spirtes2000causation} infer structure from conditional-independence (CI) information. Causal ABA~\citep{russo2024abacd} instead treats selected CI facts as defeasible commitments that every returned graph must satisfy if enforced, exposing data-derived premises and a point of contestation. In this context, a \emph{contestation} asks the system to reconsider a result by revising those input facts. Computational argumentation supports this view: contestation should lead beyond static explanation to reliable revision when redress is warranted~\citep{leofante2024contestable}; argumentative decision making captures that reliability through contestability guarantees under user-driven changes~\citep{freedman2025argllm}.

Finite-sample CI tests may conflict, making Causal ABA incoherent. Causal ABA restores coherence by fixed strength-ordered release, revealing discarded facts but offering neither optimality nor a guaranteed response to a contested release or cost. We make the correction set the \emph{redress} object and prove optimal repair. Our claims concern \emph{CI-fact selection}: redress guarantees revised commitments and consistency with them, without guaranteeing recovery of the generating graph.
We focus on the realistic regime where finite-sample CI tests are noisy and potentially conflicting. Majority-PC (MPC)~\citep{colombo2014order}, an improved PC variant for finite samples, learns a partially directed graph from CI tests. Under the standard assumptions it is sound with an oracle for error-free CI judgements, but finite-sample outcomes may conflict.

\begin{example}[Causal discovery]\label{ex:cd}
\sbox{\OptABAExampleDAG}{%
\begin{minipage}[t]{0.30\linewidth}
\normalfont
\captionsetup{type=figure,justification=centering}
\centering
\begin{tikzpicture}[xscale=1.15,yscale=1.15]
\tikzset{cnode/.style={draw,circle,inner sep=2pt,minimum size=0.5cm}}
\node[cnode] (E) at (0,0.5) {$E$};
\node[cnode] (R) at (0,-0.5) {$R$};
\node[cnode] (O) at (1.1,0) {$O$};
\node[cnode] (I) at (2.3,0) {$I$};
\path[->,thick,>=Stealth]
 (E) edge (O)
 (R) edge (O)
 (O) edge (I)
 (E) edge[out=0,in=155] (I);
\end{tikzpicture}
\caption{Example DAG.}\label{fig:gt-dag}
\end{minipage}%
}
\parshape 8
  0pt 0.68\linewidth 0pt 0.68\linewidth 0pt 0.68\linewidth
  0pt 0.68\linewidth 0pt 0.68\linewidth 0pt 0.68\linewidth
  0pt 0.68\linewidth 0pt \linewidth
\leavevmode\vadjust pre{\nobreak\vbox to 0pt{%
  \hbox to \linewidth{\hfill\usebox{\OptABAExampleDAG}}\vss}}%
Consider a simple socioeconomic setting with Education~($E$), Race~($R$), Occupation~($O$) and Income~($I$) with ground-truth directed acyclic graph (DAG) in Fig.~\ref{fig:gt-dag}. When fed data generated from this DAG, MPC runs $14$ CI tests (see Ex.~\ref{ex:abapc}) and uses them sequentially to remove and orient edges. With finite-sample noise, it may report \emph{mutually incompatible} outcomes. For instance, it includes the high-strength independences $E\indep O \given \{I\}$ and $O\indep I,$ even though it also reports dependencies such as $E\notindep O$ and $E\notindep I.$ These four statements already cannot be simultaneously satisfied by any DAG: if $E$ and $O$ are dependent but become independent after conditioning on $I$, then $I$ must block every active $E$--$O$ path, which implies $O\notindep I$---contradicting $O\indep I$. Nevertheless, MPC returns a non-empty graph (the single undirected edge $O\!-\!R$), yet neither identifies the incompatible constraints nor records which it effectively discards.
\end{example}

We build on work on explanation and repair in logic programming~\citep{Reiter87Diagnosis,Liffiton08MUS,ALVIANO2023103931}, particularly minimal unsatisfiable subsets (MUSes) and minimal correction subsets (MCSes). A MUS isolates an irreducible conflict that explains incoherence; an MCS is a minimal set of rules or facts whose removal restores coherence. In Causal ABA, MUSes explain which CI-test outcomes conflict, while MCSes expose the minimal releases available as redress when a user contests a result.
Operationally, weighted CI facts yield retained and released facts, compatible DAGs, and a conflict certificate for each release; a user can then require a fact or raise its release cost and re-solve. Concretely, we: (i) formulate Causal ABA's released CI facts as an explicit repair; (ii) prove a two-way correspondence between optimal MCSes and optimal stable models; (iii) prove soundness and completeness of a compact Bayes-ball encoding~\citep{shachter1998bayesball}; (iv) establish conflict certificates and guaranteed hard and graded contestability; (v) implement the optimisation as OptABA-PC; and (vi) compare it with Causal ABA and standard baselines, generally improving fact selection and/or graph reconstruction while providing redress guarantees.
\vspace{-0.2cm}
\section{Related Work}\label{sec:related}
\vspace{-0.2cm}
The relationship between inconsistency, repairs, and hitting sets is well established. In model-based diagnosis, diagnoses are hitting sets of conflicts~\citep{Reiter87Diagnosis}; in constraint systems and SAT, MUSes expose irreducible conflicts and MCSes repair them by hitting those conflicts~\citep{Liffiton08MUS}. Related ideas appear in optimisation as maximum feasible subsystem and IIS-cover problems, where one deletes a minimum-weight set of constraints to restore feasibility~\citep{ChinneckDravnieks1991IIS,AmaldiPfetschTrotter2003MaxFS}, and in implicit hitting-set methods, which alternate between a feasibility oracle returning an unhit conflict and an optimiser selecting a new hitting set~\citep{MorenoCentenoKarp2013IHS}. Core-guided MaxSAT solvers operationalise the same pattern over soft clauses~\citep{DBLP:journals/constraints/MorgadoHLPM13}. Our use of MUS/MCS terminology therefore does not introduce a new optimisation paradigm; rather, it makes explicit what the repair object is in Causal ABA.

In causal discovery, solver-assisted methods have used SAT to enforce constraints from overlapping observational or interventional data and query graph features common to all solutions~\citep{triantafillou2010learning,triantafillou2015COMBINE}, weighted MaxSAT to infer ancestral relations that maximise satisfied evidence~\citep{magliacane2016ancestral}, and weighted constraint optimisation to obtain robust representative graphs~\citep{borboudakis2016towards}. These methods optimise causal conclusions or graph-level agreement; they do not make the correction set of released facts the explicit output.

SAT/ASP has also been used to encode $d$-separation exactly. \citet{Hyttinen13SAT} introduce variables for directed and bidirected edges and bounded $d$-connecting walks, then compute backbone features shared by all satisfying cyclic/latent causal graphs; inconsistent inputs yield no repair. \citet{hyttinen2014ASP} handle noisy inputs by deriving \texttt{fail} atoms when a graph violates a CI constraint and minimising their total weight. They study hard-dependence, constant, and log-derived weighting schemes; we call the DAG-restricted empirical implementation of this objective ASPCR-DAG. Our proposed OptABA-PC likewise uses weighted optimisation, but its CI-strength weights are release costs, its released facts form an explicit correction set, and every retained fact is satisfied by every returned graph. The encoding of \citet{hyttinen2014ASP} instead permits penalised \texttt{fail} atoms in an optimal model. Compared to ASPCR, Causal ABA targets DAGs, which permits the compact Bayes-ball encoding while preserving an argumentative reading of retained facts.

\citet{Hyttinen2017core} recast the same weighted graph-disagreement objective as MaxSAT: weighted unit soft clauses represent CI constraints, hard clauses connect them to graph satisfaction, a SAT solver extracts unsatisfiable cores, and an integer program computes a minimum-cost hitting set of the accumulated cores. This is the closest prior optimisation architecture to OptABA-PC. Its cores and hitting sets already provide conflict/repair machinery; our additional results connect explicit released-fact sets to optimal stable models and establish fact enforcement and contestability guarantees. \citet{DBLP:journals/ijar/RantanenHJ20} use branch-and-bound with partial-graph propagation, tailored branching, and core-derived bounds.\footnote{No public implementation was identified for these latter methods, so we compare to \citep{hyttinen2014ASP}.}
Complementary SAT/ASP work weakens assumptions or extends the graph semantics~\citep{Zhalama2017SAT,Zhang2019ASP}; OptABA-PC keeps Causal~ABA's DAG target and correspondence guarantees (that ASPCR lacks), but adds optimality and contestability.

In Causal ABA, accepted CI facts are not merely soft clauses or post-hoc \texttt{fail} atoms: they are enforceable ABA/ASP facts whose satisfaction is guaranteed in every returned stable model, yielding traceability from data-derived facts to graphs. OptABA-PC makes the released facts first-class: by placing weak constraints on unselected objective atoms, optimal stable models correspond to optimal correction sets of the Causal ABA program. Whereas existing ASP work on MUSes studies subset-minimality, enumeration, and cautious reasoning~\citep{ALVIANO2023103931}, we connect optimal stable models to minimum-cost correction sets and establish contestability guarantees.
Quantitative Argumentation for Causal Discovery (QACD)~\citep{Wei2026QACD} also treats noisy CI evidence as graded arguments. It iteratively revises edge acceptability through gradual semantics on abstract argumentation frameworks, with no correspondence guarantees between CIs and output DAGs. Instead, our method identifies the optimal release of CI facts, with correspondence and redress guarantees.
\vspace{-0.2cm}
\section{Preliminaries}\label{sec:preliminaries}
\vspace{-0.2cm}
We consider a finite set of variables $\mathbf{V}$. We write $X\indep Y \given \mathbf{Z}$ for conditional independence between variables $X$ and $Y$ given conditioning set $\mathbf{Z}\subseteq \mathbf{V}\setminus\{X,Y\}$, and $X\notindep Y \given \mathbf{Z}$ for dependence ($X\indep Y$ and $X\notindep Y$ if $\mathbf{Z}=\emptyset$). In our causal-discovery setting, a CI fact is a ground assertion of one such relation, obtained from a statistical CI test for a threshold $\alpha$.
\vspace{-0.15cm}

A causal graph is a DAG $G=(\mathbf{V},E)$ whose edges represent direct causal relations. A node $Z$ on a path is a \emph{collider} on that path if the two adjacent path edges both have arrowheads into $Z$, as in $X\to Z\leftarrow Y$; otherwise it is a \emph{non-collider}. A node $Y$ is a descendant of $X$ if there is a directed path from $X$ to $Y$. A collider is \emph{open} given a conditioning set $\mathbf{Z}$ when the collider itself or at least one of its descendants belongs to $\mathbf{Z}$; otherwise it is closed. DAGs entail CI relations through $d$-separation~\citep{pearl1988probabilisticreasoning}: a path between $X$ and $Y$ is active given $\mathbf{Z}$ if every non-collider on the path is outside $\mathbf{Z}$ and every collider is open given $\mathbf{Z}$; $X$ and $Y$ are $d$-connected given $\mathbf{Z}$ if some such active path exists, and $d$-separated otherwise.
Observational data identify at most a Markov equivalence class (MEC), represented by a \emph{completed partially} DAG (CPDAG) whose directed edges are shared by all DAGs in the class and whose undirected edges are reversible~\citep{chickering2002learning}.
\vspace{-0.25cm}
\subsection{Answer Set Programming}
\vspace{-0.25cm}
Answer Set Programming (ASP) represents a problem by rules whose stable models are its solutions. A ground normal rule has form $h\texttt{ :- }l_1,\ldots,l_n$, where each body literal is an atom $a$ or default negation $\texttt{not }a$, which holds when $a$ is absent; a headless rule forbids its body. We write $AS(P)$ for the stable models of $P$ and call $P$ \emph{coherent} when $AS(P)\neq\emptyset$. \OptABAIfSupplementElse{App.~\ref{supp:asp}}{The appendix section ``Detailed ASP Semantics''} gives the reduct-based semantics~\citep{gel88}.
We use choice rules $\{a\}$ to allow selecting $a$ and weak constraints $\texttt{:}\sim B.\ [w@l,o]$ to charge cost $w$ when $B$ holds~\citep{DBLP:journals/tplp/CalimeriFGIKKLM20}. The identifier $o$ distinguishes equal-weight facts. With priority $l$ set to 1 for all constraints, an optimal stable model simply minimises the total cost.

\vspace{-0.25cm}
\subsection{ABA and Causal ABA}\label{sec:aba-prelim}
\vspace{-0.25cm}
An assumption-based argumentation framework (ABAF) is a tuple $(\mathcal{L},\mathcal{R},\mathcal{A},\overline{\cdot})$, where $\mathcal{L}$ is a language, $\mathcal{R}$ is a set of rules over $\mathcal{L}$, $\mathcal{A}\subseteq\mathcal{L}$ is a set of assumptions, and $\overline{\cdot}$ maps each assumption $a\in\mathcal{A}$ to its contrary $\overline{a}\in\mathcal{L}$~\citep{DBLP:journals/ai/BondarenkoDKT97,CyrasFST2018}. A set $S\subseteq\mathcal{A}$ derives $q$, written $S\vdash q$, if $q$ follows from $S$ by the rules. It attacks an assumption $a$ when $S'\vdash \overline{a}$ for some $S'\subseteq S$. A stable extension is a conflict-free set of assumptions that attacks every assumption outside it.
This formalism is useful for causal discovery because assumptions are defeasible commitments: they may be accepted or rejected based on their interrelations.

Causal ABA~\citep{russo2024abacd} uses assumptions for candidate graph choices (directed edges) and CI facts, and rules encoding acyclicity and $d$-separation. A stable extension of Causal ABA corresponds to a candidate DAG together with the CI facts it accepts. Enforcing a CI fact means requiring it to be in the accepted set and rejecting any stable extension whose graph violates it; releasing a fact means allowing it to be absent so coherence can be restored. Under the standard correspondence between ABA under stable semantics and ASP stable models~\citep{RapbergerUT24}, this ABAF can be solved as an ASP program. Hence Causal ABA inherits both the argumentation reading of assumptions and the computational machinery of ASP.
Causal ABA may admit multiple stable extensions for the same enforced CI facts, with corresponding DAGs spanning one or more MECs. This multiplicity is informative: the CI--DAG correspondence guarantees that every candidate satisfies those facts, while differing graphs expose the structural ambiguity they leave unresolved.
These sets describe uncertainty conditional on the enforced facts, without implying calibrated confidence or guaranteed coverage of the generating MEC.
\vspace{-0.2cm}
\subsection{Causal ABA-PC}
\vspace{-0.2cm}
Example~\ref{ex:cd} highlights the situation Causal ABA is designed to expose: the empirical CI facts may be mutually incompatible, yet we still want returned graphs to satisfy a subset of the facts that remain enforced and to record the ones released.

\begin{example}[Causal ABA]\label{ex:encodingpreds}
Continuing Ex.~\ref{ex:cd}, Causal ABA's grounded ASP encoding would include:
\begin{lstlisting}
collider(o,e,r)  :- arrow(e,o), arrow(r,o), e!=r.
nb(o,e,r,Z)      :- in(o,Z), collider(o,e,r).
ap(e,r,p_eor,Z)  :- arrow(e,o), arrow(r,o), nb(o,e,r,Z).
ap(e,r,p_eior,Z) :- arrow(e,i), arrow(i,o), arrow(r,o), nb(o,e,r,Z).
dep(e,r,Z)       :- ap(e,r,_,Z).
indep(e,r,Z):- not ap(e,r,p_eor,Z),not ap(e,r,p_eior,Z).
\end{lstlisting}
\vspace{-0.2cm}
Intuitively, $\texttt{arrow(x,y)}$ encodes a directed edge $X\to Y$; $\texttt{in(v,Z)}$ encodes membership in the conditioning set; $\texttt{collider}$ and $\texttt{nb}$ characterise whether a node blocks a path given $\mathbf{Z}$; and $\texttt{ap(x,y,p,Z)}$ witnesses an active path, where labels such as $\texttt{p\_eor}$ and $\texttt{p\_eior}$ are mnemonic identifiers for the corresponding path shapes. The atoms $\texttt{dep}/\texttt{indep}$ are then derived from $\texttt{ap}$. Empirical CI tests are provided as input facts (e.g., $\texttt{dep(o,r,\{\})}$ for $O\notindep R$ and $\texttt{indep(r,i,\{o\})}$ for $R\indep I\given\{O\}$) and enforced by rejecting any stable model whose derived $\texttt{dep}/\texttt{indep}$ atoms disagree with the input. For $\mathbf{Z}=\{O\}$, conditioning on the collider $O$ opens the path $E\to O\leftarrow R$, so the encoding derives $\texttt{dep(e,r,\{o\})}$, representing $E\not\indep R\mid O$.
\end{example}

If the enforced facts are jointly inconsistent, the ASP program becomes incoherent (no stable models) and a repair step must decide which facts to release. Causal ABA-PC~\citep{russo2024abacd} is a concrete instantiation of Causal ABA in which the defeasible constraints are CI facts obtained from MPC~\citep{colombo2014order}. Causal ABA-PC assigns each fact a \emph{strength} derived from its test $p$-value~\citep{casella2002statistical} and then incrementally releases low-strength facts until coherence is restored: a simple heuristic baseline repair strategy.

\begin{example}[Causal ABA-PC]\label{ex:abapc}
Continuing Ex.~\ref{ex:cd}, consider a Causal ABA-PC run on variables $(E,R,O,I)$ whose MPC CI tests produce a trace of $\texttt{dep}/\texttt{indep}$ facts, sorted by descending strength $S$ and labelled as true (T) or false (F) w.r.t.\ the ground-truth DAG in Fig.~\ref{fig:gt-dag}, as follows:

\noindent\begin{minipage}[t]{0.48\linewidth}
\vspace{-0.45cm}
\begin{lstlisting}
indep(e,o,{i}) S=1.00 (F)
indep(r,i,{o}) S=1.00 (F)
dep(o,r,{})    S=0.99 (T)
dep(o,r,{e})   S=0.95 (T)
dep(e,o,{})    S=0.92 (T)
dep(e,o,{r})   S=0.92 (T)
dep(e,i,{})    S=0.84 (T)
\end{lstlisting}
\end{minipage}\hfill
\begin{minipage}[t]{0.48\linewidth}
\vspace{-0.45cm}
\begin{lstlisting}[escapechar=|]
dep(e,i,{r})   S=0.56 (T)
dep(r,i,{})    S=0.54 (T)
|\smash{\rlap{\tikz[overlay]\draw[dash pattern=on 3pt off 2pt,line width=0.8pt] (0,2ex) -- (0.72\linewidth,2ex);}}|indep(e,i,{o}) S=0.53 (F)
indep(r,i,{e}) S=0.53 (F)
dep(o,r,{i})   S=0.52 (T)
indep(o,i,{})  S=0.51 (F)
indep(e,r,{})  S=0.50 (T)
\end{lstlisting}
\end{minipage}\par\noindent
MPC uses all $14$ tests (including $5$ wrong ones), producing a CPDAG with a single undirected edge ($O\!-\!R$). ABA-PC releases facts in reverse strength order until the nine entries before the dashed separator remain enforced ($7$ true and $2$ false), yielding one compatible DAG with one correct edge ($O\to I$) and two wrong ones ($O\to R$, $I\to E$).
\end{example}

Example~\ref{ex:abapc} illustrates that enforcing all CI facts may be impossible, and that the heuristic release strategy of Causal ABA-PC may discard true facts while retaining false ones. We now show how MUSes and MCSes can be used to identify sets of facts that carry desirable properties for repair.
\vspace{-0.25cm}
\subsection{MUSes and MCSes}
\vspace{-0.25cm}
Explaining and repairing incoherence is a classic theme in diagnosis of ASP programs~\citep{Reiter87Diagnosis,Liffiton08MUS}, and it naturally applies here because noisy CI-test outcomes need not be jointly satisfiable.

MUSes and MCSes are defined relative to a base ASP program $\aspP$ and a distinguished set $\aspO$ of \emph{objective atoms}. For $\aspS\subseteq \aspO$, let $\enforce(\aspP,\aspS,\aspO)$ denote the program obtained from $\aspP$ by (i) adding a choice rule over atoms in $\aspO$ and (ii) forcing every atom in $\aspS$ to be true. Concretely:
\[
\enforce(\aspP,\aspS,\aspO):=\aspP\cup\{\{o\}.\mid o\in\aspO\}
\cup\{\texttt{:- not }o\mid o\in\aspS\}.
\]
\vspace{-0.4cm}
\begin{definition}[\citeauthor{ALVIANO2023103931}~\citeyear{ALVIANO2023103931}]\label{def:mus}
    Let $\aspP$ be a program and $\aspO$ be a set of objective atoms. Then $\aspS\subseteq \aspO$ is an \emph{unsatisfiable set} if $\enforce(\aspP,\aspS,\aspO)$ is incoherent.
    $\aspS$ is a minimal unsatisfiable set (MUS) if no strict subset $\aspS'\subset \aspS$ is an unsatisfiable set.
\end{definition}
\begin{definition}[\citeauthor{ALVIANO2023103931}~\citeyear{ALVIANO2023103931}]\label{def:mcs}
    Let $\aspP$ be a program and $\aspO$ be a set of objective atoms. Then $\aspS\subseteq \aspO$ is a \emph{correction set} if $\enforce(\aspP,\aspO\setminus\aspS,\aspO)$ is coherent.
    $\aspS$ is a minimal correction set (MCS) if no strict subset $\aspS'\subset \aspS$ is a correction set.
\end{definition}
MUSes pinpoint irreducible conflicts, whereas MCSes identify subset-minimal repairs. Write $MUS(\aspP,\aspO)$ and $MCS(\aspP,\aspO)$ for the families of all MUSes and MCSes of $\aspP$ w.r.t.\ $\aspO$, respectively. Every MCS must hit every MUS; indeed, $MCS(\aspP,\aspO)$ is exactly the family of minimal hitting sets of $MUS(\aspP,\aspO)$, and conversely~\citep[Thm.~5.3]{ALVIANO2023103931}. Their number may be exponential.

\begin{example}[MUS/MCS intuition]\label{ex:mus}
Consider a program $\aspP$:
\begin{lstlisting}
a :- not b, o(1).  b :- not c, o(2).  c :- not a, o(3).  :- not d, o(4).
\end{lstlisting}
with objective $\aspO=\{\texttt{o(1)},\texttt{o(2)},\texttt{o(3)},\texttt{o(4)}\}$. Enforcing all of $\aspO$ makes $\enforce(\aspP,\aspO,\aspO)$ incoherent, but for \emph{two independent reasons}: enforcing $\texttt{o(4)}$ forces $\texttt{d}$ even though it cannot be derived, and enforcing $\texttt{o(1)},\texttt{o(2)},\texttt{o(3)}$ activates the negative loop over $\texttt{a},\texttt{b},\texttt{c}$. Hence the MUSes are $\{\texttt{o(4)}\}$ and $\{\texttt{o(1)},\texttt{o(2)},\texttt{o(3)}\}$, while the MCSes are the minimal hitting sets $\{\texttt{o(1)},\texttt{o(4)}\}$, $\{\texttt{o(2)},\texttt{o(4)}\}$, and $\{\texttt{o(3)},\texttt{o(4)}\}$.
\end{example}
\vspace{-0.25cm}
\section{Optimal Correction Sets in ASP}\label{sec:opt-mcs}
\vspace{-0.25cm}
Let $CS(\aspP,\aspO)$ contain all correction sets for $\aspP$ w.r.t.\ $\aspO$. We assign each objective atom a strictly positive release cost, $\aspW=\{(o,w_o)\mid o\in\aspO,\ w_o>0\}$, and write \( \mathcal{C}(\aspR,\aspW)=\sum_{o\in\aspR}w_o. \) We seek a minimum-cost coherent repair without enumerating exponentially many MCSes.

\begin{definition}[Optimal MCS]\label{def:opt-mcs}
Let $\aspP$ be a program, $\aspO$ a set of objective atoms, and $\aspW$ weights over $\aspO$. An MCS $\aspS\in MCS(\aspP,\aspO)$ is \emph{optimal} (w.r.t.\ $\aspW$) if for every $\aspS'\in MCS(\aspP,\aspO)$ it holds that $\mathcal{C}(\aspS,\aspW)\le \mathcal{C}(\aspS',\aspW)$.
\end{definition}

\begin{restatable}[Positive-cost minimality]{lemma}{positivecostminimality}\label{lem:positive-minimality}
Let $\aspP$ be a finite ground program, let $\aspO$ be finite, and let $\aspW=\{(o,w_o)\mid o\in\aspO,\ w_o>0\}$. Every minimum-cost member of $CS(\aspP,\aspO)$ is an optimal MCS with respect to $\aspW$.
\end{restatable}

All proofs are in \OptABAIfSupplementElse{Appendix~\ref{supp:bayes}--\ref{supp:contestability}}{the appendix sections ``Correctness of the Implemented Bayes-Ball Rules'', ``Optimal-MCS'', ``OptABA-PC Guarantees'', and ``Contestability and Redress''}.

Rather than enumerating repairs, we expose each objective atom as a choice and use the solver's native optimisation. For a program $\aspP$, objective atoms $\aspO$, and weights $\aspW$, define
\[
\adorned(\aspP,\aspO,\aspW)=\aspP\cup\{\{o\}.\mid o\in\aspO\}
\cup\{ : \sim\ \texttt{not}\ o.\ [w_o@1,o]\mid(o,w_o)\in\aspW\}.
\]
For a stable model $M$, let its released-fact projection be $\rho(M)=\aspO\setminus M$. The identifier $o$ in the weak-constraint tuple prevents equal-weight facts from collapsing into one occurrence. Consequently, the weak-constraint cost of $M$ is exactly $\mathcal{C}(\rho(M),\aspW)$.

\begin{example}[Optimal MCS]\label{ex:opt-mcs}
In Ex.~\ref{ex:mus}, assigning costs $5,10,15,20$ to $\texttt{o(1)},\ldots,\texttt{o(4)}$ gives the three repairs costs $25,30,35$; releasing $\texttt{o(1)}$ and $\texttt{o(4)}$ is therefore uniquely optimal.
For the same instance, an optimal model is $M=\{\texttt{o(2)},\texttt{o(3)},\texttt{c}\}$, so $\rho(M)=\{\texttt{o(1)},\texttt{o(4)}\}$; the fully instantiated program is in \OptABAIfSupplementElse{Appendix~\ref{supp:projection}}{the appendix section ``Optimal-MCS''}.

\end{example}

\begin{restatable}[Projection correspondence]{theorem}{projectioncorrespondence}\label{thm:projection}
Let $\aspP$ be a finite ground program, $\aspO$ a finite set of objective atoms, and $\aspW$ assign strictly positive release weights to all atoms in $\aspO$. Then:
\par\noindent$(\Rightarrow)$ If $M$ is an optimal stable model of $\adorned(\aspP,\aspO,\aspW)$, then $\rho(M)$ is an optimal MCS (Def.~\ref{def:opt-mcs}), and the weak-constraint cost of $M$ is $\mathcal C(\rho(M),\aspW)$.
\par\noindent$(\Leftarrow)$ For every optimal MCS $R$, there exists an optimal stable model $M$ of $\adorned(\aspP,\aspO,\aspW)$ with $\rho(M)=R$.
\par\noindent Distinct optimal stable models may have the same released-fact projection.
\end{restatable}

\begin{restatable}[Conflict justification]{proposition}{privateconflictjustification}\label{prop:private-conflict}
Let $\aspP$ be a finite ground program and $\aspO$ a finite set of objective atoms. For every $\aspR\in MCS(\aspP,\aspO)$ and every $f\in\aspR$, there is a MUS $U_f\in MUS(\aspP,\aspO)$ such that $U_f\cap\aspR=\{f\}$.
\end{restatable}

Thus every released objective atom has a private conflict certificate containing no other released fact. Reinstating that atom while retaining the others in the certificate restores incoherence. We use this principled repair object to provide contestability and redress for Causal ABA, where objective atoms are weighted CI facts.

\section{OptABA-PC}

\subsection{Bayes-ball Encoding of Causal ABA}\label{sec:bayesball-encoding}
We separate the inherited Causal ABA machinery from our replacement for its active-path computation. For a fixed variable set $\mathbf V$ and finite set $\aspO$ of valid CI objective atoms, let $\mathcal S_{\aspO}$ contain the conditioning-set identifiers occurring in $\aspO$. The inherited $\aspP^{\mathrm{CABA}}_{\mathrm{core}}$ supplies the variable and conditioning-set vocabulary, CI symmetry and consistency, candidate-DAG generation, adjacency, directed reachability, and acyclicity; $\aspP^{\mathrm{new}}_{\mathrm{BB}}$ computes $d$-connection and enforces selected CI facts. Their union is the finite ground base program
\begin{equation}\label{eq:pca-decomposition}
\begin{aligned}
\aspP_{\mathrm{CA}}(\mathbf V,\aspO) &= \aspP^{\mathrm{CABA}}_{\mathrm{core}}(\mathbf V,\mathcal S_{\aspO})
\cup\aspP^{\mathrm{new}}_{\mathrm{BB}}(\mathbf V,\aspO).
\end{aligned}
\end{equation}
For a fixed instance write $\aspP_{\mathrm{CA}}$. The core chooses $X\to Y$, $Y\to X$, or no edge for each unordered pair, yielding one candidate DAG per stable model; its rules are in \OptABASuppListingsRef{lst:caba-vocabulary}{lst:caba-acyclicity}. Objective choices and weak constraints are added by $\enforce$ and $\adorned$.

The original Causal ABA encoding derives $\texttt{dep}/\texttt{indep}$ by materialising active paths for each queried triple $(X,Y,\mathbf{Z})$. This is faithful to $d$-separation but expensive, since the number of simple paths can be exponential. We therefore replace the path-enumerating modules with an ASP implementation of the Bayes-ball algorithm~\citep{shachter1998bayesball}. Bayes-ball does not enumerate paths; for a fixed query $(X,Y,\mathbf{Z})$, it propagates a marker through the current DAG according only to the current node, the direction from which the marker arrived, and whether the node is in, or has a descendant in, $\mathbf{Z}$.
Bayes-ball is instantiated for input CI queries and their symmetric counterparts. Its two states record arrival from a child or parent; reaching the target derives $\texttt{ap3(X,Y,Z)}$, a $d$-connection witness. The final constraints require this witness for an enforced dependence and forbid it for an enforced independence. \OptABAIfSupplementElse{Listing~\ref{lst:bayes-ball} in Appendix~\ref{supp:bb-encoding}}{The appendix section ``Novel Bayes-ball encoding''} gives the complete encoding.

\begin{restatable}[Bayes-ball encoding correctness]{lemma}{bayesballcorrectness}\label{lem:bayes-ball}
Let $\mathbf V$ be finite, let $G\in\operatorname{DAG}(\mathbf V)$ fix the \texttt{arrow} atoms, and let \texttt{qpair(X,Y,Z)} encode a valid query with distinct endpoints $X,Y\in\mathbf V\setminus\mathbf Z$. In the least fixed point of the grounded propagation rules in \OptABAIfSupplementElse{Listing~\ref{lst:bayes-ball}}{the Bayes-ball encoding}, \texttt{ap3(X,Y,Z)} is derived iff $X$ and $Y$ are $d$-connected given $\mathbf Z$ in $G$.
\end{restatable}

For a stable model $M$ of a Causal ABA instance using $\aspP_{\mathrm{CA}}$, define its \emph{graph projection} by
\[
\pi_G(M)=(\mathbf V,E_M),\qquad
E_M=\{(X,Y)\in\mathbf V^2\mid \texttt{arrow}(X,Y)\in M\}.
\]
This discards the objective, reachability, and Bayes-ball atoms and differs from the released-fact projection $\rho(M)$.

\begin{restatable}[Graph fidelity]{corollary}{graphfidelity}\label{cor:graph-fidelity}
For the finite sets $\mathbf V$ and $\aspO$ defining $\aspP_{\mathrm{CA}}=\aspP_{\mathrm{CA}}(\mathbf V,\aspO)$, let $\operatorname{DAG}(\mathbf V)$ be the set of DAGs on the fixed vertex set $\mathbf V$, and let $G\models T$ abbreviate $G\models f$ for every $f\in T$. For every $T\subseteq\aspO$ of valid CI facts,
\[
\{\pi_G(M)\mid M\in AS(\enforce(\aspP_{\mathrm{CA}},T,\aspO))\}
=\{G\in\operatorname{DAG}(\mathbf V)\mid G\models T\}.
\]
\end{restatable}

In OptABA-PC, the objective atoms in $\aspO$ are the CI facts produced by MPC~\citep{colombo2014order}. Let $\aspW$ contain their strictly positive release costs. OptABA-PC solves $\adorned(\aspP_{\mathrm{CA}},\aspO,\aspW)$ with \textit{clingo}~\citep{GebserKS11}. Unlike ABA-PC's incremental strength-ordered release, one solve minimises the total cost of the released set. Note that we use $w_f=\max\{1,\operatorname{round}(9{,}999{,}999s(p_f))\}$, where $s(p)=1-p/(2\alpha)$ for $p<\alpha$ and $s(p)=(p+1-2\alpha)/(2(1-\alpha))$ otherwise, following ABA-PC~\citep{russo2024abacd} without weighting on conditioning set size. These costs encode preferences in line with the expected behaviour of $p$-values of being most uncertain around the threshold $\alpha$~\citep{hung97pvalue,casella2002statistical}, not calibrated probabilities.

\begin{example}[OptABA-PC]\label{ex:OptABA-PC}
Continuing Ex.~\ref{ex:abapc} with the same fact weights, OptABA-PC releases the optimal MCS
\(
R=\{\texttt{indep(e,o,\{i\})},\texttt{indep(r,i,\{e\})},\texttt{indep(o,i,\{\})}\}.
\)
All three facts are wrong with respect to the $d$-separations of the DAG in Fig.~\ref{fig:gt-dag}; the repair retains all nine correct facts and the two remaining wrong facts. Its single compatible DAG is $E\to O$, $R\to O$, $O\to I$, with all three edges correctly oriented; while ABA-PC only one (cf.\ Ex.~\ref{ex:abapc}). Thus global minimum-cost repair can improve both CI selection and its graph witness.
\end{example}
\subsection{Optimality and Contestability Guarantees}

\begin{restatable}[OptABA-PC guarantees]{theorem}{optabaguarantees}\label{thm:optaba-guarantees}
Let $\mathbf V$ be finite, $\aspO$ a finite set of valid CI facts over $\mathbf V$, and $\aspW$ assign strictly positive release weights to every fact in $\aspO$. Then $\adorned(\aspP_{\mathrm{CA}},\aspO,\aspW)$ has an optimal stable model. For every such model $M$, put $R=\rho(M)$ and $T=\aspO\setminus R$. Then:
\vspace{-0.1cm}
\par\noindent\textup{(1)} $\pi_G(M)$ is a DAG on $\mathbf V$ satisfying every fact in $T$;
\vspace{-0.1cm}
\par\noindent\textup{(2)} $R$ is an MCS with minimum release cost among all correction sets of $\aspP_{\mathrm{CA}}$ w.r.t.\ $\aspO$;
\vspace{-0.1cm}
\par\noindent\textup{(3)} for every $f\in R$, enforcing $T\cup\{f\}$ is incoherent, and a MUS $U_f$ satisfies $U_f\cap R=\{f\}$.
\vspace{-0.1cm}

Conversely, every optimal MCS $R$ is $\rho(M)$ for some optimal stable model $M$, with the same cost.
\end{restatable}
The correspondence is under projection, not a bijection: several optimal stable models may have the same released-fact projection and possibly different graph projections. Our implementation saves a selected repair and separately enumerates compatible graphs within a time budget; the reported graph sets need not exhaust the entire family.
These guarantees support two forms of user contestation. A release-cost contestation increases the penalty for releasing $f$; a fact-enforcement contestation excludes every repair that releases $f$.

Given a stable model $M$ and a repair $R=\rho(M)$, a fact $f$ is \emph{retained} if $f\in M$, equivalently $f\notin R$, and \emph{released} if $f\notin M$, equivalently $f\in R$. For a base program $\aspP$, write $\mathcal F=CS(\aspP,\aspO)$. In OptABA-PC, returned repairs are the projections $\rho(M)$ of optimal stable models. For a contested fact $f$, let the hard-retention program be:
\[
\adorned_f(\aspP,\aspO,\aspW)
=\adorned(\aspP,\aspO,\aspW)\cup\{\texttt{:- not }f\}.
\]

\begin{restatable}[Release-cost contestability]{property}{releasecostcontestability}\label{prop:release-cost}
For a weighted repair method over $\mathcal F$, let $\operatorname{Rep}_{\aspW}\subseteq\mathcal F$ denote the repairs returned under weights $\aspW$. Suppose $\aspW'$ differs from $\aspW$ only in that $w'_f>w_f$. The method is \emph{release-cost contestable at $f$} if its returned repairs can change only towards retaining $f$:
\[
\{R\in\operatorname{Rep}_{\aspW}:f\notin R\}\subseteq\operatorname{Rep}_{\aspW'},
\qquad
\{R\in\operatorname{Rep}_{\aspW'}:f\in R\}\subseteq\operatorname{Rep}_{\aspW}.
\]
\end{restatable}
Effectively, the optimal MCSes not containing $f$ will remain optimal under $W'$ while the ones that contain $f$ under the updated $W'$ must have already been optimal. Thus increasing the release cost may remove releasing repairs or add retaining repairs, but not conversely.
\begin{restatable}[Fact-enforcement contestability]{property}{factenforcementcontestability}\label{prop:fact-enforcement}
An ASP repair method based on $\adorned(\aspP,\aspO,\aspW)$ is \emph{fact-enforcement contestable at $f$} if
\[
\{\rho(M)\mid M\text{ is optimal for }\adorned_f(\aspP,\aspO,\aspW)\}
=\arg\min_{R\in\mathcal F,\,f\notin R}\mathcal{C}(R,\aspW).
\]
It is \emph{total} on a fact domain if, for every fact $f$ in that domain, some $R\in\mathcal F$ satisfies $f\notin R$. Thus making $f$ mandatory returns exactly the best repairs that retain it.
\end{restatable}

\begin{restatable}[OptABA-PC contestability]{proposition}{optabacontestability}\label{prop:optaba-contestability}
Let $\mathbf V$ and $\aspO$ be finite, let every member of $\aspO$ be a valid CI fact over $\mathbf V$, put $\aspP_{\mathrm{CA}}=\aspP_{\mathrm{CA}}(\mathbf V,\aspO)$ as in Equation~\eqref{eq:pca-decomposition}, and let $\aspW=\{(o,w_o)\mid o\in\aspO,\ w_o>0\}$. Then OptABA-PC satisfies Properties~\ref{prop:release-cost} and~\ref{prop:fact-enforcement} for every $f\in\aspO$. Moreover, fact-enforcement contestability is total: $\adorned_f(\aspP_{\mathrm{CA}},\aspO,\aspW)$ has an optimal stable model for every $f\in\aspO$.
\end{restatable}

Besides making $f$ mandatory, a user may increase its release cost, as in~\citep{freedman2025argllm}. Property~\ref{prop:release-cost} guarantees that such graded contestation can only favour retention.

\begin{restatable}[Contestability margin]{definition}{contestabilitymargin}\label{def:margin}
Let $\mathcal F$ be a finite family of repairs and $\aspW$ assign a finite release cost to every objective atom. For any fact $f$ for which $\mathcal F$ contains both a repair retaining $f$ and a repair releasing $f$, define
\[
c_{\mathrm{keep}}(f)=\min_{R\in\mathcal F,\,f\notin R}\mathcal{C}(R,\aspW),
\qquad
c_{\mathrm{rel}}(f)=\min_{R\in\mathcal F,\,f\in R}\mathcal{C}(R,\aspW).
\]
If at least one globally optimal repair under $\aspW$ releases $f$, then $c_{\mathrm{rel}}(f)\leq c_{\mathrm{keep}}(f)$. The \emph{contestability margin} of $f$ is $\mu(f)=c_{\mathrm{keep}}(f)-c_{\mathrm{rel}}(f)\geq 0$.
\end{restatable}

Both minima exist: Prop.~\ref{prop:optaba-contestability} gives a repair retaining $f$, and releasing all facts is feasible. Increasing $w_f$ by $\delta$ adds $\delta$ only to repairs releasing $f$. Thus below $\mu(f)$ every optimum releases $f$, at equality both kinds coexist, and above it every optimum retains $f$; for integer weights, the least forcing increase is $\mu(f)+1$ (proof in \OptABAIfSupplementElse{App.~\ref{supp:contestability}, Prop.~\ref{prop:margin-threshold}}{appendix section ``Contestability and Redress''}).

\begin{example}[Contestation of a released fact]\label{ex:contestability}
Consider $f=\texttt{indep(e,o,\{i\})}$, released in Ex.~\ref{ex:OptABA-PC}. For this illustration, set each release cost to $1000S$ using the strengths in Ex.~\ref{ex:abapc}. The three released facts in $R$ therefore cost $c_{\rm rel}(f)=1000+530+510=2040$. One MUS for $f$ contains two retained facts: $\texttt{dep(e,o,\{\})}$ and $\texttt{indep(e,i,\{o\})}$. The two independences imply $\texttt{indep(e,o,\{\})}$, contradicting the retained dependence; hence $f$ cannot be reinstated while both other facts stay. The cheapest repair of the full trace with $f$ mandatory replaces its release by releases of $\texttt{indep(e,i,\{o\})}$ (cost $530$) and $\texttt{indep(e,r,\{\})}$ (cost $500$), leaving the other two releases unchanged. Thus $c_{\rm keep}(f)=2040-1000+530+500=2070$ and $\mu(f)=30$. Increasing only $w_f$ from $1000$ to $1030$ ties the two repair costs; $1031$ makes every optimum retain $f$. The returned DAGs satisfy the requested $f$, even though, in this case, it is false wrt the generating DAG.
\end{example}

Within a fixed trace and weighting scheme, the margin measures relative weighted degree of inconsistency; \OptABAIfSupplementElse{Apps.~\ref{supp:contestability} and~\ref{supp:evidence-overview}}{the appendix sections ``Contestability and Redress'' and ``Fact Enforcement Comparison''} develop and evaluate this interpretation.

\section{Experiments}\label{sec:experiments}

\textbf{Setup.} Following~\citep{russo2024abacd}, we use the Cancer, Earthquake, Survey, and Asia bnlearn networks from~\citep{Scutari14bnlearn} and add Erd\H{o}s--R\'enyi (ER)~\citep{erdds1959random} and scale-free (SF)~\citep{barabasi1999emergence} DAGs with five and eight nodes. Each dataset has $N=5000$ observations and 50 seeds; fixed networks retain their graph and conditional probability tables, whereas each synthetic seed generates both. ABA-PC and OptABA-PC use the same Bayes-ball encoding and receive identical MPC~\citep{colombo2014order} $G^2$ CI-test traces and weights ($\alpha=0.01$), isolating their repair mechanisms. Fast Greedy Search (FGS)~\citep{ramsey2017million} and ASPCR-DAG~\citep{hyttinen2014ASP} use their native configurations. ABA runs use a $300$-second solve limit and a separate $120$-second graph-evaluation limit; $32/400$ OptABA-PC runs reach the optimisation limit and return a valid incumbent without proved optimality. Full protocol in \OptABAIfSupplementElse{App.~\ref{supp:experimental-protocol}}{the appendix section ``Experimental Protocol and Metrics''}, implementation and experimental material are available at \url{https://github.com/briziorusso/ArgCausalDisco}.

\textbf{Metrics.} Table~\ref{tab:main-results} reports mean paired OptABA-PC--ABA-PC differences. Retained true facts agree with the generating DAG; FactF1 also penalises retained false and released true facts. \#CPDAG counts saved equivalence classes satisfying all retained facts; CPDAG F1 and Structural Hamming Distance (SHD) measure reconstruction over these budget-limited sets. To evaluate contestability, we instantiate Def.~\ref{def:margin} for every released fact $f$: the baseline optimum gives $c_{\mathrm{rel}}(f)$, a re-solve enforcing $f$ gives $c_{\mathrm{keep}}(f)$, and their difference is the empirical margin $\mu(f)$. \OptABAIfSupplementElse{App.~\ref{supp:empirical-evidence}}{The appendix section ``Additional Empirical Evidence''} gives absolute results (including Earthquake), variability, significance, precision--recall, hard-retention feasibility, and margin distributions.

\begin{table}[t]
\caption{Paired OptABA-PC--ABA-PC effects on seven bnlearn and synthetic datasets (nodes in brackets). Entries are mean within-seed differences $\Delta$; parentheses give relative change from ABA-PC, except runtime ($t_{\mathrm{Opt}}/t_{\mathrm{ABA}}$). Arrows indicate preferred directions. $\Delta n_{\mathrm{acc}}^{T}$ counts additional true CI facts retained; \#CPDAG counts saved MECs satisfying every retained fact. Counts and graph means describe saved, budget-limited sets. \OptABAIfSupplementElse{Tables~\ref{tab:absolute-results} and~\ref{tab:absolute-cpdag-results} in the supplement}{The tables in the appendix section ``Additional Empirical Results and Discussion''} give absolute mean $\pm$ standard deviation and significance for all datasets; Earthquake is omitted because outcomes do not differ.}
\label{tab:main-results}
\centering
\small
\setlength{\tabcolsep}{4pt}
\textbf{CPDAG evaluation}\par\smallskip
\begin{tabular}{@{}lccc@{}}
\toprule
Dataset & $\Delta\#\mathrm{CPDAG}\downarrow$ & $\Delta\mathrm{F1}\uparrow$ & $\Delta\mathrm{SHD}\downarrow$\\
\midrule
Cancer (5) & $-14.1\ (-93.4\%)$ & $+0.054\ (+10.3\%)$ & $-0.974\ (-34.6\%)$\\
Survey (6) & $-1277.2\ (-99.9\%)$ & $+0.159\ (+36.3\%)$ & $-2.790\ (-39.1\%)$\\
Asia (8) & $-10821.5\ (-100.0\%)$ & $+0.313\ (+139.7\%)$ & $-10.507\ (-73.1\%)$\\
ER (5) & $-5.2\ (-63.6\%)$ & $+0.027\ (+4.9\%)$ & $-0.792\ (-27.5\%)$\\
ER (8) & $-3835.3\ (-99.9\%)$ & $+0.150\ (+31.3\%)$ & $-4.661\ (-56.9\%)$\\
SF (5) & $-3.3\ (-3.7\%)$ & $+0.067\ (+10.1\%)$ & $-0.741\ (-26.8\%)$\\
SF (8) & $-2937.8\ (-94.5\%)$ & $+0.131\ (+26.3\%)$ & $-3.725\ (-45.7\%)$\\
\bottomrule
\end{tabular}
\par\medskip
\textbf{Fact selection and runtime}\par\smallskip
\begin{tabular}{@{}lccc@{}}
\toprule
Dataset & $\Delta n_{\mathrm{acc}}^{T}\uparrow$ & $\Delta\mathrm{FactF1}\uparrow$ & $\Delta t\downarrow$\\
\midrule
Cancer (5) & $+1.6\ (+6.8\%)$ & $-0.004\ (-0.5\%)$ & $-0.233\ (\times0.85)$\\
Survey (6) & $+4.9\ (+19.6\%)$ & $-0.047\ (-5.9\%)$ & $-12.728\ (\times0.19)$\\
Asia (8) & $+21.8\ (+20.6\%)$ & $+0.082\ (+9.2\%)$ & $-28.904\ (\times0.70)$\\
ER (5) & $+1.3\ (+4.1\%)$ & $+0.016\ (+1.6\%)$ & $-0.106\ (\times0.90)$\\
ER (8) & $+15.4\ (+21.5\%)$ & $+0.082\ (+9.5\%)$ & $+91.044\ (\times2.90)$\\
SF (5) & $+1.7\ (+6.0\%)$ & $+0.034\ (+3.6\%)$ & $-0.094\ (\times0.95)$\\
SF (8) & $+21.1\ (+29.2\%)$ & $+0.093\ (+11.1\%)$ & $+56.807\ (\times2.34)$\\
\bottomrule
\end{tabular}
\end{table}

\textbf{Results.} Across all seven datasets in Table~\ref{tab:main-results}, OptABA-PC retains more true facts, identifies fewer compatible classes, and improves mean CPDAG F1 and SHD compared to ABA-PC. Against other baselines, graph accuracy is mixed: MPC and FGS often score better, particularly on synthetic benchmarks\OptABAIfSupplement{\ (Table~\ref{tab:final-graph-metrics})}. Our scores average over a returned set of compatible CPDAGs, whereas MPC and FGS return one. MPC has no significant advantage over the best F1 and SHD in our set on three of four synthetic settings\OptABAIfSupplement{\ (Tables~\ref{tab:final-cpdag-ranges} and~\ref{tab:final-cpdag-f1-ranges})}. OptABA-PC achieves higher CPDAG and DAG F1 than its closest ASP-based counterpart, ASPCR-DAG, on all shared benchmarks.

The main gains in Table~\ref{tab:main-results} are stronger fact selection ($\Delta n_{\mathrm{acc}}^{T}$ and $\Delta\mathrm{FactF1}$) and fewer competing Markov Equivalence Classes ($\Delta\#\mathrm{CPDAG}$). Relative to ABA-PC, OptABA-PC retains more true independence and dependence facts on all seven datasets. FactF1 also improves on Asia and every synthetic benchmark, indicating better retain/release classification overall. Survey trades precision for more true-fact retention; Cancer's small FactF1 decline is not significant. Over $94\%$ fewer compatible CPDAGs are saved on every eight-variable benchmark, reducing uncertainty over structures consistent with the retained evidence.

CI-DAG correspondence further distinguishes the outputs: every graph by OptABA-PC satisfies the explicitly retained facts, with conflict certificates explaining releases, while MPC and ASPCR-DAG outputs respect their full CI traces in only $80/400$ ($20\%$) and $32/250$ ($12.8\%$) runs, respectively\OptABAIfSupplement{\ (Table~\ref{tab:contestability-overview})}.
On datasets with at most six variables, all $742$ requests to reinstate released facts yielded minimum-cost consistent repairs. For graded contestation, the median increase in a fact's release cost needed to make retention optimal was $45.5\%$\OptABAIfSupplement{\ (Tables~\ref{tab:contestability-margin-comparison} and~\ref{tab:margin-by-dataset})}.
OptABA-PC uses one exact optimisation instead of repeated repair solves: it is faster than ABA-PC on five datasets, but slower on ER(8) and SF(8) (Table~\ref{tab:main-results}). Under matched heuristic repair, our Bayes-ball encoding is over $2\times$ as fast as the original active-path encoding on six- and eight-variable benchmarks\OptABAIfSupplement{\ (Table~\ref{tab:bayesball-runtime})}.

\vspace{-0.2cm}
\section{Conclusion}
\vspace{-0.2cm}
We cast causal discovery from noisy data as transparent CI-fact repair. By aligning optimal MCSes with weak-constraint optimisation,
OptABA-PC couples minimum-cost CI-fact repair with guaranteed CI-DAG correspondence, conflict explanations, and predictable redress. The contribution is optimal and provable redress, not solely competitive graph accuracy. OptABA-PC traceable link from retained facts to every returned DAG supports transparency and contestation. The empirical evidence of higher true CI retention also supports the effectiveness of the approach in contexts where data is limited or uncertain and external input is needed and available, making it particularly suited to high-stakes decisions with modest variable sets. The observed benefits and theoretical guarantees motivate more efficient methods that preserve these properties. Our proposed Bayes-ball encoding is one concrete step in this direction; scalability and robustness to release costs remain priorities.

\section*{Acknowledgements}
Russo was funded by the ERC under the ERC-POC programme (grant number 101189053), Toni under the EU’s Horizon 2020 research and innovation programme (grant number 101020934) while Mazzotta, Dodaro and Ricca were partially funded by the Italian Ministry of Industrial Development (MISE) under project EI-TWIN n. F/310168/05/X56 CUP B29J24000680005 and also partially by projects  SIGENERA (CUP J29I24001770005) and MOZART (J49I24001740005) selected within the framework of the PR FESR – FSE Calabria 2021/2027 and implemented with the support of the Italian State and the Calabria Region.

\bibliographystyle{preprint}
\bibliography{references}

@article{ALVIANO2023103931,
  author       = {Mario Alviano and Carmine Dodaro and Salvatore Fiorentino and Alessandro Previti and Francesco Ricca},
  title        = {{ASP} and subset minimality: Enumeration, cautious reasoning and MUSes},
  journal      = {Artificial Intelligence},
  volume       = {320},
  pages        = {103931},
  year         = {2023},
  doi          = {10.1016/j.artint.2023.103931},
  issn         = {0004-3702}
}

@article{AmaldiPfetschTrotter2003MaxFS,
  author       = {Edoardo Amaldi and Marc E. Pfetsch and Leslie E. Jr. Trotter },
  title        = {On the Maximum Feasible Subsystem Problem, {IIS}s and {IIS}-Hypergraphs},
  journal      = {Mathematical Programming},
  volume       = {95},
  number       = {3},
  pages        = {533--554},
  year         = {2003},
  doi          = {10.1007/s10107-002-0359-y}
}

@article{barabasi1999emergence,
  author       = {Albert-László Barabási and Réka Albert},
  title        = {Emergence of Scaling in Random Networks},
  journal      = {Science},
  publisher    = {American Association for the Advancement of Science},
  volume       = {286},
  number       = {5439},
  pages        = {509--512},
  year         = {1999},
  url          = {http://www.doi.org/10.1126/science.286.5439.509}
}

@article{benjamini1995controlling,
  author       = {Yoav Benjamini and Yosef Hochberg},
  title        = {Controlling the False Discovery Rate: A Practical and Powerful Approach to Multiple Testing},
  journal      = {Journal of the Royal Statistical Society: Series B (Methodological)},
  volume       = {57},
  number       = {1},
  pages        = {289--300},
  year         = {1995},
  doi          = {10.1111/j.2517-6161.1995.tb02031.x}
}

@inproceedings{borboudakis2016towards,
  author       = {Giorgos Borboudakis and Ioannis Tsamardinos},
  title        = {Towards Robust and Versatile Causal Discovery for Business Applications},
  booktitle    = {Proceedings of the 22nd {ACM} {SIGKDD} International Conference on Knowledge Discovery and Data Mining, {KDD} 2016},
  pages        = {1435--1444},
  year         = {2016},
  doi          = {10.1145/2939672.2939872}
}

@book{casella2002statistical,
  author       = {Casella, George and Berger, Roger},
  title        = {Statistical Inference},
  publisher    = {Thomson Learning},
  year         = {2002},
  doi          = {10.1201/9781003456285},
  isbn         = {9780534243128}
}

@article{chickering2002learning,
  author       = {David Maxwell Chickering},
  title        = {Learning Equivalence Classes of Bayesian-Network Structures},
  journal      = {Journal of Machine Learning Research},
  volume       = {2},
  pages        = {445--498},
  year         = {2002},
  url          = {https://jmlr.org/papers/v2/chickering02a.html}
}

@article{chickering2002optimal,
  author       = {David Maxwell Chickering},
  title        = {Optimal Structure Identification With Greedy Search},
  journal      = {Journal of Machine Learning Research},
  volume       = {3},
  pages        = {507--554},
  year         = {2002},
  url          = {https://jmlr.org/papers/v3/chickering02b.html}
}

@article{ChinneckDravnieks1991IIS,
  author       = {John W. Chinneck and E. W. Dravnieks},
  title        = {Locating Minimal Infeasible Constraint Sets in Linear Programs},
  journal      = {{ORSA} Journal on Computing},
  volume       = {3},
  number       = {2},
  pages        = {157--168},
  year         = {1991},
  doi          = {10.1287/ijoc.3.2.157}
}

@article{colombo2014order,
  author       = {Diego Colombo and Marloes H. Maathuis},
  title        = {Order-independent constraint-based causal structure learning},
  journal      = {Journal of Machine Learning Research},
  volume       = {15},
  number       = {1},
  pages        = {3741--3782},
  year         = {2014},
  doi          = {10.5555/2627435.2750365}
}

@article{CyrasFST2018,
  author       = {Kristijonas Cyras and Xiuyi Fan and Claudia Schulz and Francesca Toni},
  title        = {Assumption-based Argumentation: Disputes, Explanations, Preferences},
  journal      = {{IfCoLog Journal of Logics and their Applications}},
  volume       = {4},
  number       = {8},
  year         = {2017},
  url          = {http://www.collegepublications.co.uk/downloads/ifcolog00017.pdf}
}

@article{DBLP:journals/ai/BondarenkoDKT97,
  author       = {Andrei Bondarenko and Phan Minh Dung and Robert A. Kowalski and Francesca Toni},
  title        = {An Abstract, Argumentation-Theoretic Approach to Default Reasoning},
  journal      = {Artificial Intelligence},
  volume       = {93},
  pages        = {63--101},
  year         = {1997},
  doi          = {10.1016/S0004-3702(97)00015-5}
}

@article{DBLP:journals/constraints/MorgadoHLPM13,
  author       = {Ant{\'{o}}nio Morgado and Federico Heras and Mark H. Liffiton and Jordi Planes and Jo{\~{a}}o Marques{-}Silva},
  title        = {Iterative and core-guided MaxSAT solving: {A} survey and assessment},
  journal      = {Constraints An Int. J.},
  volume       = {18},
  number       = {4},
  pages        = {478--534},
  year         = {2013},
  doi          = {10.1007/S10601-013-9146-2},
  url          = {https://doi.org/10.1007/s10601-013-9146-2}
}

@article{DBLP:journals/ijar/RantanenHJ20,
  author       = {Kari Rantanen and Antti Hyttinen and Matti J{\"{a}}rvisalo},
  title        = {Discovering causal graphs with cycles and latent confounders: An exact branch-and-bound approach},
  journal      = {Int. J. Approx. Reason.},
  volume       = {117},
  pages        = {29--49},
  year         = {2020},
  doi          = {10.1016/j.ijar.2019.10.009}
}

@article{DBLP:journals/tplp/CalimeriFGIKKLM20,
  author       = {Francesco Calimeri and Wolfgang Faber and Martin Gebser and Giovambattista Ianni and Roland Kaminski and Thomas Krennwallner and Nicola Leone and Marco Maratea and Francesco Ricca and Torsten Schaub},
  title        = {ASP-Core-2 Input Language Format},
  journal      = {Theory and Practice of Logic Programming},
  volume       = {20},
  number       = {2},
  pages        = {294--309},
  year         = {2020},
  doi          = {10.1017/S1471068419000450},
  url          = {https://doi.org/10.1017/S1471068419000450}
}

@article{erdds1959random,
  author       = {Erdős, Paul and Alfréd Rényi },
  title        = {On random graphs {I}},
  journal      = {Publicationes Mathematicae Debrecen},
  volume       = {6},
  number       = {290-297},
  pages        = {18},
  year         = {1959}
}

@article{freedman2025argllm,
  author       = {Gabriel Freedman and Adam Dejl and Deniz Gorur and Xiang Yin and Antonio Rago and Francesca Toni},
  title        = {Argumentative Large Language Models for Explainable and Contestable Claim Verification},
  journal      = {Proceedings of the AAAI Conference on Artificial Intelligence, {AAAI} 2025},
  volume       = {39},
  number       = {14},
  pages        = {14930--14939},
  year         = {2025},
  doi          = {10.1609/aaai.v39i14.33637},
  url          = {https://doi.org/10.1609/aaai.v39i14.33637}
}

@article{GebserKS11,
  author       = {Martin Gebser and Roland Kaminski and Torsten Schaub},
  title        = {Complex optimization in answer set programming},
  journal      = {Theory and Practice of Logic Programming},
  volume       = {11},
  number       = {4-5},
  pages        = {821--839},
  year         = {2011},
  doi          = {10.1017/S1471068411000329}
}

@inproceedings{gel88,
  author       = {Michael Gelfond and Vladimir Lifschitz},
  title        = {The Stable Model Semantics for Logic Programming},
  booktitle    = {Proceedings of the International Logic Programming Conference and Symposium},
  publisher    = {{MIT} Press},
  pages        = {1070--1080},
  year         = {1988},
  url          = {https://www.cs.utexas.edu/~ai-lab/pubs/gel88.pdf}
}

@article{glymour2019review,
  author       = {Glymour, Clark and Zhang, Kun and Spirtes, Peter},
  title        = {Review of causal discovery methods based on graphical models},
  journal      = {Frontiers in Genetics},
  publisher    = {Frontiers Media SA},
  volume       = {10},
  pages        = {524},
  year         = {2019},
  doi          = {10.3389/fgene.2019.00524}
}

@article{hung97pvalue,
  author       = {H. M. James Hung and Robert T. O'Neill and Peter Bauer and Karl Kohne},
  title        = {The Behavior of the P-Value When the Alternative Hypothesis is True},
  journal      = {Biometrics},
  publisher    = {[Wiley, International Biometric Society]},
  volume       = {53},
  number       = {1},
  pages        = {11--22},
  year         = {1997},
  url          = {http://www.jstor.org/stable/2533093},
  issn         = {0006341X, 15410420}
}

@inproceedings{Hyttinen13SAT,
  author       = {Antti Hyttinen and Patrik O. Hoyer and Frederick Eberhardt and Matti J{\"{a}}rvisalo},
  title        = {Discovering Cyclic Causal Models with Latent Variables: {A} General SAT-Based Procedure},
  booktitle    = {Proceedings of the Twenty-Ninth Conference on Uncertainty in Artificial Intelligence, {UAI} 2013},
  year         = {2013},
  url          = {https://dslpitt.org/uai/displayArticleDetails.jsp?mmnu=1\&smnu=2\&article\_id=2391\&proceeding\_id=29}
}

@inproceedings{hyttinen2014ASP,
  author       = {Antti Hyttinen and Frederick Eberhardt and Matti J{\"{a}}rvisalo},
  title        = {{Constraint-based Causal Discovery: Conflict Resolution with Answer Set Programming}},
  booktitle    = {Proceedings of the Thirtieth Conference on Uncertainty in Artificial Intelligence, UAI 2014},
  pages        = {340--349},
  year         = {2014},
  url          = {https://www.cs.helsinki.fi/u/mjarvisa/papers/hyttinen-eberhardt-jarvisalo.uai14.pdf}
}

@inproceedings{Hyttinen2017core,
  author       = {Antti Hyttinen and Paul Saikko and Matti J{\"{a}}rvisalo},
  title        = {A Core-Guided Approach to Learning Optimal Causal Graphs},
  booktitle    = {Proceedings of the Twenty-Sixth International Joint Conference on Artificial Intelligence, {IJCAI} 2017},
  pages        = {645--651},
  year         = {2017},
  doi          = {10.24963/ijcai.2017/90}
}

@inproceedings{leofante2024contestable,
  author       = {Francesco Leofante and Hamed Ayoobi and Adam Dejl and Gabriel Freedman and Deniz Gorur and Junqi Jiang and Guilherme Paulino{-}Passos and Antonio Rago and Anna Rapberger and Fabrizio Russo and Xiang Yin and Dekai Zhang and Francesca Toni},
  title        = {Contestable {AI} Needs Computational Argumentation},
  booktitle    = {Proceedings of the 21st International Conference on Principles of Knowledge Representation and Reasoning, {KR} 2024},
  year         = {2024},
  url          = {https://doi.org/10.24963/kr.2024/83}
}

@article{Liffiton08MUS,
  author       = {Liffiton, Mark H. and Sakallah, Karem A.},
  title        = {Algorithms for computing minimal unsatisfiable subsets of constraints},
  journal      = {Journal of Automated Reasoning},
  volume       = {40},
  number       = {1},
  pages        = {1--33},
  year         = {2008},
  doi          = {10.1007/s10817-007-9084-z}
}

@book{Lifschitz2019ASP,
  author       = {Lifschitz, Vladimir},
  title        = {Answer Set Programming},
  publisher    = {Springer},
  address      = {Cham},
  year         = {2019},
  doi          = {10.1007/978-3-030-24658-7},
  isbn         = {978-3-030-24658-7}
}

@inproceedings{magliacane2016ancestral,
  author       = {Sara Magliacane and Tom Claassen and Joris M. Mooij},
  title        = {Ancestral Causal Inference},
  booktitle    = {Advances in Neural Information Processing Systems 29: Annual Conference on Neural Information Processing Systems, NeurIPS 2016},
  pages        = {4466--4474},
  year         = {2016},
  url          = {https://proceedings.neurips.cc/paper/2016/hash/f3d9de86462c28781cbe5c47ef22c3e5-Abstract.html}
}

@article{MorenoCentenoKarp2013IHS,
  author       = {Erick Moreno-Centeno and Richard M. Karp},
  title        = {The Implicit Hitting Set Approach to Solve Combinatorial Optimization Problems with an Application to Multigenome Alignment},
  journal      = {Operations Research},
  volume       = {61},
  number       = {2},
  pages        = {453--468},
  year         = {2013},
  doi          = {10.1287/opre.1120.1139}
}

@book{pearl1988probabilisticreasoning,
  author       = {Judea Pearl},
  title        = {Probabilistic reasoning in intelligent systems - networks of plausible inference},
  publisher    = {Morgan Kaufmann},
  year         = {1989},
  doi          = {10.1016/C2009-0-27609-4},
  isbn         = {978-0-08-051489-5}
}

@book{pearl2009causality,
  author       = {Pearl, Judea},
  title        = {Causality},
  publisher    = {Cambridge University Press},
  edition      = {2},
  year         = {2009},
  doi          = {10.1017/CBO9780511803161}
}

@article{ramsey2017million,
  author       = {Joseph D. Ramsey and Madelyn Glymour and Ruben Sanchez{-}Romero and Clark Glymour},
  title        = {A million variables and more: the Fast Greedy Equivalence Search algorithm for learning high-dimensional graphical causal models, with an application to functional magnetic resonance images},
  journal      = {Int. J. Data Sci. Anal.},
  volume       = {3},
  number       = {2},
  pages        = {121--129},
  year         = {2017},
  url          = {https://doi.org/10.1007/s41060-016-0032-z}
}

@inproceedings{RapbergerUT24,
  author       = {Anna Rapberger and Markus Ulbricht and Francesca Toni},
  title        = {On the Correspondence of Non-flat Assumption-based Argumentation and Logic Programming with Negation as Failure in the Head},
  booktitle    = {Proceedings of the 22nd International Workshop on Nonmonotonic Reasoning {(NMR} 2024)},
  series       = {{CEUR} Workshop Proceedings},
  volume       = {3835},
  pages        = {112--121},
  year         = {2024},
  url          = {https://ceur-ws.org/Vol-3835/paper12.pdf}
}

@article{Reiter87Diagnosis,
  author       = {Reiter, Raymond},
  title        = {A theory of diagnosis from first principles},
  journal      = {Artificial Intelligence},
  volume       = {32},
  number       = {1},
  pages        = {57--95},
  year         = {1987},
  doi          = {10.1016/0004-3702(87)90062-2}
}

@inproceedings{russo2024abacd,
  author       = {Fabrizio Russo and Anna Rapberger and Francesca Toni},
  title        = {Argumentative Causal Discovery},
  booktitle    = {Proceedings of the 21st International Conference on Principles of Knowledge Representation and Reasoning, {KR} 2024},
  pages        = {938–949},
  year         = {2024},
  doi          = {10.24963/KR.2024/88}
}

@misc{Scutari14bnlearn,
  author       = {Marco Scutari},
  title        = {Bayesian Network Repository},
  year         = {2014},
  howpublished = "\url{http://www.bnlearn.com/bnrepository}"
}

@inproceedings{shachter1998bayesball,
  author       = {Shachter, Ross D.},
  title        = {Bayes-Ball: Rational Pastime (for Determining Irrelevance and Requisite Information in Belief Networks and Influence Diagrams)},
  booktitle    = {Proceedings of the Fourteenth Conference on Uncertainty in Artificial Intelligence, {UAI} 1998},
  publisher    = {Morgan Kaufmann},
  pages        = {480--487},
  year         = {1998}
}

@book{spirtes2000causation,
  author       = {Peter Spirtes and Clark Glymour and Richard Scheines},
  title        = {Causation, Prediction, and Search, Second Edition},
  publisher    = {{MIT} Press},
  series       = {Adaptive computation and machine learning},
  year         = {2000},
  url          = {https://www.cs.cmu.edu/afs/andrew/scs/cs/15-381/archive/OldFiles/lib/cvsub/.g/group/sdss/.g/group2/g/scottd/fullbook.pdf},
  isbn         = {978-0-262-19440-2}
}

@inproceedings{triantafillou2010learning,
  author       = {Sofia Triantafilou and Ioannis Tsamardinos and Ioannis G. Tollis},
  title        = {Learning Causal Structure from Overlapping Variable Sets},
  booktitle    = {Proceedings of the Thirteenth International Conference on Artificial Intelligence and Statistics, {AISTATS} 2010},
  volume       = {9},
  pages        = {860--867},
  year         = {2010},
  url          = {http://proceedings.mlr.press/v9/triantafillou10a.html}
}

@article{triantafillou2015COMBINE,
  author       = {Sofia Triantafillou and Ioannis Tsamardinos},
  title        = {Constraint-based Causal Discovery from Multiple Interventions over Overlapping Variable Sets},
  journal      = {Journal of Machine Learning Research},
  volume       = {16},
  number       = {66},
  pages        = {2147--2205},
  year         = {2015},
  url          = {http://jmlr.org/papers/v16/triantafillou15a.html}
}

@article{tsamardinos2006MMHC,
  author       = {Ioannis Tsamardinos and Laura E. Brown and Constantin F. Aliferis},
  title        = {The max-min hill-climbing Bayesian network structure learning algorithm},
  journal      = {Machine Learning},
  volume       = {65},
  number       = {1},
  pages        = {31--78},
  year         = {2006},
  doi          = {10.1007/s10994-006-6889-7}
}

@inproceedings{Wei2026QACD,
  author       = {Sheng Wei and Yulin Chen and Beishui Liao},
  title        = {{Causal Discovery as Dialectical Aggregation: A Quantitative Argumentation Framework}},
  booktitle    = {Proceedings of the 23rd International Conference on Principles of Knowledge Representation and Reasoning},
  pages        = {665--675},
  year         = {2026},
  doi          = {10.24963/kr.2026/63},
  url          = {https://proceedings.kr.org/2026/63/}
}

@article{wilcoxon1945individual,
  author       = {Frank Wilcoxon},
  title        = {Individual Comparisons by Ranking Methods},
  journal      = {Biometrics Bulletin},
  volume       = {1},
  number       = {6},
  pages        = {80--83},
  year         = {1945},
  doi          = {10.2307/3001968}
}

@article{wilks1935likelihood,
  author       = {S. S. Wilks},
  title        = {The Likelihood Test of Independence in Contingency Tables},
  journal      = {The Annals of Mathematical Statistics},
  volume       = {6},
  number       = {4},
  pages        = {190--196},
  year         = {1935},
  doi          = {10.1214/aoms/1177732564}
}

@article{Zhalama2017SAT,
  author       = {Zhalama and Jiji Zhang and Frederick Eberhardt and Wolfgang Mayer},
  title        = {SAT-Based Causal Discovery under Weaker Assumptions},
  journal      = {Proceedings of the Thirty-Third Conference on Uncertainty in Artificial Intelligence, {UAI} 2017},
  year         = {2017},
  url          = {http://auai.org/uai2017/proceedings/papers/234.pdf}
}

@inproceedings{Zhang2019ASP,
  author       = {Zhalama and Jiji Zhang and Frederick Eberhardt and Wolfgang Mayer and Mark Junjie Li},
  title        = {ASP-based Discovery of Semi-Markovian Causal Models under Weaker Assumptions},
  booktitle    = {Proceedings of the Twenty-Eighth International Joint Conference on Artificial Intelligence, {IJCAI} 2019},
  pages        = {1488--1494},
  year         = {2019},
  url          = {https://doi.org/10.24963/ijcai.2019/206}
}

\ifOptABAIncludeSupplement
  \appendix
  \def\OptABAIntegratedSupplement{1}
  \clearpage
  \newcommand{\OptABAMainTableRef}{Table~\ref{tab:main-results}}
  \section{ASP Semantics}\label{supp:asp}

We make explicit the semantics used in the main paper. A finite ground normal rule is
\[
h\ \texttt{:-}\ a_1,\ldots,a_m,\texttt{not }b_1,\ldots,
\texttt{not }b_n,
\]
where the head $h$ may be absent. Its body is true in an interpretation $I$ iff $a_i\in I$ for every $i$ and $b_j\notin I$ for every $j$. The rule is satisfied iff its body is false or $h\in I$; when the head is absent, satisfaction therefore requires the body to be false.

For a program $P$ and interpretation $I$, the Gelfond--Lifschitz reduct $P^I$ deletes each rule having some negative body literal $\texttt{not }b$ with $b\in I$ and removes all remaining negative body literals~\citep{gel88}. Then $I$ is a stable model of $P$ iff it is a subset-minimal model of $P^I$. Non-ground safe programs are interpreted through their finite ground instantiation. The encoding also uses ground choice and disjunctive constructs, including the inherited default-negated head literal in Listing~\ref{lst:caba-graph}, under their \textit{clingo} stable-model semantics~\citep{GebserKS11,DBLP:journals/tplp/CalimeriFGIKKLM20}.

A ground weak constraint $\texttt{:}\sim B.\ [w@l,\vec t]$ is active in $I$ when $B$ is true in $I$. At priority $l$, the solver sums $w$ once for each distinct activated tuple $(w,l,\vec t)$ and minimises the resulting cost vectors lexicographically by priority. OptABA-PC uses
\[
\texttt{:}\sim\ \texttt{not }o.\ [w_o@1,o]
\]
for every objective atom $o$. Including $o$ as a tuple identifier is essential: two different facts of equal weight then activate different tuples and are both counted. All release weights are strictly positive.

\section{Causal ABA Encodings}\label{supp:caba-core-encoding}

\ifdefined\OptABAIntegratedSupplement
Equation~\eqref{eq:pca-decomposition} separates the rules inherited unchanged from Causal ABA~\citep{russo2024abacd}, given in Listings~\ref{lst:caba-vocabulary}--\ref{lst:caba-acyclicity}, from the new Bayes-ball encoding in Listing~\ref{lst:bayes-ball}.
\else
The main paper defines $\aspP_{\mathrm{CA}}$ by the decomposition
\begin{equation}\label{eq:pca-decomposition}
\aspP_{\mathrm{CA}}(\mathbf V,\aspO)
=\aspP^{\mathrm{CABA}}_{\mathrm{core}}(\mathbf V,\mathcal S_{\aspO})
\cup\aspP^{\mathrm{new}}_{\mathrm{BB}}(\mathbf V,\aspO).
\end{equation}
\fi
This section gives every rule schema in the inherited component $\aspP^{\mathrm{CABA}}_{\mathrm{core}}$. Together with Listing~\ref{lst:bayes-ball}, Listings~\ref{lst:caba-vocabulary}--\ref{lst:caba-acyclicity} completely define $\aspP_{\mathrm{CA}}$.

\subsection{Original Causal ABA encodings}

The instance supplies one fact $\texttt{var(v)}$ for every $v\in\mathbf V$ and one fact $\texttt{in(v,s)}$ for every member $v$ of each represented conditioning set $s\in\mathcal S_{\aspO}$; \texttt{empty} denotes $\emptyset$. A CI fact is \emph{valid over $\mathbf V$} when it asserts independence or dependence for distinct endpoints $X,Y\in\mathbf V$ given $\mathbf{Z}\subseteq\mathbf V\setminus\{X,Y\}$. In Listing~\ref{lst:caba-vocabulary}, lines~1--2 construct the conditioning-set domain, lines~3--4 close CI atoms under endpoint symmetry, and line~5 rejects a simultaneous independence and dependence assertion for the same valid query.
\par\noindent\begin{minipage}{\linewidth}
\begin{aspencoding}[caption={Inherited conditioning-set and CI-consistency rules.},label={lst:caba-vocabulary}]
set(S) :- in(N,S).
set(empty).
indep(X,Y,S) :- indep(Y,X,S), var(X), var(Y), set(S), X!=Y.
dep(X,Y,S) :- dep(Y,X,S), var(X), var(Y), set(S), X!=Y.
:- dep(X,Y,S), indep(X,Y,S), set(S), var(X), var(Y), X!=Y, not in(X,S), not in(Y,S).
\end{aspencoding}
\end{minipage}

The candidate-graph rules in Listing~\ref{lst:caba-graph} are also inherited unchanged. Line~1 chooses a directed edge in either orientation or no adjacency. Lines~2--3 make \texttt{edge/2} the symmetric adjacency relation. Line~4 exploits the fact that a direct edge cannot coexist with a valid conditional independence between its endpoints~\citep{spirtes2000causation}; lines~5--7 keep edge and arrow choices consistent and exclude a two-cycle.
\par\noindent\begin{minipage}{\linewidth}
\begin{aspencoding}[caption={Inherited candidate-DAG and adjacency rules.},label={lst:caba-graph}]
arrow(X,Y) | arrow(Y,X) | not edge(X,Y) :- var(X), var(Y), X!=Y.
edge(X,Y) :- arrow(X,Y), X!=Y, var(X), var(Y).
edge(X,Y) :- edge(Y,X), X!=Y, var(X), var(Y).
:- edge(X,Y), indep(X,Y,S), set(S), var(X), var(Y), X!=Y, not in(X,S), not in(Y,S).
:- arrow(X,Y), not edge(X,Y), var(X), var(Y).
:- edge(X,Y), not arrow(X,Y), not arrow(Y,X), var(X), var(Y).
:- arrow(Y,X), arrow(X,Y), X!=Y, var(X), var(Y).
\end{aspencoding}
\end{minipage}

Finally, lines~1--2 of Listing~\ref{lst:caba-acyclicity} define directed reachability, and line~3 excludes every directed cycle.
\par\noindent\begin{minipage}{\linewidth}
\begin{aspencoding}[caption={Inherited reachability and acyclicity rules.},label={lst:caba-acyclicity}]
dpath(X,Y) :- arrow(X,Y), X!=Y, var(X), var(Y).
dpath(X,Y) :- arrow(X,Z), dpath(Z,Y), var(X).
:- dpath(X,X), var(X).
\end{aspencoding}
\end{minipage}

Before any CI fact is enforced, the graph projections produced by Listings~\ref{lst:caba-graph}--\ref{lst:caba-acyclicity} are exactly all labelled DAGs on the fixed vertex set $\mathbf V$; enforcing CI facts subsequently filters this family. The original Causal ABA encoding continues with instance-generated simple-path witnesses and \texttt{collider}/\texttt{nb} rules. Those path-enumerating rules are precisely the component replaced by $\aspP^{\mathrm{new}}_{\mathrm{BB}}$. They are therefore not included in $\aspP^{\mathrm{CABA}}_{\mathrm{core}}$ or in $\aspP_{\mathrm{CA}}$; $\aspP^{\mathrm{new}}_{\mathrm{BB}}$ derives \texttt{ap3/3} without materialising paths and enforces the selected $\texttt{dep}/\texttt{indep}$ atoms against that result. Thus Listings~\ref{lst:caba-vocabulary}--\ref{lst:caba-acyclicity}, together with Listing~\ref{lst:bayes-ball}, display every rule affecting the graph, Bayes-ball closure, or CI enforcement.

\subsection{Novel Bayes-ball encoding}\label{supp:bb-encoding}
We instantiate Bayes-ball for the input CI queries and their symmetric counterparts: for each objective atom over $(X,Y,\mathbf Z)$, the instance supplies both $\texttt{qpair(X,Y,Z)}$ and $\texttt{qpair(Y,X,Z)}$, independently of whether that atom is selected. This symmetric query domain ensures that the enforcement constraints also check the reversed CI atoms derived by Listing~\ref{lst:caba-vocabulary}. The facts $\texttt{in(N,Z)}$ represent membership $N\in\mathbf Z$. Listing~\ref{lst:bayes-ball} gives the complete Bayes-ball encoding. Lines~2--4 compute whether a node belongs to, or has a descendant in, $\mathbf Z$, and whether it has at least two distinct parents. The predicate \texttt{is\_collider} tests only this two-parent condition; collider status on a particular path remains path-relative. The state $\texttt{bb\_u(X,N,Z)}$ records arrival at $N$ from a child of $N$; $\texttt{bb\_d(X,N,Z)}$ records arrival from a parent. Lines~5--6 initialise both states at the source.
\par\noindent\begin{minipage}{\linewidth}
\begin{aspencoding}[caption={Bayes-ball encoding $\aspP^{\mathrm{new}}_{\mathrm{BB}}$ of OptABA-PC.},label={lst:bayes-ball}]
qpair(X,Y,Z). % generated for queried CI facts
anc_obs(N,Z) :- in(N,Z).
anc_obs(N,Z) :- dpath(N,M), in(M,Z).
is_collider(N) :- arrow(P,N), arrow(Q,N), P<Q.
bb_u(X,X,Z) :- qpair(X,_,Z).
bb_d(X,X,Z) :- qpair(X,_,Z).
bb_u(X,P,Z) :- bb_u(X,N,Z), not in(N,Z), arrow(P,N).
bb_u(X,P,Z) :- bb_d(X,N,Z), is_collider(N), anc_obs(N,Z), arrow(P,N).
bb_d(X,C,Z) :- bb_u(X,N,Z), not in(N,Z), arrow(N,C).
bb_d(X,C,Z) :- bb_d(X,N,Z), not in(N,Z), arrow(N,C).
ap3(X,Y,Z) :- qpair(X,Y,Z), bb_u(X,Y,Z), not in(X,Z), not in(Y,Z).
ap3(X,Y,Z) :- qpair(X,Y,Z), bb_d(X,Y,Z), not in(X,Z), not in(Y,Z).
:- dep(X,Y,Z),   not ap3(X,Y,Z).
:- indep(X,Y,Z),     ap3(X,Y,Z).
\end{aspencoding}
\end{minipage}
The four propagation rules in Listing~\ref{lst:bayes-ball} distinguish both arrival direction and destination. Line~7 sends a ball that arrived from a child through an unobserved $N$ to each parent of $N$, while line~9 sends the same state to each child. Line~8 sends a ball that arrived from a parent out through the parents of a collider that is itself in, or has a descendant in, $\mathbf{Z}$; this is the contracted opened-collider transition justified in Appendix~\ref{supp:bayes} (Lemma~\ref{supp:bb-simulation}). Line~10 continues a parent-arriving ball to each child of an unobserved $N$. All other local cases are blocked and therefore have no propagation rule. Lines~11--12 derive $\texttt{ap3}$ when either state reaches the target, witnessing $d$-connection for a valid query. Line~13 requires this witness for an enforced dependence, while line~14 forbids it for an enforced independence, preserving the Causal ABA correspondence between selected CI facts and graph semantics.

\section{Correctness of the Implemented Bayes-Ball Rules}\label{supp:bayes}

\def\OptABAEncodingSource{Listing~\ref{lst:bayes-ball}}

To prove the Bayes-ball correctness lemma recalled below, fix the objects in its statement: a finite DAG $G=(V,E)$ and a valid query $(X,Y,\mathbf{Z})$ with distinct endpoints $X,Y\notin\mathbf{Z}$. Let \texttt{Z} denote the ASP identifier encoding $\mathbf{Z}$, so that \texttt{in(v,Z)} holds exactly when $v\in\mathbf{Z}$. For any $A\subseteq V$, let
\[
\operatorname{An}_G(A)
  :=\{v\in V\mid v\leadsto_G a\text{ for some }a\in A\},
\]
where $v\leadsto_G a$ denotes a directed path in $G$, including one of length zero. Thus $A\subseteq\operatorname{An}_G(A)$. Once $G$ and $\mathbf{Z}$ are fixed, the recursive definitions of \texttt{dpath/2}, \texttt{anc\_obs/2}, \texttt{bb\_u/3}, \texttt{bb\_d/3}, and \texttt{ap3/3} form finite positive strata and therefore have a least fixed point. Lines~2--3 of \OptABAEncodingSource{} give
\[
\texttt{anc\_obs(v,Z)}\quad\Longleftrightarrow\quad v\in\operatorname{An}_G(\mathbf{Z}),
\]
and line~4 defines $\texttt{is\_collider}(v)$ to hold exactly when $v$ has at least two distinct parents in $G$.

The atom \texttt{bb\_u(X,v,Z)} records arrival at $v$ from a child, and \texttt{bb\_d(X,v,Z)} records arrival from a parent. With $X$ and the identifier \texttt{Z} for $\mathbf{Z}$ fixed, write $u(v)$ for \texttt{bb\_u(X,v,Z)} and $d(v)$ for \texttt{bb\_d(X,v,Z)}. Lines~5--6 of \OptABAEncodingSource{} initialise both states at $X$, and lines~7--10 close them under:
\begin{align}
u(v),\ v\notin\mathbf{Z},\ p\to v &\Longrightarrow u(p),\label{eq:bb-child-to-parent}\\
u(v),\ v\notin\mathbf{Z},\ v\to c &\Longrightarrow d(c),\label{eq:bb-child-to-child}\\
d(v),\ v\notin\mathbf{Z},\ v\to c &\Longrightarrow d(c),\label{eq:bb-parent-to-child}\\
d(v),\ v\in\operatorname{An}_G(\mathbf{Z}),\ \texttt{is\_collider}(v),\ p\to v
 &\Longrightarrow u(p).\label{eq:bb-open-collider}
\end{align}
Equations~\eqref{eq:bb-child-to-parent}, \eqref{eq:bb-child-to-child}, \eqref{eq:bb-parent-to-child}, and~\eqref{eq:bb-open-collider} are, respectively, lines~7, 9, 10, and~8 of \OptABAEncodingSource{}.

\paragraph{Shachter's operational characterisation.} A structured belief network in~\citet{shachter1998bayesball} is $(V,E,F)$, where $F$ is the set of nodes whose variables are deterministic functions of their parents. Such a node need not be a root or a constant; $F=\emptyset$ here means that our DAG query supplies no functional-determinacy annotations. Apply Shachter's Algorithm~2 with source set $J=\{X\}$ and observed set $K=\mathbf{Z}$. Algorithm~2 schedules each node for processing according to whether the ball arrives from a child or a parent; sending a ball to its parents marks its top, while sending a ball to its children marks its bottom~\citep[Alg.~2]{shachter1998bayesball}. These marks are not the ASP states themselves: $u(v)$ and $d(v)$ encode the two arrival directions. For example, at an unobserved $v$, line~7 of \OptABAEncodingSource{} sends $u(v)$ to its parents and corresponds to a top mark, whereas lines~9--10 send $u(v)$ or $d(v)$ to its children and correspond to a bottom mark. Because $F=\emptyset$, either arrival direction at an unobserved node causes a bottom mark. Lemma~\ref{supp:shachter} states the specialisation of Shachter's irrelevance theorem needed below.

\begin{supplemma}[Shachter specialisation]\label{supp:shachter}
For every $v\notin\mathbf{Z}$, the execution of Shachter's Algorithm~2 marks $v$ on the bottom iff there is an active path from $X$ to $v$ given $\mathbf{Z}$, equivalently iff $X$ and $v$ are $d$-connected given $\mathbf{Z}$.
\end{supplemma}
\begin{proof}
Apply Theorem~2 of~\citet{shachter1998bayesball} with the set-valued parameters $J=\{X\}$, $K=\mathbf{Z}$, and $L=\{v\}$, and with $F=\emptyset$. Step~4 of Algorithm~2 defines $N_i(J\mid K)$ as exactly the nodes not marked on the bottom, while Theorem~2 identifies $L\subseteq N_i(J\mid K)$ with irrelevance of $X_J$ to $X_L$ given $X_K$. Hence $v$ is bottom-marked iff it is not irrelevant to $X$ given $\mathbf{Z}$. Shachter's Theorem~1 identifies irrelevance with the absence of an active path; when $F=\emptyset$, this is ordinary $d$-separation. Thus $v$ is bottom-marked iff an active path exists, equivalently iff $X$ and $v$ are $d$-connected given $\mathbf{Z}$.
\end{proof}

The ASP rules are deliberately a collapsed version of Shachter's schedule. Lemma~\ref{supp:bb-simulation} states the simulation property that must be proved for our implementation.

\begin{supplemma}[ASP--Bayes-ball simulation]\label{supp:bb-simulation}
For every $v\notin\mathbf{Z}$, either $u(v)$ or $d(v)$ belongs to the ASP closure iff the execution of Shachter's Algorithm~2 marks $v$ on the bottom.
\end{supplemma}
\begin{proof}
A state $u(v)$ represents Shachter's Algorithm~2 scheduling $v$ as reached from a child, and $d(v)$ represents scheduling it as reached from a parent. At an unobserved probabilistic node, Steps~3(c)(i)--(ii) send a child-arriving ball to parents and children, and Step~3(d)(ii) sends a parent-arriving ball to children. These transitions are exactly Equations~\eqref{eq:bb-child-to-parent}--\eqref{eq:bb-parent-to-child}. The remaining ASP transition, Equation~\eqref{eq:bb-open-collider}, contracts the detour used in the proof of Shachter's Theorem~2: from an ancestor $n$, Algorithm~2 can travel down a directed path to an observed $z\in\mathbf{Z}$, bounce there under Step~3(d)(i), return to $n$, and leave through a parent of $n$. Such a detour exists exactly when $n\in\operatorname{An}_G(\mathbf{Z})$.

\noindent$({\Rightarrow})$ First omit the additional seed $d(X)$ from line~6 of \OptABAEncodingSource{}. Because $X\notin\mathbf Z$, every propagation consequence of this seed already follows from $u(X)$: its child transitions are available from $u(X)$, as are any parent transitions. Omitting the seed therefore does not change which vertices are reached, and $X$ itself remains reached by $u(X)$. For the remaining derivations, induction shows that each reached state can be scheduled by Shachter's Algorithm~2 with the corresponding arrival direction. The base state $u(X)$ matches Algorithm~2, Step~2. Each application of Equations~\eqref{eq:bb-child-to-parent}--\eqref{eq:bb-parent-to-child} is an ordinary transition. Equation~\eqref{eq:bb-open-collider} expands into the directed detour to an observed descendant and back described above, using the first observed node on that path (or bouncing immediately if the current node is observed). Thus every ASP-reached unobserved vertex is scheduled and, since $F=\emptyset$, bottom-marked.

\noindent$({\Leftarrow})$ Suppose Shachter's Algorithm~2 bottom-marks an unobserved vertex $v$. If $v=X$, then $u(X)$ is already initial. Otherwise, Lemma~\ref{supp:shachter} supplies a simple active path from $X$ to $v$. The initial state $u(X)$ can traverse its first edge in either direction because $X\notin\mathbf Z$. Follow the path inductively. At each non-collider, activity gives a vertex outside $\mathbf Z$, and the arrival and departure directions are exactly one of the three ordinary transitions in Equations~\eqref{eq:bb-child-to-parent}--\eqref{eq:bb-parent-to-child}. At each collider, the path enters and leaves along two distinct incoming edges, so \texttt{is\_collider} holds; activity also places the collider in $\operatorname{An}_G(\mathbf Z)$. Equation~\eqref{eq:bb-open-collider} therefore traverses that step. The final step derives either $u(v)$ or $d(v)$. This proves the required equivalence for reached vertices; it does not require the two procedures to generate identical arrival states.
\end{proof}

We now recall the Bayes-ball correctness lemma from the main paper.
\bayesballcorrectness*
\begin{proof}
Because $V$ is finite, grounding the propagation rules in \OptABAEncodingSource{} for the fixed \texttt{arrow}, \texttt{in}, and \texttt{qpair(X,Y,Z)} atoms gives a finite program whose positive recursive strata have the least fixed point described above. The endpoint guards \texttt{not in(X,Z)} and \texttt{not in(Y,Z)} on lines~11--12 hold because validity gives $X,Y\notin\mathbf{Z}$.

\noindent$({\Rightarrow})$ If $\texttt{ap3(X,Y,Z)}$ is derived, one of the endpoint rules on lines~11--12 of \OptABAEncodingSource{} has reached $Y$ in state $u(Y)$ or $d(Y)$. Lemma~\ref{supp:bb-simulation} gives a bottom mark at $Y$ in Shachter's Algorithm~2, and Lemma~\ref{supp:shachter} yields $d$-connection given $\mathbf{Z}$.

\noindent$({\Leftarrow})$ If $X$ and $Y$ are $d$-connected given $\mathbf{Z}$, Lemma~\ref{supp:shachter} gives a bottom mark at $Y$ in Shachter's Algorithm~2. Lemma~\ref{supp:bb-simulation} then derives $u(Y)$ or $d(Y)$, so the corresponding endpoint rule on line~11 or~12 of \OptABAEncodingSource{} derives $\texttt{ap3(X,Y,Z)}$.
\end{proof}

Let $\aspP_{\mathrm{CA}}$ be the finite ground program defined component-wise in Equation~\eqref{eq:pca-decomposition}: the inherited core in Listings~\ref{lst:caba-vocabulary}--\ref{lst:caba-acyclicity}, together with the Bayes-ball encoding in \OptABAEncodingSource{}. The objective-atom choices and weak constraints are not part of $\aspP_{\mathrm{CA}}$; they are introduced by $\enforce$ and $\adorned$. Recall from the main paper that, for a stable model $M$ of a Causal ABA instance on variables $V$, its \emph{graph projection} is
\[
\pi_G(M)=(V,E_M),\qquad
E_M=\{(x,y)\in V^2\mid\texttt{arrow}(x,y)\in M\}.
\]
Unlike $\rho(M)$, which records released objective atoms, $\pi_G(M)$ discards all non-edge atoms.

We now recall the graph-fidelity corollary from the main paper.
\graphfidelity*
\begin{proof}
Fix the finite sets $\mathbf V$ and $\aspO$ and the program $\aspP_{\mathrm{CA}}(\mathbf V,\aspO)$ from Equation~\eqref{eq:pca-decomposition}.

\noindent$({\subseteq})$ Let $M\in AS(\enforce(\aspP_{\mathrm{CA}},T,\aspO))$. Lines~1 and~5--7 of Listing~\ref{lst:caba-graph} choose at most one orientation per unordered vertex pair, and line~3 of Listing~\ref{lst:caba-acyclicity} rejects every directed cycle; hence $\pi_G(M)\in\operatorname{DAG}(\mathbf V)$. For every enforced dependence fact, line~13 of \OptABAEncodingSource{} requires $\texttt{ap3}$; for every enforced independence fact, line~14 forbids $\texttt{ap3}$. Lemma~\ref{lem:bayes-ball} therefore gives $\pi_G(M)\models T$.

\noindent$({\supseteq})$ Fix a DAG $G$ on $\mathbf V$ satisfying $T$. Proposition~3.3 of~\citet{russo2024abacd} and the stable ABA--ASP correspondence~\citep{RapbergerUT24} supply stable choices of the inherited graph generator whose \texttt{arrow} atoms are exactly the edges of $G$. Choose every objective atom in $T$ and close these choices under CI symmetry; no atom outside this symmetric closure needs to be selected. The remaining finite positive or stratified rules have a least closure. Every derived CI fact is either in $T$ or its endpoint reversal, so $G$ satisfies it. The symmetric query domain ensures that Lemma~\ref{lem:bayes-ball} applies to each such fact in the enforcement constraints. Acyclicity holds because $G$ is a DAG, and an enforced independence cannot join adjacent endpoints because a direct edge would be an active path. Thus all graph and CI constraints hold. Constraint filtering~\eqref{eq:constraint-filtering} preserves the resulting stable model after the constraints enforcing $T$ are added, giving some $M\in AS(\enforce(\aspP_{\mathrm{CA}},T,\aspO))$ with $\pi_G(M)=G$.
\end{proof}

\section{Optimal-MCS}\label{supp:projection}

\begin{example}[Optimal MCS]\label{ex:opt-mcs-appendix}
    Let $\aspP$ and $\aspO$ be the program and objective atoms from Ex.~\ref{ex:mus}, respectively. Then $MCS(\aspP,\aspO)=\{M_1,M_2,M_3\}$, with $M_1=\{\texttt{o(1)},\texttt{o(4)}\}$, $M_2=\{\texttt{o(2)},\texttt{o(4)}\}$, and $M_3=\{\texttt{o(3)},\texttt{o(4)}\}$. Let $\aspW=\{(\texttt{o(1)},\texttt{5}),(\texttt{o(2)},\texttt{10}),(\texttt{o(3)},\texttt{15}),(\texttt{o(4)},\texttt{20})\}$. Then $\mathcal{C}(M_1,\aspW)=25$, $\mathcal{C}(M_2,\aspW)=30$, and $\mathcal{C}(M_3,\aspW)=35$, hence $M_1$ is the unique optimal MCS by Def.~\ref{def:opt-mcs}.
\end{example}

\begin{example}[MCSes and Weak Constraints]\label{ex:opt-mcs-wc}
    Let $\aspP$ and $\aspO$ be the program and objective atoms from Ex.~\ref{ex:mus}, respectively, and let $\aspW$ be as in Ex.~\ref{ex:opt-mcs-appendix}. Then $\adorned(\aspP,\aspO,\aspW)$ is:
    \begin{lstlisting}[escapechar=|]
{o(1);o(2);o(3);o(4)}.
a :- not b, o(1).  b :- not c, o(2).  c :- not a, o(3).  :- not d, o(4).
:|$\sim$| not o(1). [5@1,o(1)]   :|$\sim$| not o(2). [10@1,o(2)]
:|$\sim$| not o(3). [15@1,o(3)]  :|$\sim$| not o(4). [20@1,o(4)]
    \end{lstlisting}
    An optimal stable model of $\adorned(\aspP,\aspO,\aspW)$ is $M=\{\texttt{o(2)},\texttt{o(3)},\texttt{c}\}$, making $\aspO\setminus M=\{\texttt{o(1)},\texttt{o(4)}\}$ an optimal MCS for $\aspP$ w.r.t.~$\aspO,\aspW$.
\end{example}

\subsection{Proofs}

Throughout this section, $\aspP$ is a finite ground program, $\aspO$ is a finite set of objective atoms, and $\aspW=\{(o,w_o)\mid o\in\aspO,\ w_o>0\}$ assigns one strictly positive release weight to every objective atom.

Let
\[
A(\aspP,\aspO)=\aspP\cup\{\{o\}.\mid o\in\aspO\}.
\]
For $S\subseteq\aspO$, $\enforce(\aspP,S,\aspO)$ adds to $A(\aspP,\aspO)$ the constraints $\texttt{:- not }o$ for $o\in S$. For a stable model $M$ of $A(\aspP,\aspO)$, its \emph{released-fact projection} is $\rho(M)=\aspO\setminus M$: it records exactly the objective atoms that are false in $M$.

\ifdefined\OptABAIntegratedSupplement
\else
For standalone reading, write $AS(\aspP)$ for the stable models of $\aspP$. A set $U\subseteq\aspO$ is unsatisfiable when $\enforce(\aspP,U,\aspO)$ is incoherent, and is a MUS when it is inclusion-minimal with this property. A set $R\subseteq\aspO$ is a correction set when $\enforce(\aspP,\aspO\setminus R,\aspO)$ is coherent, and is a minimal correction set (MCS) when it is inclusion-minimal. We write $CS(\aspP,\aspO)$, $MUS(\aspP,\aspO)$, and $MCS(\aspP,\aspO)$ for the corresponding families.

Let $\aspW=\{(o,w_o)\mid o\in\aspO,\ w_o>0\}$ and define the release cost $\mathcal C(R,\aspW)=\sum_{o\in R}w_o$. An optimal MCS minimises this cost over $MCS(\aspP,\aspO)$. Finally, $\adorned(\aspP,\aspO,\aspW)$ is $A(\aspP,\aspO)$ together with, for every $(o,w_o)\in\aspW$, the weak constraint
\[
\texttt{:}\!\sim\ \texttt{not }o.\ [w_o@1,o].
\]
\fi

We first recall the positive-cost lemma from the main paper.
\positivecostminimality*
\begin{proof}
Let $R$ be a minimum-cost member of $CS(\aspP,\aspO)$. If a strict subset $R'\subset R$ were also a correction set, then $\mathcal C(R',\aspW)<\mathcal C(R,\aspW)$ because every weight in the nonempty difference $R\setminus R'$ is strictly positive, contradicting the choice of $R$. Thus $R\in MCS(\aspP,\aspO)$. Since $MCS(\aspP,\aspO)\subseteq CS(\aspP,\aspO)$, the minimum cost of $R$ over all correction sets is no greater than that of any MCS; hence $R$ is an optimal MCS with respect to $\aspW$.
\end{proof}
Consequently, whenever $CS(\aspP,\aspO)\neq\emptyset$, minimising over correction sets has the same value as minimising over MCSes.

We use the standard theorem that constraints filter, but do not create, stable models. For every finite ground program $Q$ and finite set of constraints $D$,
\begin{equation}
AS(Q\cup D)=\{M\in AS(Q)\mid M\models D\},
\label{eq:constraint-filtering}
\end{equation}
where $M\models D$ means that $M$ satisfies every constraint in $D$~\citep[Sec.~5.5, Theorem on Constraints]{Lifschitz2019ASP}.

\begin{supplemma}[Exact MCS realisability]\label{supp:exact-realisability}
Let $\aspP$ be a finite ground program, let $\aspO$ be finite, and let $\aspW=\{(o,w_o)\mid o\in\aspO,\ w_o>0\}$. For every $R\in MCS(\aspP,\aspO)$ there exists a stable model $M_R$ of $A(\aspP,\aspO)$ such that $\rho(M_R)=R$. The same $M_R$ is a stable model of $\adorned(\aspP,\aspO,\aspW)$, and its weak-constraint cost is exactly $\mathcal C(R,\aspW)$.
\end{supplemma}
\begin{proof}
Because $R$ is a correction set, $\enforce(\aspP,\aspO\setminus R,\aspO)$ is coherent; choose one of its stable models $M_R$. This program is $A(\aspP,\aspO)$ together with the constraints enforcing $\aspO\setminus R$. Constraint filtering~\eqref{eq:constraint-filtering} shows that removing those constraints leaves $M_R$ stable for $A(\aspP,\aspO)$, while the constraints themselves give $\aspO\setminus R\subseteq M_R$.

Suppose some $r\in R$ also belonged to $M_R$. Then $M_R$ would satisfy the additional constraint $\texttt{:- not }r$. Constraint filtering~\eqref{eq:constraint-filtering} shows that adding this satisfied constraint preserves stability. But the resulting program is $\enforce(\aspP,(\aspO\setminus R)\cup\{r\},\aspO)$, so $R\setminus\{r\}$ would be a correction set, contradicting minimality of $R$. Therefore $M_R\cap\aspO=\aspO\setminus R$ and $\rho(M_R)=R$.

Weak constraints do not alter the stable models of $A(\aspP,\aspO)$; they only rank them. Exactly those indexed by atoms in $R$ are active in $M_R$, so its cost is $\sum_{r\in R}w_r=\mathcal C(R,\aspW)$.
\end{proof}

We next recall the main-paper projection theorem.
\projectioncorrespondence*
\begin{proof}
\noindent$({\Rightarrow})$ Let $M$ be an optimal stable model of $\adorned(\aspP,\aspO,\aspW)$. Weak constraints do not alter stable models, so $M$ is stable for $A(\aspP,\aspO)$. Let $S=M\cap\aspO=\aspO\setminus\rho(M)$. The model $M$ satisfies every constraint enforcing an atom in $S$; hence constraint filtering~\eqref{eq:constraint-filtering} makes $M$ stable for $\enforce(\aspP,S,\aspO)$. Therefore $\rho(M)$ is a correction set. Moreover, precisely the weak constraints indexed by $\rho(M)$ are active, so the model cost is $\mathcal C(\rho(M),\aspW)$.

Suppose a correction set $Q$ had lower cost. Since $\aspO$ is finite, $Q$ contains an MCS $R\subseteq Q$. Positivity of the weights gives
\[
\mathcal C(R,\aspW)\leq\mathcal C(Q,\aspW)
 <\mathcal C(\rho(M),\aspW).
\]
Lemma~\ref{supp:exact-realisability} supplies a stable model $M_R$ whose weak-constraint cost is $\mathcal C(R,\aspW)$, contradicting optimality of $M$. Thus $\rho(M)$ is a minimum-cost correction set and, by Lemma~\ref{lem:positive-minimality}, an optimal MCS.

\noindent$({\Leftarrow})$ Let $R$ be an optimal MCS. Lemma~\ref{supp:exact-realisability} supplies a stable model $M_R$ with $\rho(M_R)=R$ and cost $\mathcal C(R,\aspW)$. If $M_R$ were not optimal, finiteness of the ground program would let us choose a globally optimal stable model $N$ of strictly lower cost. The already proved $(\Rightarrow)$ direction now applies to this optimal $N$: $\rho(N)$ is an MCS and its cost equals the cost of $N$, so
\[
\mathcal C(\rho(N),\aspW)<\mathcal C(R,\aspW),
\]
contradicting optimality of $R$. Hence $M_R$ is optimal and realises $R$ exactly.

The two directions establish that $\rho$ maps the optimal stable models onto exactly the optimal MCSes of Definition~\ref{def:opt-mcs}, preserving their costs. If $A(\aspP,\aspO)$ is incoherent, constraint filtering implies that there are no correction sets, so both families are empty. Thus $\rho$ is surjective onto the optimal MCSes in all cases. Injectivity is not guaranteed: for $\aspP=\{\texttt{a :- not b.},\ \texttt{b :- not a.}\}$, $\aspO=\{\texttt{o}\}$, and any $w_{\texttt{o}}>0$, the two optimal stable models $\{\texttt{a},\texttt{o}\}$ and $\{\texttt{b},\texttt{o}\}$ both project to the unique optimal MCS $\emptyset$.
\end{proof}

The optimal-repair result also yields a local justification for every released fact. We recall and prove the main-paper private-conflict proposition.
\privateconflictjustification*
\begin{proof}
Recall the MUS--MCS hitting-set duality stated in the main paper: $MCS(\aspP,\aspO)$ is exactly the family of inclusion-minimal hitting sets of $MUS(\aspP,\aspO)$~\citep[Thm.~5.3]{ALVIANO2023103931}. Let $R\in MCS(\aspP,\aspO)$ and $f\in R$. Since $R\setminus\{f\}$ is not a hitting set, some $U_f\in MUS(\aspP,\aspO)$ is disjoint from $R\setminus\{f\}$. Since $R$ does hit every MUS, it follows that $f\in U_f$ and therefore $U_f\cap R=\{f\}$. This proves the private-conflict condition in the proposition.
\end{proof}

The certificate is local to the contested release: every other member of $U_f$ was retained. It proves that adding $f$ back without a compensating release restores a concrete conflict.

\section{OptABA-PC Guarantees}\label{supp:optaba-guarantees}

We recall the OptABA-PC guarantees theorem from the main paper.
\optabaguarantees*
\begin{proof}
We first prove the existence of an optimal stable model. Proposition~3.3 of~\citet{russo2024abacd} establishes a stable Causal ABA extension for every DAG on the fixed vertex set $\mathbf V$. Apply it to the edgeless DAG $G_{\emptyset}=(\mathbf V,\emptyset)$, which retains all vertices and has no edges. The stable ABA--ASP correspondence~\citep{RapbergerUT24} supplies the corresponding stable choices for the inherited ASP core. With no objective facts selected, the finite positive or stratified reachability and Bayes-ball modules have their least closure, and no CI-enforcement constraint is activated. Hence $A(\aspP_{\mathrm{CA}},\aspO)=\enforce(\aspP_{\mathrm{CA}},\emptyset,\aspO)$ has a stable model. Weak constraints preserve this nonempty feasible set and only rank its members. The program is finite ground, so its attainable costs form a nonempty finite set and have a minimum; therefore $\adorned(\aspP_{\mathrm{CA}},\aspO,\aspW)$ has an optimal stable model.

Let $M$ be any optimal stable model, let $R=\rho(M)$ and $T=\aspO\setminus R=M\cap\aspO$, and abbreviate $\mathcal F_{\mathrm{CA}}=CS(\aspP_{\mathrm{CA}},\aspO)$.

\noindent\emph{1. Graph fidelity.} Weak constraints do not alter stable models, so $M$ is stable for $A(\aspP_{\mathrm{CA}},\aspO)$. Because $T=M\cap\aspO$, $M$ satisfies every constraint enforcing an atom in $T$. Constraint filtering~\eqref{eq:constraint-filtering} therefore gives $M\in AS(\enforce(\aspP_{\mathrm{CA}},T,\aspO))$. Corollary~\ref{cor:graph-fidelity} now yields $\pi_G(M)\in\operatorname{DAG}(\mathbf V)$ and $\pi_G(M)\models T$, which is item~1.

\noindent\emph{2. Repair optimality.} The $(\Rightarrow)$ direction of Theorem~\ref{thm:projection} gives $R\in MCS(\aspP_{\mathrm{CA}},\aspO)$, identifies $\mathcal C(R,\aspW)$ with the weak-constraint cost of $M$, and makes $R$ minimum-cost among all MCSes. By definition, membership in $MCS$ gives subset-minimality. Lemma~\ref{lem:positive-minimality} and the observation following its proof show that the minimum cost over all correction sets equals the minimum cost over MCSes. Therefore
\[
R\in MCS(\aspP_{\mathrm{CA}},\aspO),\qquad
\mathcal C(R,\aspW)=\min_{Q\in\mathcal F_{\mathrm{CA}}}\mathcal C(Q,\aspW),
\]
This proves item~2.

\noindent\emph{3. Conflict justification.} Fix $f\in R$. Since $R$ is an MCS, $R\setminus\{f\}$ is not a correction set. Because
\[
\aspO\setminus(R\setminus\{f\})=T\cup\{f\},
\]
$AS(\enforce(\aspP_{\mathrm{CA}},T\cup\{f\},\aspO))=\emptyset$, which is precisely the incoherence caused by reinstating $f$ while retaining $T$. Proposition~\ref{prop:private-conflict}, proved above by MUS--MCS hitting-set duality, supplies a MUS $U_f$ with $U_f\cap R=\{f\}$. Hence $f\in U_f$, while every other member of $U_f$ lies in $T=\aspO\setminus R$. Thus $U_f$ is a conflict contained in the reinstated set $T\cup\{f\}$ and certifies that incoherence, proving item~3.

Finally, let $R'$ be any optimal MCS. The $(\Leftarrow)$ direction of Theorem~\ref{thm:projection} supplies an optimal stable model $N$ of $\adorned(\aspP_{\mathrm{CA}},\aspO,\aspW)$ such that $\rho(N)=R'$ and whose weak-constraint cost is $\mathcal C(R',\aspW)$. This proves the converse clause of Theorem~\ref{thm:optaba-guarantees}. The witness need not be unique, as stated in Theorem~\ref{thm:projection}.
\end{proof}

\section{Contestability and Redress}\label{supp:contestability}

Let $\mathcal F=CS(\aspP,\aspO)$, and partition it into
\[
K_f=\{R\in\mathcal F:f\notin R\},\qquad
L_f=\{R\in\mathcal F:f\in R\}.
\]
Because a repair lists released facts, $K_f$ contains the repairs retaining $f$, whereas $L_f$ contains those releasing it. Throughout this section, $\aspO$ is finite, all release weights are positive, and $f\in\aspO$.
For hard retention, write
\[
A_f=\adorned_f(\aspP,\aspO,\aspW)
=\adorned(\aspP,\aspO,\aspW)\cup\{\texttt{:- not }f\}.
\]

We recall the two contestability properties and their OptABA-PC guarantee from the main paper.
\releasecostcontestability*
\factenforcementcontestability*
\optabacontestability*
\begin{proof}
Fix the finite sets $\mathbf V$ and $\aspO$ of the OptABA-PC instance, where every member of $\aspO$ is a valid CI fact over $\mathbf V$, let $\aspP=\aspP_{\mathrm{CA}}(\mathbf V,\aspO)$ be the unrestricted-DAG program in Equation~\eqref{eq:pca-decomposition}, let $\aspW$ assign one strictly positive release weight to every objective atom, and fix $f\in\aspO$.

\emph{Release-cost contestability.}
For either positive weight assignment $\overline{\aspW}\in\{\aspW,\aspW'\}$, both directions of Theorem~\ref{thm:projection} identify the returned projections with the minimum-cost MCSes, and Lemma~\ref{lem:positive-minimality} identifies those with all minimum-cost correction sets. Consequently,
\[
\operatorname{Rep}_{\overline{\aspW}}=\arg\min_{R\in\mathcal F}\mathcal C(R,\overline{\aspW}).
\]
Let $\delta=w'_f-w_f>0$, and abbreviate $\mathcal C(R)=\mathcal C(R,\aspW)$ and $\mathcal C'(R)=\mathcal C(R,\aspW')$. Because $\aspW'$ changes only $w_f$, $\mathcal C'(R)=\mathcal C(R)$ for $R\in K_f$, while $\mathcal C'(R)=\mathcal C(R)+\delta$ for $R\in L_f$.
If either $K_f$ or $L_f$ is empty, every repair receives the same cost change: no repair changes cost when $L_f=\emptyset$, and every repair gains $\delta$ when $K_f=\emptyset$. In either case $\operatorname{Rep}_{\aspW'}=\operatorname{Rep}_{\aspW}$, so both required inclusions are immediate. Assume henceforth that both classes are nonempty. Since $\mathcal F$ is finite, define the attained within-class minima and their complete sets of minimisers by
\[
\begin{aligned}
k&=\min_{R\in K_f}\mathcal C(R), & K_f^*&=\arg\min_{R\in K_f}\mathcal C(R),\\
\ell&=\min_{R\in L_f}\mathcal C(R), & L_f^*&=\arg\min_{R\in L_f}\mathcal C(R).
\end{aligned}
\]
The ranking inside $K_f$ is unchanged because none of its costs changes, and the ranking inside $L_f$ is unchanged because the same $\delta$ is added to every one of its costs. Hence $K_f^*$ and $L_f^*$ remain the respective within-class minimisers, while the best class costs change from $k$ and $\ell$ under $\aspW$ to $k$ and $\ell+\delta$ under $\aspW'$. The following five cases exhaust the possible comparisons.
\begin{enumerate}
    \item If $k<\ell$, then $\operatorname{Rep}_{\aspW}=K_f^*$. Since $k<\ell<\ell+\delta$, also $\operatorname{Rep}_{\aspW'}=K_f^*$.
    \item If $k=\ell$, then $\operatorname{Rep}_{\aspW}=K_f^*\cup L_f^*$. Since $k<\ell+\delta$, the increase breaks the tie in favour of retention and $\operatorname{Rep}_{\aspW'}=K_f^*$.
    \item If $\ell<k$ and $\delta<k-\ell$, then $\ell+\delta<k$, so $\operatorname{Rep}_{\aspW}=\operatorname{Rep}_{\aspW'}=L_f^*$.
    \item If $\ell<k$ and $\delta=k-\ell$, then $\ell+\delta=k$. Therefore $\operatorname{Rep}_{\aspW}=L_f^*$, while $\operatorname{Rep}_{\aspW'}=K_f^*\cup L_f^*$: the best retaining repairs join the old releasing optima at a tie.
    \item If $\ell<k$ and $\delta>k-\ell$, then $k<\ell+\delta$. Therefore $\operatorname{Rep}_{\aspW}=L_f^*$ and $\operatorname{Rep}_{\aspW'}=K_f^*$.
\end{enumerate}
Cases~1 and~2 are exactly the cases in which an old optimum retains $f$: every such repair belongs to $K_f^*$ and remains a new optimum. In Cases~3--5 there is no old retaining optimum. Since $K_f=\{R\in\mathcal F:f\notin R\}$, this establishes the first inclusion in Property~\ref{prop:release-cost}:
\[
\{R\in\operatorname{Rep}_{\aspW}:f\notin R\}=\operatorname{Rep}_{\aspW}\cap K_f\subseteq\operatorname{Rep}_{\aspW'}.
\]
Cases~3 and~4 are exactly the cases in which a new optimum releases $f$: every such repair belongs to $L_f^*$ and was already an old optimum. In Cases~1, 2, and~5 there is no new releasing optimum. Since $L_f=\{R\in\mathcal F:f\in R\}$, this establishes the second inclusion:
\[
\{R\in\operatorname{Rep}_{\aspW'}:f\in R\}=\operatorname{Rep}_{\aspW'}\cap L_f\subseteq\operatorname{Rep}_{\aspW}.
\]
Together, the two inclusions prove release-cost contestability at $f$.

\emph{Fact-enforcement contestability.}
By constraint filtering~\eqref{eq:constraint-filtering}, adding $\texttt{:- not }f$ removes precisely the stable models with $f\notin M$, equivalently those with $f\in\rho(M)$, and changes no weak-constraint cost. Thus every surviving projection lies in $K_f$.

\noindent$({\subseteq})$ Let $M$ be an optimal stable model of $A_f$, let $R=\rho(M)$, and let $S=M\cap\aspO$. Constraint filtering~\eqref{eq:constraint-filtering} shows that removing the hard constraint leaves $M$ stable for $A(\aspP,\aspO)$. Applying the same identity when adding the constraints enforcing $S$ shows that $R$ is a correction set; since $f\in M$, $R\in K_f$. Suppose some $Q\in K_f$ were cheaper. Because $\aspO$ is finite, $Q$ contains an MCS $R'\subseteq Q$, which also belongs to $K_f$. Lemma~\ref{supp:exact-realisability} supplies a stable model $M_{R'}$ with projection $R'$ and cost at most $\mathcal C(Q,\aspW)$. Because $f\notin R'$, the model satisfies the new hard constraint, and constraint filtering preserves its stability when that constraint is added. This contradicts optimality of $M$. Hence $R\in\arg\min_{Q\in K_f}\mathcal C(Q,\aspW)$.

\noindent$({\supseteq})$ Let $R\in\arg\min_{Q\in K_f}\mathcal C(Q,\aspW)$. A strict correction subset $R'\subset R$ would also lie in $K_f$ and, by positive weights, would be cheaper; hence $R$ is an MCS. Lemma~\ref{supp:exact-realisability} supplies a stable model $M_R$ with projection $R$ and cost $\mathcal C(R,\aspW)$. Since $f\notin R$, $M_R$ satisfies the hard constraint, and constraint filtering preserves stability after it is added. If a lower-cost stable model $N$ of $A_f$ existed, removing the hard constraint and then adding the constraints enforcing $N\cap\aspO$, using the two corresponding applications of~\eqref{eq:constraint-filtering}, would show that $\rho(N)\in K_f$ is a correction set. Its cost equals the lower model cost, contradicting the choice of $R$. Thus $M_R$ is optimal for $A_f$, proving the reverse inclusion.

\noindent\emph{Totality.} It remains to prove that $K_f$ is nonempty for every valid CI fact $f$. Corollary~\ref{cor:graph-fidelity} shows that $\aspP_{\mathrm{CA}}$ admits every DAG on the fixed vertex set $\mathbf V$ before CI enforcement. If $f=(X\indep Y\given\mathbf{Z})$, the edgeless (fully disconnected) DAG $(\mathbf V,\emptyset)$ satisfies $f$. If $f=(X\notindep Y\given\mathbf{Z})$, the DAG $(\mathbf V,\{(X,Y)\})$, in which every other vertex is isolated, satisfies $f$ because its direct $X\to Y$ path has no internal node and validity gives $X,Y\notin\mathbf{Z}$. In either case, Corollary~\ref{cor:graph-fidelity} with $T=\{f\}$ supplies a stable model enforcing $f$; releasing every other fact therefore gives a correction set in $K_f$. Thus the hard-retention program is coherent. It is finite ground, so its nonempty finite set of attainable costs has a minimum and it has an optimal stable model, proving totality.
\end{proof}

A fact-enforcement contestation gives categorical redress by making $f$ mandatory. A user may instead make a graded contestation, increasing the cost of releasing $f$ while leaving it possible. The operational question is then how much that cost must increase before retaining $f$ becomes at least as good as releasing it. We first define the relevant quantity generally and then prove that it is the exact switching threshold for OptABA-PC.

We recall the contestability-margin definition from the main paper.
\contestabilitymargin*
We now specialise to $\aspP=\aspP_{\mathrm{CA}}$. Proposition~\ref{prop:optaba-contestability} makes $K_f$ nonempty for every valid CI fact, and the all-release correction set belongs to $L_f$. Since $\aspO$ is finite and its release costs are finite, $\mathcal F$ is finite and every repair has finite cost; both category minima are therefore attained, as required by the definition. Proposition~\ref{prop:margin-threshold} below shows exactly how their difference governs OptABA-PC's response.

\begin{suppproposition}[OptABA-PC contestation threshold]\label{prop:margin-threshold}
Let $f$ be a valid CI fact released by at least one optimal OptABA-PC repair under $\aspW$. For $\delta\geq0$, let $\aspW_\delta$ replace $w_f$ by $w_f+\delta$, leaving every other release cost unchanged, and let $\operatorname{Rep}_\delta$ be the repairs returned after re-solving OptABA-PC under $\aspW_\delta$. Then:
\begin{enumerate}
    \item if $\delta<\mu(f)$, $\operatorname{Rep}_\delta\subseteq L_f$: every returned repair still releases $f$;
    \item if $\delta=\mu(f)$, $\operatorname{Rep}_\delta$ contains both a retaining and a releasing repair;
    \item if $\delta>\mu(f)$, $\operatorname{Rep}_\delta\subseteq K_f$: every returned repair retains $f$.
\end{enumerate}
With integer implementation weights, $\mu(f)+1$ is therefore the smallest integer increase that forces every returned repair to retain $f$.
\end{suppproposition}
\begin{proof}
Raising only $w_f$ leaves the cost of every retaining repair unchanged and adds $\delta$ to every releasing repair. Thus their respective best costs under $\aspW_\delta$ are
\[
c_{\mathrm{keep}}(f)
\quad\text{and}\quad
c_{\mathrm{rel}}(f)+\delta.
\]
Their difference is $c_{\mathrm{keep}}(f)-[c_{\mathrm{rel}}(f)+\delta]=\mu(f)-\delta$. Lemma~\ref{lem:positive-minimality} and Theorem~\ref{thm:projection} identify the globally minimum-cost correction sets under $\aspW_\delta$ with $\operatorname{Rep}_\delta$. Both $K_f$ and $L_f$ are finite and nonempty, so both displayed category minima are attained.

\noindent\emph{1. Below the margin.} If $\delta<\mu(f)$, then $c_{\mathrm{rel}}(f)+\delta<c_{\mathrm{keep}}(f)$. Hence no retaining repair is globally optimal and $\operatorname{Rep}_\delta\subseteq L_f$.

\noindent\emph{2. At the margin.} If $\delta=\mu(f)$, then $c_{\mathrm{rel}}(f)+\delta=c_{\mathrm{keep}}(f)$. A minimiser in each of $L_f$ and $K_f$ is therefore globally optimal, so $\operatorname{Rep}_\delta$ contains both a releasing and a retaining repair.

\noindent\emph{3. Above the margin.} If $\delta>\mu(f)$, then $c_{\mathrm{keep}}(f)<c_{\mathrm{rel}}(f)+\delta$. Hence no releasing repair is globally optimal and $\operatorname{Rep}_\delta\subseteq K_f$.

\noindent\emph{Integer threshold.} If all weights are integers, both category costs and $\mu(f)$ are integers. The smallest integer increase strictly greater than $\mu(f)$ is therefore $\mu(f)+1$, which by item~3 is the smallest integer increase forcing every returned repair to retain $f$.
\end{proof}

The margin is therefore not a graph-quality measure. It is an actionable threshold: $\mu(f)=0$ exposes an already optimal retaining alternative, $\mu(f)$ is the increase that creates a retaining--releasing tie, and any larger increase forces retention.

Beyond enabling an individual contestation, the margin supports a comparative reading of inconsistency. Logical incoherence remains binary, but $\mu(f)$ quantifies its weighted consequence for a released fact: it is the additional repair cost incurred by resolving the conflict in favour of retaining $f$, rather than releasing it. Thus, within a fixed trace and weighting scheme, larger margins identify releases that are more strongly supported by the competing evidence and more resistant to reinstatement. The normalised margin $\mu(f)/w_f$ gives the proportional increase in $f$'s release cost required to reach a retaining tie, allowing comparisons across facts and traces that use the same weight semantics. It should therefore be interpreted as relative weighted tension within the repair objective, not as a probability that $f$ is false or as an absolute degree of logical inconsistency. Appendix~\ref{supp:evidence-overview} evaluates both hard-retention responses and these margins (Tables~\ref{tab:contestability-margin-comparison} and~\ref{tab:margin-by-dataset}).

These guarantees concern the enforced CI commitments and have a precise consequence for their graph witnesses. Fix an optimal repair $R$ releasing $f$ and put $T=\aspO\setminus R$. By Theorem~\ref{thm:optaba-guarantees}(3), enforcing $T\cup\{f\}$ is incoherent; Corollary~\ref{cor:graph-fidelity} therefore implies that no DAG satisfying $T$ can also satisfy $f$. Hard retention, or a cost increase $\delta>\mu(f)$, returns only graphs satisfying $f$, so it excludes every graph compatible with that releasing repair. A graph belonging to a different, already optimal repair retaining $f$ may remain when $\mu(f)=0$. Below the margin, releasing optima remain optimal; at the margin, retaining and releasing optima coexist. Thus changing the enforced facts can force a change of graph witness, without guaranteeing improved graph accuracy or recovery of the generating DAG.

  \section{Additional Empirical Evidence}
\label{supp:empirical-evidence}
\label{supp:empirical-contestability}

This section complements the primary experiments and formal results.  It begins by instantiating the repair, private-conflict, and contestability constructions on the four-node example from the main paper.  It then gives the full experimental protocol and metric definitions, summarises the fact enforcement comparison across methods and datasets, compares optimal and heuristic ABA repair, reports fact-level redress, distinguishes the repair claim from MPC's procedural use of tests and ASPCR-DAG's different objective, and concludes with the full reconstruction, fact-quality, and computational results.

\subsection{Four-Node Worked Example}\label{supp:four-node-example}

The main paper reports $14$ CI facts and their strengths for the variables $(E,R,O,I)$. We multiply those reported strengths by $1000$ for illustrative integer costs; the empirical weight scaling is given in Table~\ref{tab:final-method-settings}. Exhaustive enumeration of all $543$ labelled DAGs on these variables gives the unique OptABA-PC baseline repair
\[
R=\{\texttt{indep(e,o,\{i\})},\texttt{indep(r,i,\{e\})},\texttt{indep(o,i,\{\})}\}
\]
of cost $2040$. Its unique graph witness has edges $E\to O$, $R\to O$, and $O\to I$. Table~\ref{tab:supp-example} gives one private MUS and the hard-retention cost and margin for every released fact. Each listed MUS intersects $R$ only in the contested fact.

\begin{table}[t]
\caption{Contestation certificates and exact integer margins for the three facts released in the 14-fact four-node trace. The weights are the main-paper strengths multiplied by $1000$; $c_{\rm keep}$ is the minimum repair cost after making the contested fact mandatory, and $\mu(f)=c_{\rm keep}(f)-2040$.}
\label{tab:supp-example}
\centering
\small
\begin{tabular}{@{}lrrrl@{}}
\toprule
Contested fact $f$ & $w_f$ & $c_{\rm keep}$ & $\mu(f)$ & One private MUS $U_f$ \\
\midrule
\texttt{indep(e,o,\{i\})} & 1000 & 2070 & 30 &
\begin{tabular}[t]{@{}l@{}}
\(\{f,\texttt{dep(e,o,\{\})},\)\\
\(\quad\texttt{indep(e,i,\{o\})}\}\)
\end{tabular} \\
\texttt{indep(r,i,\{e\})} & 530 & 2540 & 500 &
\begin{tabular}[t]{@{}l@{}}
\(\{\texttt{dep(r,i,\{\})},f,\)\\
\(\quad\texttt{indep(e,r,\{\})}\}\)
\end{tabular} \\
\texttt{indep(o,i,\{\})} & 510 & 2910 & 870 &
\begin{tabular}[t]{@{}l@{}}
\(\{\texttt{indep(r,i,\{o\})},\)\\
\(\quad\texttt{dep(r,i,\{\})},f\}\)
\end{tabular} \\
\bottomrule
\end{tabular}

\end{table}

The corresponding minimum-cost hard-retention repairs are, in table order,
\begin{align*}
R_1&=\{\texttt{indep(e,i,\{o\})},\texttt{indep(r,i,\{e\})},\texttt{indep(o,i,\{\})},\texttt{indep(e,r,\{\})}\},\\
R_2&=\{\texttt{indep(e,o,\{i\})},\texttt{indep(e,i,\{o\})},\texttt{indep(o,i,\{\})},\texttt{indep(e,r,\{\})}\},\\
R_3&=\{\texttt{dep(e,o,\{\})},\texttt{dep(e,o,\{r\})},\texttt{dep(r,i,\{\})},\texttt{indep(e,i,\{o\})}\}.
\end{align*}
Each differs from $R$ and retains the contested fact. The private MUS proves that reinstatement alone is incoherent. Increasing the contested fact's release weight by $30$, $500$, or $870$, respectively, makes the baseline and hard-retention costs tie; increments of $31$, $501$, or $871$ force every optimum to retain it.
A separate $13$-fact variant, obtained by omitting $\texttt{indep(e,r,\{\})}$ from the trace, is also exhaustively checked. It has $18$ MCSes, a unique optimal released-fact projection of cost $1570$, and four compatible DAGs. Making the released fact $\texttt{indep(e,i,\{o\})}$ mandatory raises the proved optimum to $2040$ and yields a disjoint set of $13$ compatible DAGs. This fact is false for the generating DAG in the main paper, whose direct edge $E\to I$ remains active given $O$; neither graph family contains that DAG. The example demonstrates enforced commitment, not truth restoration.

\subsection{Experimental Protocol and Metrics}\label{supp:experimental-protocol}

\paragraph{Benchmark choice.} Using the benchmark design of Causal ABA~\citep{russo2024abacd}, the fixed-network study includes Cancer ($5$ nodes, $4$ edges), Earthquake ($5,4$), Survey ($6,6$), and Asia ($8,8$) from the bnlearn repository~\citep{Scutari14bnlearn}.  These small, standard Bayesian networks include medical, legal, and policy decision models derived from published studies or expert elicitation.  They are appropriate for this evaluation for three reasons: their known DAGs and conditional probability tables (CPTs) provide ground truth for both CI facts and graph reconstruction; repeated finite samples expose the test errors and conflicts that motivate repair; and reusing the benchmarks and sampling design of~\citep{russo2024abacd} makes the comparison with ABA-PC a controlled extension rather than a change of empirical regime.  We add Erd\H{o}s--R\'enyi (ER)~\citep{erdds1959random} and scale-free (SF)~\citep{barabasi1999emergence} DAGs with five and eight nodes to vary topology beyond four fixed networks and to evaluate the MPC~\citep{colombo2014order} and FGS~\citep{ramsey2017million} reference methods under the same generated samples.

\paragraph{Data generation.} For each dataset we use 50 seeds, 2026--2075, and generate $N=5000$ independent observations per seed.  For a bnlearn benchmark, we load its BIF file, including the published DAG and CPTs, and forward-sample a new data matrix for each seed using the sampling protocol of~\citep{russo2024abacd}.  Thus the BN is fixed and only the finite sample changes.  For a primary synthetic run, the seed determines both a new graph and its data-generating distribution.  An ER graph is formed with $n$ edges and given a random acyclic orientation and node permutation; an SF graph is generated by directed Barab\'asi--Albert preferential attachment with one attachment per added node and then randomly permuted.  Each node is categorical with at least two states; its allowed cardinality is reduced when it participates in a larger parent configuration, using $20$ as the implementation's CPT-enumeration target.  For each node and parent configuration, its probability vector is drawn from a symmetric Dirichlet distribution whose concentration is sampled uniformly from $[1,5]$, after which the BN is forward-sampled.  Categorical states are numerically encoded and columns are aligned with the generating adjacency matrix.  Before CI testing, the bnlearn codes are additionally standardised; this is a one-to-one relabelling and therefore leaves the contingency counts for the $G^2$ CI test~\citep{spirtes2000causation} unchanged.  FGS independently factorises every categorical column into integer levels for its BDeu score.

\paragraph{Methods and matched evidence.} Our proposed OptABA-PC, ABA-PC~\citep{russo2024abacd}, and MPC~\citep{colombo2014order} use the likelihood-ratio $G^2$ conditional-independence test~\citep{wilks1935likelihood,spirtes2000causation} at $\alpha=0.01$.  Within a seed, ABA-PC and OptABA-PC receive exactly the same recorded CI facts and $p$-value-derived release strengths (as in~\citep{russo2024abacd}), as well as the same time limits; their paired comparison therefore changes only the repair rule.  MPC uses the same sample and test configuration but consumes test outcomes procedurally.  FGS~\citep{ramsey2017million} uses Tetrad's native discrete BDeu score on the same categorical sample and returns one score-selected CPDAG; because it exposes no CI trace, it is a graph-reconstruction rather than fact-repair baseline.  ASPCR-DAG~\citep{hyttinen2014ASP} uses its native Bayesian log-weighted CI configuration (prior independence $0.4$ and Bayesian $\alpha$ parameter $20$) on the same saved samples. Its graph comparison and forced-fact sensitivity audit cover Cancer, Earthquake, Survey, ER(5), and SF(5), with 50 DAG-restricted (otherwise it could return partial ancestral graphs under no causal sufficiency assumption, allowing for latent confounders) runs per dataset.

\begin{table}[t]
\caption{Baseline descriptions and final method settings. Shared data, seed, CI, and computational settings are given in the surrounding protocol.}
\label{tab:final-method-settings}
\centering
\small
\begin{tabularx}{\linewidth}{@{}l>{\raggedright\arraybackslash}X>{\raggedright\arraybackslash}X@{}}
\toprule
Method & Role and implementation & Final settings \\
\midrule
MPC & Majority-PC~\citep{colombo2014order}, an order-independent PC variant, learns a skeleton by CI testing, orients colliders by majority evidence across separating sets, and applies Meek propagation to return one CPDAG. & Stable skeleton search; $G^2$ test; $\alpha=0.01$; majority collider rule (implementation \texttt{uc\_rule=5}); collider priority $2$; no conditioning-depth limit. \\
ABA-PC & Causal ABA-PC~\citep{russo2024abacd} receives MPC's CI trace and releases facts from lowest to highest strength until its Bayes-ball Causal ABA program is coherent. & Same MPC trace; no conditioning-set-size multiplier; strength $s(p)=1-p/(2\alpha)$ for $p<\alpha$ and $s(p)=(\alpha-p/2-1/2)/(\alpha-1)$ otherwise; integer weight $\max\{1,\operatorname{round}(9{,}999{,}999s)\}$; incremental Bayes-ball encoding. \\
OptABA-PC & Our method uses the same facts, weights, and candidate-DAG encoding as ABA-PC, replacing sequential release with one minimum-cost correction-set optimisation. & MUS-reified weak constraints; \textit{clingo} branch-and-bound (\texttt{bb}); \texttt{optN} enumeration of optimal stable models; minimum total released weight; four solver threads. \\
FGS & Fast Greedy Search~\citep{ramsey2017million} is the Tetrad/\texttt{py-causal} fast implementation of GES~\citep{chickering2002optimal}; it greedily inserts and removes edges under a score and returns one CPDAG. & BDeu score (\texttt{bdeu-score}); discrete-data interface; sample prior $15$; structure prior $1$; unrestricted maximum degree (\texttt{-1}); faithfulness assumed; symmetric first step disabled. \\
ASPCR-DAG & The ASP method of~\citet{hyttinen2014ASP}, restricted here to causally sufficient DAGs, assigns a cost to every graph--test disagreement and returns one optimal DAG. & Exhaustive conditioning schedule $n-2$; Bayesian CI test; log weights; prior independence $0.4$; Bayesian $\alpha=20$; acyclic sufficient encoding; \textit{clingo} \texttt{crafty} configuration; comparison time limit $300$\,s. \\
\bottomrule
\end{tabularx}

\end{table}

\paragraph{Hyperparameters.} The configuration of parameters in Table~\ref{tab:final-method-settings} was fixed from the standard published implementation or the Causal ABA protocol. In particular, $G^2$ with $\alpha=0.01$ is the Causal ABA setting used by MPC, ABA-PC, and OptABA-PC; MPC uses the standard stable-majority rule; FGS uses Tetrad's native discrete BDeu configuration listed in the table; and ASPCR-DAG uses the authors' Bayesian log-weighted configuration.

\paragraph{Computation.} All ABA repair runs use our proposed Bayes-ball encoding, a 300-second solve limit, a separate 120-second compatible-set evaluation limit, and four \textit{clingo} threads. ASPCR-DAG is omitted above six nodes because preliminary runs exceeded one hour, well beyond the time limit in Table~\ref{tab:final-method-settings}. This is also reported in the original paper~\citep{hyttinen2014ASP}. At eight nodes, the $\binom{n}{2}2^{n-2}=1792$ possible CI constraints define a conceptual space of $2^{1792}$ fact subsets, motivating both the limits and the smaller scope of the exact contestability audit. Tables~\ref{tab:contestability-overview} and~\ref{tab:primary-repair-enforcement} audit all 400 (50 seeds $\times$ 8 datasets) primary $G^2$ traces. The exact decision audit in Table~\ref{tab:contestability-margin-comparison} uses 250 proved ABA instances and makes each released OptABA-PC fact mandatory in turn. The separate ASPCR-DAG sensitivity audit (Table~\ref{tab:contestability-margin-comparison}) re-solves all 4,330 failed facts in 250 proved native outputs. MPC has no consistent DAG extension for 12 endpoint-invalid Survey outputs and one endpoint-invalid ER(8) output, so only the affected DAG metrics display paired sample sizes of 38 and 49.

Table~\ref{tab:bayesball-runtime} isolates the efficiency differences between the original active-path~\citep{russo2024abacd} and our proposed Bayes-ball encodings under the matched primary configuration.  Across $50$ seeds on each of the eight datasets, both implementations receive identical test traces and weights and use the same sequential ABA-PC repair policy; every pair that completes returns the same retained-fact count and weight.  Bayes-ball completes all $400$ runs, while active-path materialisation times out on three ER(8) and two SF(8) runs.  Its mean end-to-end time is significantly lower on Cancer, Survey, Asia, ER(8), and SF(8), indistinguishable on ER(5), and higher on the very small Earthquake and SF(5) instances.  The largest reductions occur at eight nodes: $94.7\%$ on Asia and at least $96.8\%$ on the paired synthetic runs.  Thus the benefit is a marked improvement in scaling rather than a uniform speedup on negligible-runtime instances.  The formal proof and independent exhaustive checker establish that the Bayes-ball rules enforce the same $d$-separation semantics without materialising active paths.

\begin{table}[t]
\caption{Matched encoding-efficiency ablation over seeds 2026--2075, with $N=5000$, $G^2$ tests at $\alpha=0.01$, identical test traces and weights, and the same sequential ABA-PC repair policy. Only the causal-consistency encoding changes. Times are mean $\pm$ standard deviation for completed seed pairs and include program construction and solving; ``BB/original'' is the ratio of those means, so values below one favour Bayes-ball. Completion is over all 50 runs. Bold marks the significantly faster implementation under paired two-sided Wilcoxon tests with BH adjustment across datasets ($^*$, $^{**}$, and $^{***}$ for adjusted $p<0.05$, $0.01$, and $0.001$). Every completed pair returned identical retained-fact counts, weights, and model counts.}
\label{tab:bayesball-runtime}
\centering
\small
\resizebox{\linewidth}{!}{%
\begin{tabular}{@{}lrrrrrr@{}}
\toprule
Dataset & \multicolumn{2}{c}{Completed} & Paired $n$ & \multicolumn{2}{c}{End-to-end time (s)} & BB/original \\
\cmidrule(lr){2-3}\cmidrule(lr){5-6}
 & Active paths & Bayes-ball & & Active paths & Bayes-ball & \\
\midrule
Cancer (5) & 50/50 & 50/50 & 50 & $0.118\pm0.027$ & $\mathbf{0.102\pm0.014}^{***}$ & $0.86\times$\\
Earthquake (5) & 50/50 & 50/50 & 50 & $\mathbf{0.068\pm0.013}^{***}$ & $0.099\pm0.013$ & $1.45\times$\\
Survey (6) & 50/50 & 50/50 & 50 & $0.385\pm0.119$ & $\mathbf{0.148\pm0.058}^{***}$ & $0.38\times$\\
Asia (8) & 50/50 & 50/50 & 50 & $51.66\pm50.27$ & $\mathbf{2.74\pm4.48}^{***}$ & $0.05\times$\\
ER (5) & 50/50 & 50/50 & 50 & $0.078\pm0.029$ & $0.084\pm0.007$ & $1.08\times$\\
ER (8) & 47/50 & 50/50 & 47 & $39.72\pm58.73$ & $\mathbf{1.26\pm4.14}^{***}$ & $0.03\times$\\
SF (5) & 50/50 & 50/50 & 50 & $\mathbf{0.075\pm0.026}^{***}$ & $0.099\pm0.096$ & $1.32\times$\\
SF (8) & 48/50 & 50/50 & 48 & $34.43\pm61.81$ & $\mathbf{0.867\pm2.488}^{***}$ & $0.03\times$\\
\bottomrule
\end{tabular}
}
\end{table}

\paragraph{Fact and graph metrics.} A recorded CI fact is correct when its independence or dependence relation agrees with $d$-separation in the generating DAG. Fact precision treats retained correct and false facts as true and false positives; fact recall additionally counts released correct facts as false negatives; FactF1 is their harmonic mean. Graph F1 is the harmonic mean of edge precision and recall, and Structural Hamming Distance (SHD)~\citep{tsamardinos2006MMHC} counts adjacency and orientation edits relative to the generating graph. Because observational data identify a Markov-equivalence class, we score CPDAGs and, where the method returns a CPDAG, its \emph{consistent DAG extensions}: DAGs in the represented class that preserve its adjacencies and compelled arrowheads while orienting reversible edges. ABA-PC and OptABA-PC are scored over the compatible DAGs saved within the evaluation budget; ASPCR-DAG is scored on its single returned DAG. Average, best, and worst values describe variation within the saved structural set, with the partial-enumeration qualification stated below.

MPC and FGS natively return a CPDAG rather than one oriented DAG.  The facts selected by ABA-PC and OptABA-PC can likewise admit several DAGs, possibly spanning several CPDAGs.  Table~\ref{tab:final-graph-metrics} therefore compares equivalence-class reconstruction and then the relevant DAG outputs: every compatible DAG for ABA-PC/OptABA-PC, every consistent extension of the CPDAG returned by MPC/FGS, and the single DAG returned by ASPCR-DAG.  Tables~\ref{tab:final-cpdag-ranges}--\ref{tab:final-dag-f1-ranges} report SHD and F1 variation across the corresponding CPDAGs and DAGs.  All DAGs in one CPDAG entail the same CI model, so compatibility is class-level; their edge orientations can nevertheless have different SHD and F1 against the generating DAG.

Using PC-based methods, returned structures need not satisfy the CI evidence used during discovery. For ABA-PC and OptABA-PC, the reported compatible CPDAGs and DAGs satisfy every retained post-repair fact. For MPC and ASPCR-DAG, compatibility requires satisfaction of every test in the complete recorded native trace. It is undefined for score-based FGS and is shown as ``--''. Full-trace enforcement permits no released facts, whereas selected-set enforcement tests only the commitments exposed by a repair. We report the seed-level mean $\pm$ standard deviation.

\paragraph{Statistical analysis.}   For every reported method comparison, we align runs by dataset and seed, discard a pair only when that outcome is unavailable for either method, and apply a two-sided Wilcoxon signed-rank test~\citep{wilcoxon1945individual} to the paired values.  Benjamini--Hochberg false-discovery-rate adjustment~\citep{benjamini1995controlling} is performed within each dataset--metric family across the available comparisons; adjusted $p<0.05$ is treated as significant.  An all-zero paired difference is assigned $p=1$. Stars compare OptABA-PC with ABA-PC and daggers compare other methods with OptABA-PC. The only reduced paired samples are the MPC DAG-extension metrics identified above and locally in their tables. \OptABAMainTableRef\ reports within-seed differences and relative effect sizes; Tables~\ref{tab:absolute-results} and~\ref{tab:absolute-cpdag-results} supply their absolute scale and inferential annotations.

\paragraph{Execution environment.} Experiments were run CPU-only on Ubuntu 22.04.5 LTS (Linux 6.8.0-124, x86-64) on one Intel Xeon w5-2455X socket with 12 physical cores, 24 hardware threads, a 3.20\,GHz base frequency and 4.60\,GHz maximum frequency, and 128\,GB RAM (125\,GiB visible to the operating system); no GPU was used. The primary environment used Python 3.12.1, \textit{clingo} 5.8.0, NumPy 1.26.4, pandas 2.2.3, SciPy 1.15.3, scikit-learn 1.6.1, NetworkX 3.4.2, \texttt{causal-learn} 0.1.4.1, pgmpy 1.0.0, \texttt{py-causal} 1.2.1, and OpenJDK 19.0.2. ASPCR-DAG used R 4.1.2 with \texttt{pcalg} 2.7-11, \texttt{graph} 1.72.0, and \texttt{Rgraphviz} 2.38.0.

\begin{table}[t]
\caption{Evidence correspondence, run completion, and repair coverage. A trace is feasible if some DAG satisfies every recorded CI outcome; a full-trace output is the method's returned graph or class satisfying all of them. A timeout incumbent is a valid saved repair whose optimality was not proved before the limit. An explicit repair names excluded facts, and selected-repair enforcement verifies every retained fact against the output. ASPCR-DAG uses its primary matched run set; FGS has no CI trace. A dash denotes an undefined object.}
\label{tab:contestability-overview}
\centering
\small
\setlength{\tabcolsep}{3pt}
\begin{tabularx}{\linewidth}{@{}l>{\raggedright\arraybackslash}p{0.25\linewidth}*{4}{>{\centering\arraybackslash}X}@{}}
\toprule
Method & Run status & Feasible traces & Full-trace matches & Explicit repairs & Selected repair enforced \\
\midrule
OptABA-PC & 400/400 saved: 368 completed, 32 timeout incumbents & 88/400 & 88/400 & 312/312 & 400/400 \\
ABA-PC & 400/400 saved and completed & 88/400 & 88/400 & 312/312 & 400/400 \\
ASPCR-DAG & 250/250 saved; all optima proved & 32/250 & 32/250 & -- & -- \\
MPC & 400/400 saved and completed & 88/400 & 80/400 & -- & -- \\
FGS & 400/400 saved and completed & -- & -- & -- & -- \\
\bottomrule
\end{tabularx}

\end{table}

\subsection{Fact Enforcement Comparison}\label{supp:evidence-overview}

Table~\ref{tab:contestability-overview} reports run completion, input feasibility, output correspondence, and repair coverage. Table~\ref{tab:contestability-margin-comparison} distinguishes the decision evidence each method supplies. Table~\ref{tab:primary-repair-enforcement} then places the 400-trace OptABA-PC--ABA-PC cost comparison beside MPC's output enforcement statistics, while Table~\ref{tab:margin-by-dataset} gives the dataset-level forced-fact margins.

Table~\ref{tab:contestability-overview} states the two relevant counts directly. A \emph{feasible trace} is an input for which at least one DAG satisfies every recorded outcome; a \emph{full-trace output} is the returned graph or Markov class actually doing so. Only 88/400 primary $G^2$ traces are feasible. ABA-PC and OptABA-PC match all 88 whenever an output is present because both expose and enforce a selected repair. MPC returns a graph on every input but matches only 80 traces, leaving eight avoidable failures among the 88 feasible cases. In the primary ASPCR-DAG setting, the proved soft-objective optimum has zero disagreement on 32/250 traces; positive optimum cost establishes that the remainder are infeasible under its DAG encoding. FGS has no CI trace to audit.

Across all 400 runs, 368 completed and 32 stopped at the optimisation limit after saving a valid incumbent: two on Asia, 19 on ER(8), and 11 on SF(8). Compatible-set enumeration is a separate post-solve step: the stored primary ABA/OptABA-PC metadata flags every such enumeration as budget-limited, including all 50 SF(8) instances. The saved set sizes and set-average graph metrics are consequently partial-enumeration summaries; this does not affect the selected facts, saved repair consistency, or the separately proved exact audits below.

\begin{table}[t]
\caption{Decision and forced-fact evidence. OptABA-PC, ABA-PC, and MPC use the exact 250-trace $G^2$ audit; ASPCR-DAG uses a separate 250-run native Bayesian audit. A forced OptABA-PC solve is a user contestation and $\mu(f)/w_f$ is its exact threshold. ASPCR-DAG's forced solve is a post-hoc sensitivity test of a soft graph-disagreement objective, so its $\mathrm{gap}/w_f$ is not numerically comparable. ABA-PC and MPC define neither category optima nor a forced-fact certificate. Brackets give $[Q_1,Q_3]$; a dash means that the object is undefined. Bold marks OptABA-PC's exact certified result.}
\label{tab:contestability-margin-comparison}
\centering
\small
\setlength{\tabcolsep}{3pt}
\begin{tabularx}{\linewidth}{@{}l>{\raggedright\arraybackslash}p{0.09\linewidth}*{3}{>{\raggedright\arraybackslash}X}@{}}
\toprule
Method & Scope & Baseline decision evidence & Forced-fact evidence & Certificate supplied \\
\midrule
OptABA-PC & $G^2$, 250 & Minimum repair on \textbf{111/111} inconsistent traces; lower cost than ABA-PC on \textbf{108}, tied on 3. & \textbf{742/742} proved; zero gaps 0/742; median $\mu/w_f=0.455$ [0.182,0.912]. 11.6\% are $\leq0.1$; 22.5\% are $>1$. & \textbf{Exact user-contestation threshold.} \\
ABA-PC & $G^2$, 250 & Selected repair enforced on 250/250; minimum cost on only 3/111 inconsistent traces (108 higher). & -- & Heuristic repair; no exact response threshold. \\
MPC & $G^2$, 250 & Compatible on 126/139 feasible traces; 13 avoidable gaps; no selected repair. & -- & No repair or contestation object. \\
ASPCR-DAG & native, 250 & Full trace on 23/250; graph optimum proved on 250/250. & 4330/4330 proved; zero gaps 0/4330; median $\mathrm{gap}/w_f=2.828$ [0.890,8.647]. & Post-hoc graph-objective sensitivity. \\
\bottomrule
\end{tabularx}

\end{table}

Table~\ref{tab:contestability-margin-comparison} shifts from availability to the evidence behind a decision. Its $G^2$ rows use the same 250-trace exact ABA/MPC audit, whereas its ASPCR-DAG row uses a separate 250-run native Bayesian audit. The table therefore compares which objects are certified---minimum repair, enforced selected set, or post-hoc sensitivity---but does not treat the two normalised margin columns as paired effect sizes.

\FloatBarrier

\paragraph{Optimal versus heuristic ABA repair}
\label{supp:abapc-optimality}

ABA-PC exposes a retained/released fact set and its returned graphs satisfy every retained fact, but its strength-ordered stopping rule computes neither of the category optima $c_{\mathrm{rel}}(f)$ and $c_{\mathrm{keep}}(f)$. Re-running the heuristic with $f$ made mandatory could enforce that fact, but its cost difference would not be the exact margin or inherit Proposition~\ref{prop:margin-threshold}. Table~\ref{tab:primary-repair-enforcement} compares the paired repair costs before graph accuracy is considered and, in the final columns, contrasts this repair with MPC's enforcement of the input trace.

\begin{table}[t]
\caption{Repair quality and MPC enforcement on the same 400 primary $G^2$ traces. Feasible inputs admit a DAG satisfying the entire trace; all other inputs require a repair. \emph{Lower/inconsistent} counts saved OptABA-PC repairs with lower release cost than ABA-PC among the inconsistent inputs, and \emph{same cost} is the complementary count; ABA-PC is never cheaper. Median saving is conditional on a lower saved cost. Thirty-two OptABA-PC rows are valid timeout incumbents; Table~\ref{tab:contestability-margin-comparison} separately gives the 250-instance exact audit with proved optima. MPC \emph{full trace} counts returned classes satisfying every input test, and \emph{gap} counts feasible traces that MPC nevertheless fails to satisfy.}
\label{tab:primary-repair-enforcement}
\label{tab:abapc-optimality-by-dataset}
\label{tab:mpc-enforcement-gap}
\centering
\small
\resizebox{\linewidth}{!}{%
\begin{tabular}{@{}lrrrrrr@{}}
\toprule
& Input status & \multicolumn{3}{c}{Saved OptABA-PC repair vs. ABA-PC} & \multicolumn{2}{c}{MPC output} \\
\cmidrule(lr){2-2}\cmidrule(lr){3-5}\cmidrule(l){6-7}
Dataset & Feasible / inputs & Lower / inconsistent & Same cost & Median saving & Full trace / inputs & Gap / feasible \\
\midrule
Cancer (5) & 0/50 & 50/50 & 0 & 40.2\% & 0/50 & -- \\
Earthquake (5) & 47/50 & 1/3 & 2 & 61.8\% & 47/50 & 0/47 \\
Survey (6) & 0/50 & 50/50 & 0 & 57.1\% & 0/50 & -- \\
Asia (8) & 0/50 & 50/50 & 0 & 73.0\% & 0/50 & -- \\
ER (5) & 16/50 & 26/34 & 8 & 48.4\% & 11/50 & 5/16 \\
ER (8) & 3/50 & 46/47 & 1 & 66.4\% & 3/50 & 0/3 \\
SF (5) & 20/50 & 23/30 & 7 & 60.9\% & 17/50 & 3/20 \\
SF (8) & 2/50 & 48/48 & 0 & 71.6\% & 2/50 & 0/2 \\
\midrule
\textbf{Total} & 88/400 & \textbf{294/312} & 18 & 63.4\% & 80/400 & \textbf{8/88} \\
\bottomrule
\end{tabular}
}
\end{table}

On the 312 inconsistent primary traces, the saved OptABA-PC repair costs less on 294 and ties on 18; ABA-PC is never cheaper. The conditional median reduction is $63.4\%$, with strict improvements on every inconsistent Cancer, Survey, Asia, and SF(8) trace and on 46/47 ER(8) traces. These full-coverage figures include 32 valid timeout incumbents. The exact 250-instance audit in Table~\ref{tab:contestability-margin-comparison} proves the corresponding optimum on every pair: among its 111 inconsistent traces, OptABA-PC costs less on 108 and ties on three.

\paragraph{OptABA-PC: fact-level redress and robustness}
\label{supp:optaba-contestability}

We instantiated the hard-retention contestation and exact margin from Proposition~\ref{prop:optaba-contestability}, Definition~\ref{def:margin}, and Proposition~\ref{prop:margin-threshold} on the frozen OptABA-PC runs. The exact audit uses Cancer, Earthquake, Survey, ER(5), and SF(5), with 50 seeds (2026--2075) per dataset, $N=5000$, and the matched $G^2$ test at $\alpha=0.01$. We do not pool eight-node margins because not every baseline and hard-retention solve reached proved optimality within the computational budget.

For every fact $f$ in the released-fact projection of a baseline optimum, we re-solved the emitted optimisation program after adding the hard constraint that makes $f$ mandatory.  Because the baseline optimum releases $f$, its proved cost gives $c_{\mathrm{rel}}(f)$; the proved hard-retention optimum gives $c_{\mathrm{keep}}(f)$.  Their difference therefore directly instantiates the margin in Definition~\ref{def:margin}.  We report $\mu(f)/w_f$ and, to relate this fact-specific threshold to the complete repair, compare $\mu(f)$ with the baseline total release cost $c_{\mathrm{rel}}(f)$.

\begin{table}[t]
\caption{Dataset-level forced-fact evidence over 250 validated baselines per method (50 per dataset). Each target is a fact rejected by the baseline optimum; all re-solves prove optimality and none has zero gap. OptABA-PC's $\mu(f)/w_f$ is an exact user-contestation threshold; $c^*$ denotes the complete baseline optimum repair cost. ASPCR-DAG's $\mu_{\rm ASP}(f)/w_f$ is instead a post-hoc sensitivity gap for its soft graph-disagreement objective. These are not paired effect sizes: the methods use different traces, weights, and objectives, and their synthetic audits use different graph densities.}
\label{tab:margin-by-dataset}
\centering
\small
\resizebox{\linewidth}{!}{%
\begin{tabular}{@{}lrrrrr@{}}
\toprule
 & \multicolumn{3}{c}{OptABA-PC contestations} & \multicolumn{2}{c}{ASPCR-DAG sensitivity probes} \\
\cmidrule(lr){2-4}\cmidrule(l){5-6}
Dataset & Proved & Median $\mu/w_f$ $[Q_1,Q_3]$ & Median $\mu/c^*$ & Proved & Median $\mu_{\rm ASP}/w_f$ $[Q_1,Q_3]$ \\
\midrule
Cancer (5) & 207/207 & 0.297 [0.158,0.609] & 0.072 & 959/959 & 4.041 [1.615,7.946] \\
Earthquake (5) & 4/4 & 4.170 [1.628,6.332] & 3.544 & 80/80 & 3.625 [0.990,6.349] \\
Survey (6) & 507/507 & 0.517 [0.200,1.026] & 0.051 & 2094/2094 & 2.376 [0.537,14.074] \\
ER (5) & 0/0 & -- & -- & 503/503 & 2.163 [0.773,5.555] \\
SF (5) & 24/24 & 0.734 [0.308,1.344] & 0.182 & 694/694 & 2.437 [0.958,6.108] \\
\midrule
\textbf{Total} & 742/742 & 0.455 [0.182,0.912] & 0.057 & 4330/4330 & 2.828 [0.890,8.647] \\
\bottomrule
\end{tabular}
}
\end{table}

Table~\ref{tab:margin-by-dataset} records every dataset-level and pooled threshold. All 742 OptABA-PC contestations prove both category optima, with no timeout, infeasible hard-retention instance, or zero margin; ER(5) has no rejected fact and therefore no imputed value. This instantiates Proposition~\ref{prop:optaba-contestability} under its unrestricted-DAG assumptions.

Strict positivity means that no reported release merely selects one member of an already tied retaining/releasing optimum. It also resolves the baseline-witness issue at the fact-selection level: if another baseline optimum retained a contested fact, that fact would have zero margin, while positive weights rule out an optimal strict superset of the reconstructed release. Hence every audited instance has a unique optimal released-fact projection, although it may still have multiple stable-model or graph witnesses.

The margins distinguish releases sensitive to small increases in their cost from those robust to larger increases, holding other costs fixed. For $11.6\%$ of releases, an increase of at most $10\%$ makes retention optimal; $22.5\%$ require more than doubling their release cost (Table~\ref{tab:contestability-margin-comparison}). These thresholds mark the first tie with a retaining optimum; a strictly larger increase forces every optimum to retain the fact (Proposition~\ref{prop:margin-threshold}). The median margin is $45.5\%$ of the contested fact's release cost; normalising instead by the complete baseline repair cost gives a median of $5.7\%$ (Table~\ref{tab:margin-by-dataset}). Cancer, Survey, and SF(5) have appreciably different medians, while Earthquake's large estimate rests on only four targets.
\FloatBarrier

\paragraph{Majority-PC: use of tests versus output enforcement}
\label{supp:baseline-enforcement}

A procedure can use a CI test without constraining its final graph to satisfy that test.  This distinction is central to interpreting MPC.  Within KPC a CI outcome is consumed during skeleton discovery and collider orientation, but is not posted as a global $d$-separation constraint on the returned CPDAG.  Forcing a local test routine would therefore only force that procedural branch; other paths, separating sets, or downstream orientation steps could still yield a Markov class contradicting the same relation.  MPC also exposes neither a retained/released fact set nor a global release-cost objective, so $c_{\mathrm{rel}}(f)$ and $c_{\mathrm{keep}}(f)$ are not defined for it.

The final columns of Table~\ref{tab:primary-repair-enforcement} report the MPC audit on the same 400 primary traces; the exact 250-trace comparison, including its denser ER(5)/SF(5) setting, appears in Table~\ref{tab:contestability-margin-comparison}. Since every DAG in one CPDAG has the same conditional-independence model, one consistent extension suffices for the semantic check; endpoint-invalid outputs remain explicit and count as incompatible.

MPC's raw full-compatibility count is dominated by input inconsistency: only 88/400 primary traces are feasible. Conditional on feasibility, it succeeds on 80/88 and leaves eight avoidable gaps, five on ER(5) and three on SF(5). The separate exact audit gives the same conditional pattern: 126/139 compatible feasible traces and 13 gaps, two on ER(5) and 11 on SF(5). The corresponding success rates, $90.9\%$ and $90.6\%$, show that the change in raw counts reflects trace feasibility and graph density rather than a reversal in method behaviour. MPC can still be a strong graph estimator (Table~\ref{tab:final-graph-metrics}), but it neither identifies a repair on inconsistent inputs nor distinguishes unavoidable contradictions from avoidable output gaps.

\paragraph{ASPCR-DAG: forced-fact sensitivity under a different objective}
\label{supp:aspcr-contestability}

ASPCR-DAG permits a forced-fact sensitivity audit, but not a test of the OptABA-PC contestability guarantee.  Let $\mathcal F_{\mathrm{ASP}}$ be its exhaustive native Bayesian CI trace and $w_i$ the integer log weight of fact $f_i$.  Its DAG-restricted encoding derives $\mathsf{fail}(f_i)$ exactly when a candidate graph contradicts $f_i$ and minimises
\[
  C_{\mathrm{ASP}}(G)
  =\sum_i w_i\,\mathbf 1[G\not\models f_i].
\]
The term $\mathbf 1[\cdot]$ denotes the indicator function, equal to one when its argument is true and zero otherwise.
For every fact $f$ failed by the saved optimum, we perform what we call a \emph{hardening}: as a sensitivity probe, we add the integrity constraint $\leftarrow\mathsf{fail}(f)$ and re-solve.  This turns that one test from a soft, penalised correspondence into a mandatory requirement that the new DAG satisfy it, while every other test remains soft.  Hence
\[
  \mu_{\mathrm{ASP}}(f)
  =\min_{G\models f}C_{\mathrm{ASP}}(G)-\min_G C_{\mathrm{ASP}}(G)
\]
is an exact sensitivity margin for the modified ASPCR-DAG objective: the new DAG is forced to correspond to the targeted test, while its correspondence to every other test remains soft.  The hard constraint is imposed by our audit rather than supplied as a method-level redress operation, and ASPCR-DAG has no counterpart of Proposition~\ref{prop:margin-threshold}.  It returns one optimal DAG rather than the complete structural set supported by an explicit correction set.

Tables~\ref{tab:contestability-margin-comparison} and~\ref{tab:margin-by-dataset} report this diagnostic alongside, but separately from, the OptABA-PC audit. The regenerated ASP programs reproduce the saved proved objective on all 250 ASPCR-DAG baselines. Only 23 native traces have zero graph--test disagreement; the other baselines supply 4,330 failed facts for hardening. Every hardening reaches a proved optimum, with no timeout, infeasible target, enforcement failure, or zero gap. Thus every graph at the original optimum cost contradicts the targeted test; enforcing correspondence to it changes the optimum but leaves all remaining test--graph correspondences soft. This is sensitivity evidence about ASPCR-DAG's objective, not actual redress since no consistent repair is offered.

Within that objective, the pooled median normalised gap is $2.828$, with the dataset variation shown in Table~\ref{tab:margin-by-dataset}. If $g=\mu_{\mathrm{ASP}}(f)/w_f$, hardening removes the target's penalty $w_f$ yet raises the optimum by $g w_f$; the net penalty transferred to other disagreements is therefore $(1+g)w_f$. At the pooled median this is $3.828$ times the target's weight. This is useful sensitivity information, but not evidence that ASPCR-DAG is ``more contestable'': its weights, trace, objective, and operation differ from OptABA-PC's.

The absence of zero gaps agrees only at the level of solver sensitivity: neither saved optimum fails a targeted fact merely because it selected one member of a retaining/failing tie.  It does not transfer OptABA-PC's correction-set semantics, total fact-enforcement result, or exact response guarantee to ASPCR-DAG.  Nor should the normalised values be ranked: ASPCR-DAG uses all $\binom{n}{2}2^{n-2}$ native Bayesian tests and log weights, whereas OptABA-PC uses the matched adaptive G$^2$ trace and corresponding release weights.

\subsection{Additional Empirical Results and Discussion}
\label{app:additional-experiment-details}

Tables~\ref{tab:absolute-results}--\ref{tab:final-runtime-completion} give the complete evidence behind the compact primary comparison for the fixed bnlearn networks and the ER/SF graphs under the setup in the main paper. Tables~\ref{tab:absolute-results} and~\ref{tab:absolute-cpdag-results} report absolute ABA-PC--OptABA-PC fact/time and CPDAG results for all eight datasets, including Earthquake; Table~\ref{tab:final-fact-metrics} gives the fact counts and Table~\ref{tab:final-fact-scores} their precision, recall, and F1; Table~\ref{tab:final-graph-metrics} compares graph-level results; Tables~\ref{tab:final-cpdag-ranges}--\ref{tab:final-dag-f1-ranges} give SHD and F1 reconstruction ranges; and Table~\ref{tab:final-runtime-completion} reports timing, output-set sizes, and CI compatibility. Values are absolute mean $\pm$ standard deviation. Stars compare OptABA-PC with ABA-PC; daggers compare an all other baselines with OptABA-PC. No missing value is imputed. ASPCR-DAG is reported on Cancer, Earthquake, Survey, ER(5), and SF(5), with 50 matched runs on each.

\begin{table}[t]
\caption{Absolute retained-fact quality and runtime for the primary ABA-PC--OptABA-PC comparison (mean $\pm$ standard deviation over 50 matched seeds, $N=5000$, $\alpha=0.01$). Both methods repair the same recorded CI tests. All eight datasets are included, with node counts in brackets. Correct retained counts facts agreeing with $d$-separation in the generating DAG; Fact F1 scores retained-fact classification. Time includes program construction, solving, and compatible-set evaluation. Arrows indicate preferred directions. Only a significant paired ABA-PC--OptABA-PC winner is bold and starred: one, two, and three stars denote BH-adjusted $p<0.05$, $0.01$, and $0.001$ from two-sided Wilcoxon tests. CPDAG evaluation is in Table~\ref{tab:absolute-cpdag-results}.}
\label{tab:absolute-results}
\centering
\small
\setlength{\tabcolsep}{3pt}
\begin{tabular*}{\linewidth}{@{\extracolsep{\fill}}llccc@{}}
\toprule
Dataset & Method & Correct retained $\uparrow$ & Fact F1 $\uparrow$ & Time (s) $\downarrow$\\
\midrule
Cancer (5) & ABA-PC & $24.0\pm8.7$ & $0.904\pm0.046$ & $1.51\pm0.48$\\
 & OptABA-PC & $\mathbf{25.6\pm9.8}^{***}$ & $0.900\pm0.064$ & $\mathbf{1.28\pm0.38}^{***}$\\
\addlinespace[1pt]
Earthquake (5) & ABA-PC & $54.1\pm0.8$ & $0.999\pm0.007$ & $1.83\pm0.83$\\
 & OptABA-PC & $54.2\pm0.8$ & $0.999\pm0.003$ & $\mathbf{1.68\pm0.81}^{***}$\\
\addlinespace[1pt]
Survey (6) & ABA-PC & $24.8\pm9.6$ & $\mathbf{0.789\pm0.075}^{**}$ & $15.7\pm30.2$\\
 & OptABA-PC & $\mathbf{29.7\pm12.3}^{***}$ & $0.742\pm0.107$ & $2.94\pm1.75$\\
\addlinespace[1pt]
Asia (8) & ABA-PC & $105.8\pm8.8$ & $0.891\pm0.018$ & $95.4\pm48.8$\\
 & OptABA-PC & $\mathbf{127.6\pm9.7}^{***}$ & $\mathbf{0.973\pm0.015}^{***}$ & $\mathbf{66.5\pm84.9}^{*}$\\
\addlinespace[1pt]
ER (5) & ABA-PC & $30.9\pm7.3$ & $0.953\pm0.049$ & $1.01\pm0.15$\\
 & OptABA-PC & $\mathbf{32.2\pm7.3}^{***}$ & $\mathbf{0.969\pm0.039}^{**}$ & $\mathbf{0.90\pm0.08}^{***}$\\
\addlinespace[1pt]
ER (8) & ABA-PC & $71.4\pm23.5$ & $0.871\pm0.069$ & $\mathbf{48.0\pm54.3}^{***}$\\
 & OptABA-PC & $\mathbf{86.8\pm25.7}^{***}$ & $\mathbf{0.953\pm0.043}^{***}$ & $139.1\pm139.4$\\
\addlinespace[1pt]
SF (5) & ABA-PC & $29.0\pm7.2$ & $0.953\pm0.056$ & $1.85\pm2.03$\\
 & OptABA-PC & $\mathbf{30.7\pm6.0}^{***}$ & $\mathbf{0.987\pm0.026}^{***}$ & $\mathbf{1.76\pm2.05}^{***}$\\
\addlinespace[1pt]
SF (8) & ABA-PC & $72.1\pm37.1$ & $0.838\pm0.065$ & $\mathbf{42.3\pm53.4}^{**}$\\
 & OptABA-PC & $\mathbf{93.2\pm45.4}^{***}$ & $\mathbf{0.931\pm0.051}^{***}$ & $99.1\pm122.0$\\
\bottomrule
\end{tabular*}
\end{table}

\begin{table}[t]
\caption{Absolute CPDAG evaluation for the primary ABA-PC--OptABA-PC comparison on the same 50 matched seeds per dataset as Table~\ref{tab:absolute-results}. Counts, F1, and SHD are mean $\pm$ standard deviation. Node counts are in brackets. \#CPDAGs counts saved equivalence classes satisfying every retained post-repair fact; F1 and SHD average reconstruction quality over these budget-limited sets. Arrows indicate preferred directions. Only a significant paired ABA-PC--OptABA-PC winner is bold and starred: one, two, and three stars denote BH-adjusted $p<0.05$, $0.01$, and $0.001$ from two-sided Wilcoxon tests.}
\label{tab:absolute-cpdag-results}
\centering
\small
\setlength{\tabcolsep}{3pt}
\begin{tabular*}{\linewidth}{@{\extracolsep{\fill}}llccc@{}}
\toprule
Dataset & Method & \#CPDAGs $\downarrow$ & CPDAG F1 $\uparrow$ & CPDAG SHD $\downarrow$\\
\midrule
Cancer (5) & ABA-PC & $15.1\pm28.8$ & $0.522\pm0.214$ & $2.81\pm1.18$\\
 & OptABA-PC & $\mathbf{1.0\pm0.0}^{***}$ & $\mathbf{0.576\pm0.259}^{*}$ & $\mathbf{1.84\pm1.11}^{***}$\\
\addlinespace[1pt]
Earthquake (5) & ABA-PC & $1.0\pm0.3$ & $0.991\pm0.033$ & $0.09\pm0.30$\\
 & OptABA-PC & $1.0\pm0.0$ & $0.991\pm0.030$ & $0.08\pm0.27$\\
\addlinespace[1pt]
Survey (6) & ABA-PC & $1278.5\pm2807.2$ & $0.438\pm0.088$ & $7.14\pm1.39$\\
 & OptABA-PC & $\mathbf{1.2\pm0.7}^{***}$ & $\mathbf{0.597\pm0.127}^{***}$ & $\mathbf{4.35\pm1.09}^{***}$\\
\addlinespace[1pt]
Asia (8) & ABA-PC & $10822.7\pm7112.3$ & $0.224\pm0.073$ & $14.4\pm3.26$\\
 & OptABA-PC & $\mathbf{1.2\pm0.4}^{***}$ & $\mathbf{0.536\pm0.061}^{***}$ & $\mathbf{3.86\pm0.60}^{***}$\\
\addlinespace[1pt]
ER (5) & ABA-PC & $8.2\pm13.0$ & $0.562\pm0.224$ & $2.88\pm2.01$\\
 & OptABA-PC & $\mathbf{3.0\pm3.4}^{***}$ & $\mathbf{0.589\pm0.213}^{*}$ & $\mathbf{2.09\pm1.71}^{**}$\\
\addlinespace[1pt]
ER (8) & ABA-PC & $3840.0\pm5377.9$ & $0.478\pm0.179$ & $8.19\pm4.21$\\
 & OptABA-PC & $\mathbf{4.7\pm9.5}^{***}$ & $\mathbf{0.627\pm0.167}^{***}$ & $\mathbf{3.53\pm1.89}^{***}$\\
\addlinespace[1pt]
SF (5) & ABA-PC & $88.4\pm202.7$ & $0.670\pm0.176$ & $2.77\pm2.10$\\
 & OptABA-PC & $\mathbf{85.1\pm203.9}^{**}$ & $\mathbf{0.737\pm0.148}^{**}$ & $\mathbf{2.03\pm1.97}^{**}$\\
\addlinespace[1pt]
SF (8) & ABA-PC & $3108.1\pm4287.4$ & $0.498\pm0.154$ & $8.15\pm3.77$\\
 & OptABA-PC & $\mathbf{170.3\pm1023.2}^{***}$ & $\mathbf{0.629\pm0.129}^{***}$ & $\mathbf{4.42\pm2.05}^{***}$\\
\bottomrule
\end{tabular*}
\end{table}

\paragraph{Graph-level findings.} Table~\ref{tab:final-graph-metrics} confirms that optimal repair improves mean CPDAG reconstruction over heuristic repair on seven datasets, while Earthquake is effectively tied. Across the five matched ASPCR-DAG datasets, OptABA-PC has significantly higher mean CPDAG F1 throughout and lower SHD on four; Survey SHD is indistinguishable. Its returned-DAG F1 is likewise higher on all five and SHD on four; SF(5) DAG SHD is indistinguishable. Native discrete FGS is a strong graph estimator: it significantly outperforms OptABA-PC on both mean CPDAG metrics for Cancer, ER(5), ER(8), and SF(5), on F1 for Survey, and on SHD for SF(8). OptABA-PC significantly outperforms FGS on both metrics for Earthquake and on SHD for Asia; the remaining mean CPDAG differences are not significant. This mixed comparison reinforces the distinction between reconstruction and repair: FGS selects a graph by score but provides no retained CI set, compatibility audit, or contestation object. The ASPCR-DAG comparisons use the same saved samples and generating DAG, but their different fact universes, weights, and optimisation costs are not numerically comparable. MPC is stronger on both mean CPDAG metrics for Survey and the four synthetic datasets; Cancer favours MPC only on CPDAG F1, Asia favours OptABA-PC on CPDAG F1, and Earthquake is tied. These are graph-estimation comparisons, not fact-repair comparisons.

The synthetic MPC result is consistent with Majority-PC's specialisation. It returns one targeted CPDAG when structurally valid, whose reversible orientations remain uncommitted, whereas OptABA-PC is evaluated over the saved, budget-limited structural set admitted by its selected coherent fact repair. OptABA-PC admits on average $3.0$, $4.7$, $85.1$, and $170.3$ compatible CPDAGs on ER(5), ER(8), SF(5), and SF(8), respectively (Table~\ref{tab:final-graph-metrics}). MPC beats the OptABA-PC average on all four, but does not significantly beat its best compatible CPDAG on ER(5), ER(8), or SF(5); SF(8) remains a significant MPC advantage (Tables~\ref{tab:final-cpdag-ranges} and~\ref{tab:final-cpdag-f1-ranges}). This narrowing should not be mistaken for a best-DAG win by OptABA-PC. After expanding each valid class, MPC has significantly better best-DAG SHD and F1 on ER(5), SF(5), and SF(8), and better best-DAG SHD on ER(8), where best-DAG F1 is indistinguishable (Tables~\ref{tab:final-dag-ranges} and~\ref{tab:final-dag-f1-ranges}). Thus comparing MPC's single partially oriented class with OptABA-PC's saved output set explains part, but not all, of MPC's average margin. In this broader sparse primary setting, full-trace compatibility holds for only $22\%$, $6\%$, $34\%$, and $4\%$ of the corresponding MPC outputs (Table~\ref{tab:final-runtime-completion}); Table~\ref{tab:contestability-margin-comparison} gives the separate exact 250-trace comparison. Its structural advantage therefore does not establish a fact-repair advantage.

The range tables further show that OptABA-PC significantly improves worst CPDAG and worst DAG SHD/F1 over ABA-PC on seven of eight datasets.  Best-of-set values can favour ABA-PC because its much larger output set has more opportunities to contain one unusually accurate member; average and worst values expose the cost of that ambiguity. Compatibility is stricter than structural validity: across the broader eight-dataset primary setting, only 80 of 400 MPC outputs have a CPDAG satisfying every recorded G$^2$ outcome, while Cancer, Survey, and Asia have none (Table~\ref{tab:final-runtime-completion}). For saved ASPCR-DAG outputs the corresponding native-trace count is 32 of 250: 22 on Earthquake, six on ER(5), and four on SF(5). FGS has no CI trace and is therefore marked \texttt{--} rather than zero.

\paragraph{Fact quality and computation.} Table~\ref{tab:final-fact-metrics} reports retained correct/false facts; Table~\ref{tab:final-fact-scores} reports their precision, recall, and FactF1. Table~\ref{tab:final-runtime-completion} reports timing and structural and CI-compatible output counts. Paired ABA-PC--OptABA-PC timings include program construction, solving, and compatible-set evaluation on a like-for-like basis. Cross-method scopes differ: MPC/FGS time the learner call, while ASPCR-DAG reports its native end-to-end R run. They describe implementation cost, not a controlled speed ranking.

\begin{table}[t]
\caption{Graph reconstruction and CI-compatible equivalence classes on the matched primary experiments (mean $\pm$ standard deviation over seeds). Numbers in brackets after dataset names give the number of nodes. ABA-PC and OptABA-PC CPDAG values average over every equivalence class admitted by the selected repair; MPC and FGS return one CPDAG when their output is structurally valid, and ASPCR-DAG returns one DAG whose equivalence class supplies its CPDAG score. DAG columns average every compatible DAG for ABA-PC/OptABA-PC and every consistent extension of a valid MPC/FGS CPDAG; ASPCR-DAG is scored on its returned DAG. The final column counts CPDAGs satisfying the enforced evidence: retained post-repair facts for ABA-PC/OptABA-PC and every recorded test for MPC/ASPCR-DAG; it is undefined for score-based FGS (`--'). Stars identify the significant ABA-PC--OptABA-PC winner. Daggers on an external method indicate a paired difference from OptABA-PC; bold daggers favour that external method, whereas unbold daggers favour OptABA-PC. One, two, and three symbols denote BH-adjusted $p<0.05$, $0.01$, and $0.001$. Subscripts occur only for MPC DAG-extension metrics when its endpoint output is invalid: $n=38$ on Survey and $n=49$ on ER(8).}
\label{tab:final-graph-metrics}
\centering
\scriptsize
\setlength{\tabcolsep}{3pt}
\resizebox{\linewidth}{!}{%
\begin{tabular}{llccccc}
\toprule
Dataset & Method & CPDAG SHD $\downarrow$ & CPDAG F1 $\uparrow$ & DAG SHD $\downarrow$ & DAG F1 $\uparrow$ & Compat. CPDAGs $\downarrow$\\
\midrule
Cancer (5) & ABA-PC & $2.81\pm1.18$ & $0.522\pm0.214$ & $2.98\pm1.20$ & $0.483\pm0.223$ & $15.1\pm28.8$\\
 & OptABA-PC & $\mathbf{1.84\pm1.11}^{***}$ & $\mathbf{0.576\pm0.259}^{*}$ & $\mathbf{1.89\pm1.14}^{***}$ & $\mathbf{0.563\pm0.262}^{*}$ & $\mathbf{1.0\pm0.0}^{***}$\\
 & MPC & $1.58\pm0.97$ & $\mathbf{0.661\pm0.202}^{\dagger}$ & $1.70\pm1.01$ & $0.623\pm0.222$ & $0.0\pm0.0$\\
 & FGS & $\mathbf{1.12\pm0.96}^{\dagger\dagger}$ & $\mathbf{0.742\pm0.212}^{\dagger\dagger}$ & $\mathbf{1.11\pm0.94}^{\dagger\dagger\dagger}$ & $\mathbf{0.739\pm0.215}^{\dagger\dagger}$ & --\\
 & ASPCR-DAG & $2.24\pm0.56^{\dagger}$ & $0.490\pm0.123^{\dagger}$ & $2.34\pm0.89^{\dagger}$ & $0.462\pm0.225^{\dagger}$ & $0.0\pm0.0$\\
\midrule
Earthquake (5) & ABA-PC & $0.09\pm0.30$ & $0.991\pm0.033$ & $0.10\pm0.34$ & $0.989\pm0.039$ & $1.0\pm0.3$\\
 & OptABA-PC & $0.08\pm0.27$ & $0.991\pm0.030$ & $0.08\pm0.27$ & $0.991\pm0.030$ & $1.0\pm0.0$\\
 & MPC & $0.20\pm0.73$ & $0.991\pm0.030$ & $0.08\pm0.27$ & $0.991\pm0.030$ & $0.9\pm0.2$\\
 & FGS & $4.64\pm1.34^{\dagger\dagger\dagger}$ & $0.422\pm0.143^{\dagger\dagger\dagger}$ & $4.41\pm1.15^{\dagger\dagger\dagger}$ & $0.417\pm0.149^{\dagger\dagger\dagger}$ & --\\
 & ASPCR-DAG & $1.08\pm1.16^{\dagger\dagger\dagger}$ & $0.866\pm0.168^{\dagger\dagger\dagger}$ & $1.06\pm1.20^{\dagger\dagger\dagger}$ & $0.871\pm0.176^{\dagger\dagger\dagger}$ & $0.4\pm0.5$\\
\midrule
Survey (6) & ABA-PC & $7.14\pm1.39$ & $0.438\pm0.088$ & $6.73\pm1.39$ & $0.340\pm0.103$ & $1278.5\pm2807.2$\\
 & OptABA-PC & $\mathbf{4.35\pm1.09}^{***}$ & $\mathbf{0.597\pm0.127}^{***}$ & $\mathbf{3.77\pm0.91}^{***}$ & $\mathbf{0.482\pm0.140}^{***}$ & $\mathbf{1.2\pm0.7}^{***}$\\
 & MPC & $\mathbf{3.68\pm1.49}^{\dagger\dagger}$ & $\mathbf{0.815\pm0.088}^{\dagger\dagger\dagger}$ & $\mathbf{3.17\pm1.49}^{\dagger}_{38}$ & $0.528\pm0.242_{38}$ & $0.0\pm0.0$\\
 & FGS & $3.88\pm1.55$ & $\mathbf{0.708\pm0.187}^{\dagger\dagger\dagger}$ & $3.47\pm1.28$ & $0.465\pm0.218$ & --\\
 & ASPCR-DAG & $4.52\pm1.23$ & $0.496\pm0.154^{\dagger\dagger}$ & $4.16\pm1.23^{\dagger}$ & $0.385\pm0.220^{\dagger\dagger}$ & $0.0\pm0.0$\\
\midrule
Asia (8) & ABA-PC & $14.4\pm3.26$ & $0.224\pm0.073$ & $14.0\pm2.81$ & $0.310\pm0.095$ & $10822.7\pm7112.3$\\
 & OptABA-PC & $\mathbf{3.86\pm0.60}^{***}$ & $\mathbf{0.536\pm0.061}^{***}$ & $\mathbf{4.51\pm0.58}^{***}$ & $\mathbf{0.665\pm0.009}^{***}$ & $\mathbf{1.2\pm0.4}^{***}$\\
 & MPC & $3.82\pm0.44$ & $0.483\pm0.045^{\dagger\dagger\dagger}$ & $\mathbf{3.92\pm0.25}^{\dagger\dagger\dagger}$ & $0.620\pm0.013^{\dagger\dagger\dagger}$ & $0.0\pm0.0$\\
 & FGS & $6.66\pm2.22^{\dagger\dagger\dagger}$ & $0.561\pm0.102$ & $6.01\pm1.70^{\dagger\dagger\dagger}$ & $0.610\pm0.105^{\dagger\dagger}$ & --\\
\midrule
ER (5) & ABA-PC & $2.88\pm2.01$ & $0.562\pm0.224$ & $3.17\pm1.73$ & $0.567\pm0.215$ & $8.2\pm13.0$\\
 & OptABA-PC & $\mathbf{2.09\pm1.71}^{**}$ & $\mathbf{0.589\pm0.213}^{*}$ & $\mathbf{2.49\pm1.35}^{***}$ & $\mathbf{0.620\pm0.188}^{**}$ & $\mathbf{3.0\pm3.4}^{***}$\\
 & MPC & $\mathbf{0.96\pm1.32}^{\dagger\dagger}$ & $\mathbf{0.712\pm0.251}^{\dagger\dagger\dagger}$ & $\mathbf{1.23\pm0.94}^{\dagger\dagger\dagger}$ & $\mathbf{0.768\pm0.177}^{\dagger\dagger\dagger}$ & $0.2\pm0.4$\\
 & FGS & $\mathbf{1.14\pm1.44}^{\dagger\dagger}$ & $\mathbf{0.677\pm0.239}^{\dagger\dagger}$ & $\mathbf{1.58\pm1.11}^{\dagger\dagger\dagger}$ & $\mathbf{0.710\pm0.196}^{\dagger\dagger}$ & --\\
 & ASPCR-DAG & $2.90\pm1.25^{\dagger\dagger}$ & $0.377\pm0.292^{\dagger\dagger\dagger}$ & $3.26\pm1.19^{\dagger\dagger}$ & $0.426\pm0.251^{\dagger\dagger\dagger}$ & $0.1\pm0.3$\\
\midrule
ER (8) & ABA-PC & $8.19\pm4.21$ & $0.478\pm0.179$ & $7.99\pm3.77$ & $0.496\pm0.181$ & $3840.0\pm5377.9$\\
 & OptABA-PC & $\mathbf{3.53\pm1.89}^{***}$ & $\mathbf{0.627\pm0.167}^{***}$ & $\mathbf{3.85\pm1.58}^{***}$ & $\mathbf{0.650\pm0.137}^{***}$ & $\mathbf{4.7\pm9.5}^{***}$\\
 & MPC & $\mathbf{2.12\pm1.79}^{\dagger\dagger\dagger}$ & $\mathbf{0.704\pm0.137}^{\dagger\dagger}$ & $\mathbf{2.38\pm1.21}^{\dagger\dagger\dagger}_{49}$ & $\mathbf{0.727\pm0.140}^{\dagger\dagger}_{49}$ & $0.1\pm0.2$\\
 & FGS & $\mathbf{1.56\pm1.99}^{\dagger\dagger\dagger}$ & $\mathbf{0.744\pm0.157}^{\dagger\dagger\dagger}$ & $\mathbf{2.07\pm1.41}^{\dagger\dagger\dagger}$ & $\mathbf{0.774\pm0.139}^{\dagger\dagger\dagger}$ & --\\
\midrule
SF (5) & ABA-PC & $2.77\pm2.10$ & $0.670\pm0.176$ & $3.07\pm1.94$ & $0.594\pm0.213$ & $88.4\pm202.7$\\
 & OptABA-PC & $\mathbf{2.03\pm1.97}^{**}$ & $\mathbf{0.737\pm0.148}^{**}$ & $\mathbf{2.52\pm1.89}^{**}$ & $\mathbf{0.650\pm0.197}^{**}$ & $\mathbf{85.1\pm203.9}^{**}$\\
 & MPC & $\mathbf{0.34\pm0.80}^{\dagger\dagger\dagger}$ & $\mathbf{0.874\pm0.173}^{\dagger\dagger\dagger}$ & $\mathbf{0.48\pm0.58}^{\dagger\dagger\dagger}$ & $\mathbf{0.893\pm0.120}^{\dagger\dagger\dagger}$ & $0.3\pm0.5$\\
 & FGS & $\mathbf{0.56\pm0.93}^{\dagger\dagger\dagger}$ & $\mathbf{0.854\pm0.173}^{\dagger\dagger\dagger}$ & $\mathbf{0.69\pm0.85}^{\dagger\dagger\dagger}$ & $\mathbf{0.861\pm0.143}^{\dagger\dagger\dagger}$ & --\\
 & ASPCR-DAG & $2.76\pm1.51^{\dagger}$ & $0.448\pm0.303^{\dagger\dagger\dagger}$ & $2.60\pm1.26$ & $0.448\pm0.283^{\dagger\dagger\dagger}$ & $0.1\pm0.3$\\
\midrule
SF (8) & ABA-PC & $8.15\pm3.77$ & $0.498\pm0.154$ & $8.14\pm3.46$ & $0.481\pm0.166$ & $3108.1\pm4287.4$\\
 & OptABA-PC & $\mathbf{4.42\pm2.05}^{***}$ & $\mathbf{0.629\pm0.129}^{***}$ & $\mathbf{4.76\pm1.97}^{***}$ & $\mathbf{0.601\pm0.146}^{***}$ & $\mathbf{170.3\pm1023.2}^{***}$\\
 & MPC & $\mathbf{1.58\pm1.49}^{\dagger\dagger\dagger}$ & $\mathbf{0.774\pm0.163}^{\dagger\dagger\dagger}$ & $\mathbf{1.88\pm1.23}^{\dagger\dagger\dagger}$ & $\mathbf{0.773\pm0.143}^{\dagger\dagger\dagger}$ & $0.0\pm0.2$\\
 & FGS & $\mathbf{2.50\pm1.96}^{\dagger\dagger\dagger}$ & $0.683\pm0.246$ & $\mathbf{2.82\pm1.88}^{\dagger\dagger\dagger}$ & $0.664\pm0.241$ & --\\
\bottomrule
\end{tabular}}
\end{table}

\begin{table}[t]
\caption{CI-fact counts on the matched primary experiments (mean $\pm$ standard deviation over 50 seeds per reported row). Node counts are in brackets. Correct available counts recorded tests agreeing with $d$-separation in the generating DAG; correct/false kept split retained facts. ABA-PC and OptABA-PC repair the identical test set. Only a significant paired ABA-PC--OptABA-PC winner is bold and starred: one, two, and three stars denote BH-adjusted $p<0.05$, $0.01$, and $0.001$ from two-sided Wilcoxon tests. For ASPCR-DAG, \emph{kept} means its returned DAG satisfies a native soft test relation: descriptive graph--test agreement, not an explicit retained CI set. MPC and FGS expose no comparable fact accounting. Precision, recall, and F1 are in Table~\ref{tab:final-fact-scores}.}
\label{tab:final-fact-metrics}
\centering
\small
\setlength{\tabcolsep}{3pt}
\begin{tabular*}{\linewidth}{@{\extracolsep{\fill}}llccc@{}}
\toprule
Dataset & Method & Correct available & Correct kept & False kept\\
\midrule
Cancer (5) & ABA-PC & $27.2\pm10.1$ & $24.0\pm8.7$ & $\mathbf{1.6\pm1.5}^{***}$\\
 & OptABA-PC & $27.2\pm10.1$ & $\mathbf{25.6\pm9.8}^{***}$ & $3.3\pm2.2$\\
 & ASPCR-DAG & $53.7\pm3.4$ & $48.0\pm7.2$ & $12.9\pm8.5$\\
\midrule
Earthquake (5) & ABA-PC & $54.2\pm0.9$ & $54.1\pm0.8$ & $0.1\pm0.3$\\
 & OptABA-PC & $54.2\pm0.9$ & $54.2\pm0.8$ & $0.1\pm0.3$\\
 & ASPCR-DAG & $75.6\pm2.9$ & $74.7\pm3.8$ & $3.7\pm3.0$\\
\midrule
Survey (6) & ABA-PC & $34.9\pm13.2$ & $24.8\pm9.6$ & $\mathbf{2.4\pm1.5}^{***}$\\
 & OptABA-PC & $34.9\pm13.2$ & $\mathbf{29.7\pm12.3}^{***}$ & $13.4\pm4.7$\\
 & ASPCR-DAG & $135.0\pm11.2$ & $116.1\pm12.8$ & $82.0\pm18.5$\\
\midrule
Asia (8) & ABA-PC & $129.2\pm9.6$ & $105.8\pm8.8$ & $\mathbf{2.4\pm0.9}^{***}$\\
 & OptABA-PC & $129.2\pm9.6$ & $\mathbf{127.6\pm9.7}^{***}$ & $5.4\pm3.1$\\
\midrule
ER (5) & ABA-PC & $32.7\pm7.3$ & $30.9\pm7.3$ & $\mathbf{1.2\pm1.6}^{**}$\\
 & OptABA-PC & $32.7\pm7.3$ & $\mathbf{32.2\pm7.3}^{***}$ & $1.5\pm1.9$\\
 & ASPCR-DAG & $52.1\pm9.9$ & $47.3\pm10.1$ & $22.3\pm11.8$\\
\midrule
ER (8) & ABA-PC & $89.2\pm26.2$ & $71.4\pm23.5$ & $\mathbf{2.7\pm2.6}^{***}$\\
 & OptABA-PC & $89.2\pm26.2$ & $\mathbf{86.8\pm25.7}^{***}$ & $6.0\pm5.8$\\
\midrule
SF (5) & ABA-PC & $30.9\pm5.8$ & $29.0\pm7.2$ & $0.5\pm0.6$\\
 & OptABA-PC & $30.9\pm5.8$ & $\mathbf{30.7\pm6.0}^{***}$ & $0.5\pm0.9$\\
 & ASPCR-DAG & $49.0\pm9.7$ & $42.9\pm9.4$ & $27.0\pm13.2$\\
\midrule
SF (8) & ABA-PC & $96.6\pm45.5$ & $72.1\pm37.1$ & $\mathbf{2.6\pm2.8}^{***}$\\
 & OptABA-PC & $96.6\pm45.5$ & $\mathbf{93.2\pm45.4}^{***}$ & $9.0\pm7.3$\\
\bottomrule
\end{tabular*}
\end{table}

\begin{table}[t]
\caption{CI-fact precision, recall, and F1 on the matched primary experiments (mean $\pm$ standard deviation over 50 seeds per reported row; higher is better). Node counts are in brackets. Retained correct tests are true positives, retained false tests are false positives, and released correct tests are false negatives; counts are in Table~\ref{tab:final-fact-metrics}. For ASPCR-DAG, retention denotes satisfaction of native soft tests by its returned DAG, not an explicit retained CI set. MPC and FGS expose no comparable fact accounting. Only a significant paired ABA-PC--OptABA-PC winner is bold and starred: one, two, and three stars denote BH-adjusted $p<0.05$, $0.01$, and $0.001$ from two-sided Wilcoxon tests.}
\label{tab:final-fact-scores}
\centering
\small
\setlength{\tabcolsep}{3pt}
\begin{tabular*}{\linewidth}{@{\extracolsep{\fill}}llccc@{}}
\toprule
Dataset & Method & Precision & Recall & Fact F1\\
\midrule
Cancer (5) & ABA-PC & $\mathbf{0.927\pm0.069}^{***}$ & $0.887\pm0.058$ & $0.904\pm0.046$\\
 & OptABA-PC & $0.866\pm0.091$ & $\mathbf{0.941\pm0.045}^{***}$ & $0.900\pm0.064$\\
 & ASPCR-DAG & $0.792\pm0.132$ & $0.892\pm0.114$ & $0.837\pm0.120$\\
\midrule
Earthquake (5) & ABA-PC & $0.999\pm0.005$ & $0.999\pm0.010$ & $0.999\pm0.007$\\
 & OptABA-PC & $0.999\pm0.005$ & $1.000\pm0.002$ & $0.999\pm0.003$\\
 & ASPCR-DAG & $0.952\pm0.039$ & $0.987\pm0.018$ & $0.969\pm0.026$\\
\midrule
Survey (6) & ABA-PC & $\mathbf{0.894\pm0.088}^{***}$ & $0.716\pm0.105$ & $\mathbf{0.789\pm0.075}^{**}$\\
 & OptABA-PC & $0.669\pm0.126$ & $\mathbf{0.842\pm0.082}^{***}$ & $0.742\pm0.107$\\
 & ASPCR-DAG & $0.589\pm0.077$ & $0.862\pm0.082$ & $0.697\pm0.070$\\
\midrule
Asia (8) & ABA-PC & $\mathbf{0.978\pm0.008}^{***}$ & $0.819\pm0.032$ & $0.891\pm0.018$\\
 & OptABA-PC & $0.959\pm0.023$ & $\mathbf{0.988\pm0.009}^{***}$ & $\mathbf{0.973\pm0.015}^{***}$\\
\midrule
ER (5) & ABA-PC & $\mathbf{0.963\pm0.051}^{*}$ & $0.946\pm0.064$ & $0.953\pm0.049$\\
 & OptABA-PC & $0.955\pm0.056$ & $\mathbf{0.984\pm0.027}^{***}$ & $\mathbf{0.969\pm0.039}^{**}$\\
 & ASPCR-DAG & $0.684\pm0.155$ & $0.907\pm0.085$ & $0.770\pm0.116$\\
\midrule
ER (8) & ABA-PC & $\mathbf{0.962\pm0.031}^{***}$ & $0.800\pm0.104$ & $0.871\pm0.069$\\
 & OptABA-PC & $0.934\pm0.055$ & $\mathbf{0.973\pm0.033}^{***}$ & $\mathbf{0.953\pm0.043}^{***}$\\
\midrule
SF (5) & ABA-PC & $0.982\pm0.025$ & $0.929\pm0.090$ & $0.953\pm0.056$\\
 & OptABA-PC & $0.982\pm0.032$ & $\mathbf{0.992\pm0.022}^{***}$ & $\mathbf{0.987\pm0.026}^{***}$\\
 & ASPCR-DAG & $0.623\pm0.169$ & $0.880\pm0.108$ & $0.717\pm0.128$\\
\midrule
SF (8) & ABA-PC & $\mathbf{0.957\pm0.045}^{***}$ & $0.752\pm0.106$ & $0.838\pm0.065$\\
 & OptABA-PC & $0.903\pm0.066$ & $\mathbf{0.962\pm0.043}^{***}$ & $\mathbf{0.931\pm0.051}^{***}$\\
\bottomrule
\end{tabular*}
\end{table}

\begin{table}[t]
\caption{Average, best, and worst CPDAG SHD within each run of the matched primary experiments (mean $\pm$ standard deviation over seeds; lower is better). ABA-PC and OptABA-PC are scored over the saved classes satisfying their selected facts. MPC and FGS return one class when structurally valid; ASPCR-DAG's DAG supplies one class, so their average, best, and worst values coincide within a run. Stars mark the significant ABA-PC--OptABA-PC winner. Daggers compare an external method with OptABA-PC; bold daggers favour that method, unbold daggers favour OptABA-PC. One, two, and three symbols denote BH-adjusted $p<0.05$, $0.01$, and $0.001$. Compatibility counts are in Table~\ref{tab:final-runtime-completion}.}
\label{tab:final-cpdag-ranges}
\centering
\small
\setlength{\tabcolsep}{3pt}
\begin{tabular*}{\linewidth}{@{\extracolsep{\fill}}llccc@{}}
\toprule
Dataset & Method & SHD avg & SHD best & SHD worst\\
\midrule
Cancer (5) & ABA-PC & $2.81\pm1.18$ & $1.58\pm1.18$ & $3.84\pm1.80$\\
 & OptABA-PC & $\mathbf{1.84\pm1.11}^{***}$ & $1.84\pm1.11$ & $\mathbf{1.84\pm1.11}^{***}$\\
 & MPC & $1.58\pm0.97$ & $1.58\pm0.97$ & $1.58\pm0.97$\\
 & FGS & $\mathbf{1.12\pm0.96}^{\dagger\dagger}$ & $\mathbf{1.12\pm0.96}^{\dagger\dagger}$ & $\mathbf{1.12\pm0.96}^{\dagger\dagger}$\\
 & ASPCR-DAG & $2.24\pm0.56^{\dagger}$ & $2.24\pm0.56^{\dagger}$ & $2.24\pm0.56^{\dagger}$\\
\midrule
Earthquake (5) & ABA-PC & $0.09\pm0.30$ & $0.08\pm0.27$ & $0.10\pm0.36$\\
 & OptABA-PC & $0.08\pm0.27$ & $0.08\pm0.27$ & $0.08\pm0.27$\\
 & MPC & $0.20\pm0.73$ & $0.20\pm0.73$ & $0.20\pm0.73$\\
 & FGS & $4.64\pm1.34^{\dagger\dagger\dagger}$ & $4.64\pm1.34^{\dagger\dagger\dagger}$ & $4.64\pm1.34^{\dagger\dagger\dagger}$\\
 & ASPCR-DAG & $1.08\pm1.16^{\dagger\dagger\dagger}$ & $1.08\pm1.16^{\dagger\dagger\dagger}$ & $1.08\pm1.16^{\dagger\dagger\dagger}$\\
\midrule
Survey (6) & ABA-PC & $7.14\pm1.39$ & $\mathbf{3.18\pm1.66}^{**}$ & $9.96\pm2.34$\\
 & OptABA-PC & $\mathbf{4.35\pm1.09}^{***}$ & $4.26\pm1.08$ & $\mathbf{4.44\pm1.16}^{***}$\\
 & MPC & $\mathbf{3.68\pm1.49}^{\dagger\dagger}$ & $\mathbf{3.68\pm1.49}^{\dagger}$ & $\mathbf{3.68\pm1.49}^{\dagger\dagger}$\\
 & FGS & $3.88\pm1.55$ & $3.88\pm1.55$ & $3.88\pm1.55$\\
 & ASPCR-DAG & $4.52\pm1.23$ & $4.52\pm1.23$ & $4.52\pm1.23$\\
\midrule
Asia (8) & ABA-PC & $14.4\pm3.26$ & $8.26\pm2.59$ & $20.0\pm4.20$\\
 & OptABA-PC & $\mathbf{3.86\pm0.60}^{***}$ & $\mathbf{3.76\pm0.69}^{***}$ & $\mathbf{3.96\pm0.57}^{***}$\\
 & MPC & $3.82\pm0.44$ & $3.82\pm0.44$ & $3.82\pm0.44$\\
 & FGS & $6.66\pm2.22^{\dagger\dagger\dagger}$ & $6.66\pm2.22^{\dagger\dagger\dagger}$ & $6.66\pm2.22^{\dagger\dagger\dagger}$\\
\midrule
ER (5) & ABA-PC & $2.88\pm2.01$ & $1.68\pm1.96$ & $3.88\pm2.48$\\
 & OptABA-PC & $\mathbf{2.09\pm1.71}^{**}$ & $1.34\pm1.60$ & $\mathbf{2.72\pm2.14}^{**}$\\
 & MPC & $\mathbf{0.96\pm1.32}^{\dagger\dagger}$ & $0.96\pm1.32$ & $\mathbf{0.96\pm1.32}^{\dagger\dagger\dagger}$\\
 & FGS & $\mathbf{1.14\pm1.44}^{\dagger\dagger}$ & $1.14\pm1.44$ & $\mathbf{1.14\pm1.44}^{\dagger\dagger\dagger}$\\
 & ASPCR-DAG & $2.90\pm1.25^{\dagger\dagger}$ & $2.90\pm1.25^{\dagger\dagger\dagger}$ & $2.90\pm1.25$\\
\midrule
ER (8) & ABA-PC & $8.19\pm4.21$ & $3.72\pm2.88$ & $12.2\pm6.21$\\
 & OptABA-PC & $\mathbf{3.53\pm1.89}^{***}$ & $\mathbf{2.74\pm1.98}^{*}$ & $\mathbf{4.26\pm2.16}^{***}$\\
 & MPC & $\mathbf{2.12\pm1.79}^{\dagger\dagger\dagger}$ & $2.12\pm1.79$ & $\mathbf{2.12\pm1.79}^{\dagger\dagger\dagger}$\\
 & FGS & $\mathbf{1.56\pm1.99}^{\dagger\dagger\dagger}$ & $\mathbf{1.56\pm1.99}^{\dagger\dagger}$ & $\mathbf{1.56\pm1.99}^{\dagger\dagger\dagger}$\\
\midrule
SF (5) & ABA-PC & $2.77\pm2.10$ & $1.04\pm1.80$ & $4.56\pm3.36$\\
 & OptABA-PC & $\mathbf{2.03\pm1.97}^{**}$ & $\mathbf{0.60\pm1.46}^{*}$ & $\mathbf{3.54\pm3.39}^{**}$\\
 & MPC & $\mathbf{0.34\pm0.80}^{\dagger\dagger\dagger}$ & $0.34\pm0.80$ & $\mathbf{0.34\pm0.80}^{\dagger\dagger\dagger}$\\
 & FGS & $\mathbf{0.56\pm0.93}^{\dagger\dagger\dagger}$ & $0.56\pm0.93$ & $\mathbf{0.56\pm0.93}^{\dagger\dagger\dagger}$\\
 & ASPCR-DAG & $2.76\pm1.51^{\dagger}$ & $2.76\pm1.51^{\dagger\dagger\dagger}$ & $2.76\pm1.51$\\
\midrule
SF (8) & ABA-PC & $8.15\pm3.77$ & $3.40\pm2.81$ & $12.8\pm5.18$\\
 & OptABA-PC & $\mathbf{4.42\pm2.05}^{***}$ & $3.20\pm1.81$ & $\mathbf{5.54\pm2.82}^{***}$\\
 & MPC & $\mathbf{1.58\pm1.49}^{\dagger\dagger\dagger}$ & $\mathbf{1.58\pm1.49}^{\dagger\dagger\dagger}$ & $\mathbf{1.58\pm1.49}^{\dagger\dagger\dagger}$\\
 & FGS & $\mathbf{2.50\pm1.96}^{\dagger\dagger\dagger}$ & $2.50\pm1.96$ & $\mathbf{2.50\pm1.96}^{\dagger\dagger\dagger}$\\
\bottomrule
\end{tabular*}
\end{table}

\begin{table}[t]
\caption{Average, best, and worst CPDAG F1 within each run of the matched primary experiments (mean $\pm$ standard deviation over seeds; higher is better). ABA-PC and OptABA-PC are scored over the saved classes satisfying their selected facts. MPC and FGS return one class when structurally valid; ASPCR-DAG's DAG supplies one class, so their average, best, and worst values coincide within a run. Stars mark the significant ABA-PC--OptABA-PC winner. Daggers compare an external method with OptABA-PC; bold daggers favour that method, unbold daggers favour OptABA-PC. One, two, and three symbols denote BH-adjusted $p<0.05$, $0.01$, and $0.001$. Compatibility counts are in Table~\ref{tab:final-runtime-completion}.}
\label{tab:final-cpdag-f1-ranges}
\centering
\small
\setlength{\tabcolsep}{3pt}
\begin{tabular*}{\linewidth}{@{\extracolsep{\fill}}llccc@{}}
\toprule
Dataset & Method & F1 avg & F1 best & F1 worst\\
\midrule
Cancer (5) & ABA-PC & $0.522\pm0.214$ & $0.646\pm0.262$ & $0.440\pm0.222$\\
 & OptABA-PC & $\mathbf{0.576\pm0.259}^{*}$ & $0.576\pm0.259$ & $\mathbf{0.576\pm0.259}^{**}$\\
 & MPC & $\mathbf{0.661\pm0.202}^{\dagger}$ & $\mathbf{0.661\pm0.202}^{\dagger}$ & $\mathbf{0.661\pm0.202}^{\dagger}$\\
 & FGS & $\mathbf{0.742\pm0.212}^{\dagger\dagger}$ & $\mathbf{0.742\pm0.212}^{\dagger\dagger}$ & $\mathbf{0.742\pm0.212}^{\dagger\dagger}$\\
 & ASPCR-DAG & $0.490\pm0.123^{\dagger}$ & $0.490\pm0.123^{\dagger}$ & $0.490\pm0.123^{\dagger}$\\
\midrule
Earthquake (5) & ABA-PC & $0.991\pm0.033$ & $0.991\pm0.030$ & $0.989\pm0.038$\\
 & OptABA-PC & $0.991\pm0.030$ & $0.991\pm0.030$ & $0.991\pm0.030$\\
 & MPC & $0.991\pm0.030$ & $0.991\pm0.030$ & $0.991\pm0.030$\\
 & FGS & $0.422\pm0.143^{\dagger\dagger\dagger}$ & $0.422\pm0.143^{\dagger\dagger\dagger}$ & $0.422\pm0.143^{\dagger\dagger\dagger}$\\
 & ASPCR-DAG & $0.866\pm0.168^{\dagger\dagger\dagger}$ & $0.866\pm0.168^{\dagger\dagger\dagger}$ & $0.866\pm0.168^{\dagger\dagger\dagger}$\\
\midrule
Survey (6) & ABA-PC & $0.438\pm0.088$ & $\mathbf{0.702\pm0.154}^{**}$ & $0.228\pm0.120$\\
 & OptABA-PC & $\mathbf{0.597\pm0.127}^{***}$ & $0.602\pm0.130$ & $\mathbf{0.593\pm0.126}^{***}$\\
 & MPC & $\mathbf{0.815\pm0.088}^{\dagger\dagger\dagger}$ & $\mathbf{0.815\pm0.088}^{\dagger\dagger\dagger}$ & $\mathbf{0.815\pm0.088}^{\dagger\dagger\dagger}$\\
 & FGS & $\mathbf{0.708\pm0.187}^{\dagger\dagger\dagger}$ & $\mathbf{0.708\pm0.187}^{\dagger\dagger}$ & $\mathbf{0.708\pm0.187}^{\dagger\dagger\dagger}$\\
 & ASPCR-DAG & $0.496\pm0.154^{\dagger\dagger}$ & $0.496\pm0.154^{\dagger\dagger}$ & $0.496\pm0.154^{\dagger\dagger}$\\
\midrule
Asia (8) & ABA-PC & $0.224\pm0.073$ & $0.448\pm0.083$ & $0.111\pm0.080$\\
 & OptABA-PC & $\mathbf{0.536\pm0.061}^{***}$ & $\mathbf{0.540\pm0.063}^{***}$ & $\mathbf{0.533\pm0.060}^{***}$\\
 & MPC & $0.483\pm0.045^{\dagger\dagger\dagger}$ & $0.483\pm0.045^{\dagger\dagger\dagger}$ & $0.483\pm0.045^{\dagger\dagger\dagger}$\\
 & FGS & $0.561\pm0.102$ & $0.561\pm0.102$ & $0.561\pm0.102$\\
\midrule
ER (5) & ABA-PC & $0.562\pm0.224$ & $0.672\pm0.230$ & $0.475\pm0.252$\\
 & OptABA-PC & $\mathbf{0.589\pm0.213}^{*}$ & $0.657\pm0.227$ & $\mathbf{0.535\pm0.235}^{**}$\\
 & MPC & $\mathbf{0.712\pm0.251}^{\dagger\dagger\dagger}$ & $0.712\pm0.251$ & $\mathbf{0.712\pm0.251}^{\dagger\dagger\dagger}$\\
 & FGS & $\mathbf{0.677\pm0.239}^{\dagger\dagger}$ & $0.677\pm0.239$ & $\mathbf{0.677\pm0.239}^{\dagger\dagger\dagger}$\\
 & ASPCR-DAG & $0.377\pm0.292^{\dagger\dagger\dagger}$ & $0.377\pm0.292^{\dagger\dagger\dagger}$ & $0.377\pm0.292^{\dagger\dagger}$\\
\midrule
ER (8) & ABA-PC & $0.478\pm0.179$ & $0.687\pm0.149$ & $0.309\pm0.234$\\
 & OptABA-PC & $\mathbf{0.627\pm0.167}^{***}$ & $0.670\pm0.172$ & $\mathbf{0.586\pm0.182}^{***}$\\
 & MPC & $\mathbf{0.704\pm0.137}^{\dagger\dagger}$ & $0.704\pm0.137$ & $\mathbf{0.704\pm0.137}^{\dagger\dagger\dagger}$\\
 & FGS & $\mathbf{0.744\pm0.157}^{\dagger\dagger\dagger}$ & $\mathbf{0.744\pm0.157}^{\dagger}$ & $\mathbf{0.744\pm0.157}^{\dagger\dagger\dagger}$\\
\midrule
SF (5) & ABA-PC & $0.670\pm0.176$ & $0.813\pm0.193$ & $0.560\pm0.221$\\
 & OptABA-PC & $\mathbf{0.737\pm0.148}^{**}$ & $0.851\pm0.160$ & $\mathbf{0.646\pm0.201}^{**}$\\
 & MPC & $\mathbf{0.874\pm0.173}^{\dagger\dagger\dagger}$ & $0.874\pm0.173$ & $\mathbf{0.874\pm0.173}^{\dagger\dagger\dagger}$\\
 & FGS & $\mathbf{0.854\pm0.173}^{\dagger\dagger\dagger}$ & $0.854\pm0.173$ & $\mathbf{0.854\pm0.173}^{\dagger\dagger\dagger}$\\
 & ASPCR-DAG & $0.448\pm0.303^{\dagger\dagger\dagger}$ & $0.448\pm0.303^{\dagger\dagger\dagger}$ & $0.448\pm0.303^{\dagger\dagger\dagger}$\\
\midrule
SF (8) & ABA-PC & $0.498\pm0.154$ & $0.739\pm0.148$ & $0.286\pm0.185$\\
 & OptABA-PC & $\mathbf{0.629\pm0.129}^{***}$ & $0.690\pm0.142$ & $\mathbf{0.581\pm0.144}^{***}$\\
 & MPC & $\mathbf{0.774\pm0.163}^{\dagger\dagger\dagger}$ & $\mathbf{0.774\pm0.163}^{\dagger\dagger}$ & $\mathbf{0.774\pm0.163}^{\dagger\dagger\dagger}$\\
 & FGS & $0.683\pm0.246$ & $0.683\pm0.246$ & $\mathbf{0.683\pm0.246}^{\dagger\dagger}$\\
\bottomrule
\end{tabular*}
\end{table}

\begin{table}[t]
\caption{Average, best, and worst DAG SHD within each run of the matched primary experiments (mean $\pm$ standard deviation over seeds; lower is better). ABA-PC and OptABA-PC are scored over saved compatible DAGs; MPC and FGS over consistent extensions of each valid returned CPDAG; ASPCR-DAG on its single DAG, so its average, best, and worst coincide. Subscripts mark MPC's reduced samples: 38 on Survey and 49 on ER(8). Stars mark the significant ABA-PC--OptABA-PC winner. Daggers compare an external method with OptABA-PC; bold daggers favour that method, unbold daggers favour OptABA-PC. One, two, and three symbols denote BH-adjusted $p<0.05$, $0.01$, and $0.001$. Compatibility counts are in Table~\ref{tab:final-runtime-completion}.}
\label{tab:final-dag-ranges}
\centering
\small
\setlength{\tabcolsep}{3pt}
\begin{tabular*}{\linewidth}{@{\extracolsep{\fill}}llccc@{}}
\toprule
Dataset & Method & SHD avg & SHD best & SHD worst\\
\midrule
Cancer (5) & ABA-PC & $2.98\pm1.20$ & $1.44\pm1.18$ & $4.18\pm1.96$\\
 & OptABA-PC & $\mathbf{1.89\pm1.14}^{***}$ & $1.70\pm1.15$ & $\mathbf{2.00\pm1.21}^{***}$\\
 & MPC & $1.70\pm1.01$ & $\mathbf{0.86\pm0.61}^{\dagger\dagger\dagger}$ & $2.28\pm1.37$\\
 & FGS & $\mathbf{1.11\pm0.94}^{\dagger\dagger\dagger}$ & $\mathbf{0.90\pm0.81}^{\dagger\dagger\dagger}$ & $\mathbf{1.30\pm1.18}^{\dagger\dagger}$\\
 & ASPCR-DAG & $2.34\pm0.89^{\dagger}$ & $2.34\pm0.89^{\dagger\dagger}$ & $2.34\pm0.89$\\
\midrule
Earthquake (5) & ABA-PC & $0.10\pm0.34$ & $0.08\pm0.27$ & $0.12\pm0.48$\\
 & OptABA-PC & $0.08\pm0.27$ & $0.08\pm0.27$ & $0.08\pm0.27$\\
 & MPC & $0.08\pm0.27$ & $0.08\pm0.27$ & $0.08\pm0.27$\\
 & FGS & $4.41\pm1.15^{\dagger\dagger\dagger}$ & $3.16\pm1.11^{\dagger\dagger\dagger}$ & $5.32\pm1.38^{\dagger\dagger\dagger}$\\
 & ASPCR-DAG & $1.06\pm1.20^{\dagger\dagger\dagger}$ & $1.06\pm1.20^{\dagger\dagger\dagger}$ & $1.06\pm1.20^{\dagger\dagger\dagger}$\\
\midrule
Survey (6) & ABA-PC & $6.73\pm1.39$ & $\mathbf{2.56\pm1.47}^{*}$ & $10.3\pm2.70$\\
 & OptABA-PC & $\mathbf{3.77\pm0.91}^{***}$ & $3.18\pm0.92$ & $\mathbf{4.38\pm1.14}^{***}$\\
 & MPC & $\mathbf{3.17\pm1.49}^{\dagger}_{38}$ & $\mathbf{2.21\pm0.93}^{\dagger\dagger\dagger}_{38}$ & $4.11\pm2.10_{38}$\\
 & FGS & $3.47\pm1.28$ & $\mathbf{2.44\pm1.33}^{\dagger\dagger}$ & $4.52\pm1.53$\\
 & ASPCR-DAG & $4.16\pm1.23^{\dagger}$ & $4.16\pm1.23^{\dagger\dagger\dagger}$ & $4.16\pm1.23$\\
\midrule
Asia (8) & ABA-PC & $14.0\pm2.81$ & $8.02\pm2.27$ & $19.5\pm3.45$\\
 & OptABA-PC & $\mathbf{4.51\pm0.58}^{***}$ & $\mathbf{3.72\pm0.57}^{***}$ & $\mathbf{5.02\pm0.82}^{***}$\\
 & MPC & $\mathbf{3.92\pm0.25}^{\dagger\dagger\dagger}$ & $\mathbf{2.82\pm0.44}^{\dagger\dagger\dagger}$ & $5.02\pm0.14$\\
 & FGS & $6.01\pm1.70^{\dagger\dagger\dagger}$ & $5.16\pm1.92^{\dagger\dagger\dagger}$ & $6.76\pm1.61^{\dagger\dagger\dagger}$\\
\midrule
ER (5) & ABA-PC & $3.17\pm1.73$ & $1.34\pm1.67$ & $4.80\pm2.26$\\
 & OptABA-PC & $\mathbf{2.49\pm1.35}^{***}$ & $1.04\pm1.26$ & $\mathbf{3.76\pm1.78}^{***}$\\
 & MPC & $\mathbf{1.23\pm0.94}^{\dagger\dagger\dagger}$ & $\mathbf{0.46\pm0.93}^{\dagger}$ & $\mathbf{1.94\pm1.22}^{\dagger\dagger\dagger}$\\
 & FGS & $\mathbf{1.58\pm1.11}^{\dagger\dagger\dagger}$ & $0.72\pm1.20$ & $\mathbf{2.38\pm1.34}^{\dagger\dagger\dagger}$\\
 & ASPCR-DAG & $3.26\pm1.19^{\dagger\dagger}$ & $3.26\pm1.19^{\dagger\dagger\dagger}$ & $3.26\pm1.19$\\
\midrule
ER (8) & ABA-PC & $7.99\pm3.77$ & $3.20\pm2.49$ & $12.5\pm5.64$\\
 & OptABA-PC & $\mathbf{3.85\pm1.58}^{***}$ & $\mathbf{2.14\pm1.67}^{**}$ & $\mathbf{5.44\pm2.07}^{***}$\\
 & MPC & $\mathbf{2.38\pm1.21}^{\dagger\dagger\dagger}_{49}$ & $\mathbf{1.37\pm1.30}^{\dagger\dagger}_{49}$ & $\mathbf{3.33\pm1.46}^{\dagger\dagger\dagger}_{49}$\\
 & FGS & $\mathbf{2.07\pm1.41}^{\dagger\dagger\dagger}$ & $\mathbf{0.96\pm1.54}^{\dagger\dagger}$ & $\mathbf{3.08\pm1.58}^{\dagger\dagger\dagger}$\\
\midrule
SF (5) & ABA-PC & $3.07\pm1.94$ & $0.86\pm1.47$ & $4.88\pm3.11$\\
 & OptABA-PC & $\mathbf{2.52\pm1.89}^{**}$ & $\mathbf{0.56\pm1.28}^{*}$ & $\mathbf{4.08\pm3.10}^{**}$\\
 & MPC & $\mathbf{0.48\pm0.58}^{\dagger\dagger\dagger}$ & $\mathbf{0.18\pm0.44}^{\dagger}$ & $\mathbf{0.78\pm0.82}^{\dagger\dagger\dagger}$\\
 & FGS & $\mathbf{0.69\pm0.85}^{\dagger\dagger\dagger}$ & $0.38\pm0.81$ & $\mathbf{1.00\pm0.99}^{\dagger\dagger\dagger}$\\
 & ASPCR-DAG & $2.60\pm1.26$ & $2.60\pm1.26^{\dagger\dagger\dagger}$ & $\mathbf{2.60\pm1.26}^{\dagger}$\\
\midrule
SF (8) & ABA-PC & $8.14\pm3.46$ & $3.20\pm2.69$ & $12.7\pm4.83$\\
 & OptABA-PC & $\mathbf{4.76\pm1.97}^{***}$ & $2.98\pm1.73$ & $\mathbf{6.10\pm2.67}^{***}$\\
 & MPC & $\mathbf{1.88\pm1.23}^{\dagger\dagger\dagger}$ & $\mathbf{1.24\pm1.20}^{\dagger\dagger\dagger}$ & $\mathbf{2.52\pm1.40}^{\dagger\dagger\dagger}$\\
 & FGS & $\mathbf{2.82\pm1.88}^{\dagger\dagger\dagger}$ & $\mathbf{2.08\pm1.77}^{\dagger}$ & $\mathbf{3.56\pm2.10}^{\dagger\dagger\dagger}$\\
\bottomrule
\end{tabular*}
\end{table}

\begin{table}[t]
\caption{Average, best, and worst DAG F1 within each run of the matched primary experiments (mean $\pm$ standard deviation over seeds; higher is better). ABA-PC and OptABA-PC are scored over saved compatible DAGs; MPC and FGS over consistent extensions of each valid returned CPDAG; ASPCR-DAG on its single DAG, so its average, best, and worst coincide. Subscripts mark MPC's reduced samples: 38 on Survey and 49 on ER(8). Stars mark the significant ABA-PC--OptABA-PC winner. Daggers compare an external method with OptABA-PC; bold daggers favour that method, unbold daggers favour OptABA-PC. One, two, and three symbols denote BH-adjusted $p<0.05$, $0.01$, and $0.001$. Compatibility counts are in Table~\ref{tab:final-runtime-completion}.}
\label{tab:final-dag-f1-ranges}
\centering
\small
\setlength{\tabcolsep}{3pt}
\begin{tabular*}{\linewidth}{@{\extracolsep{\fill}}llccc@{}}
\toprule
Dataset & Method & F1 avg & F1 best & F1 worst\\
\midrule
Cancer (5) & ABA-PC & $0.483\pm0.223$ & $0.688\pm0.270$ & $0.346\pm0.264$\\
 & OptABA-PC & $\mathbf{0.563\pm0.262}^{*}$ & $0.616\pm0.277$ & $\mathbf{0.531\pm0.281}^{***}$\\
 & MPC & $0.623\pm0.222$ & $\mathbf{0.856\pm0.113}^{\dagger\dagger\dagger}$ & $0.462\pm0.326$\\
 & FGS & $\mathbf{0.739\pm0.215}^{\dagger\dagger}$ & $\mathbf{0.794\pm0.179}^{\dagger\dagger\dagger}$ & $\mathbf{0.691\pm0.282}^{\dagger}$\\
 & ASPCR-DAG & $0.462\pm0.225^{\dagger}$ & $0.462\pm0.225^{\dagger\dagger}$ & $0.462\pm0.225$\\
\midrule
Earthquake (5) & ABA-PC & $0.989\pm0.039$ & $0.991\pm0.030$ & $0.985\pm0.062$\\
 & OptABA-PC & $0.991\pm0.030$ & $0.991\pm0.030$ & $0.991\pm0.030$\\
 & MPC & $0.991\pm0.030$ & $0.991\pm0.030$ & $0.991\pm0.030$\\
 & FGS & $0.417\pm0.149^{\dagger\dagger\dagger}$ & $0.652\pm0.145^{\dagger\dagger\dagger}$ & $0.246\pm0.194^{\dagger\dagger\dagger}$\\
 & ASPCR-DAG & $0.871\pm0.176^{\dagger\dagger\dagger}$ & $0.871\pm0.176^{\dagger\dagger\dagger}$ & $0.871\pm0.176^{\dagger\dagger\dagger}$\\
\midrule
Survey (6) & ABA-PC & $0.340\pm0.103$ & $\mathbf{0.677\pm0.180}^{**}$ & $0.101\pm0.135$\\
 & OptABA-PC & $\mathbf{0.482\pm0.140}^{***}$ & $0.586\pm0.130$ & $\mathbf{0.373\pm0.204}^{***}$\\
 & MPC & $0.528\pm0.242_{38}$ & $\mathbf{0.720\pm0.134}^{\dagger\dagger\dagger}_{38}$ & $0.342\pm0.367_{38}$\\
 & FGS & $0.465\pm0.218$ & $\mathbf{0.660\pm0.236}^{\dagger}$ & $0.268\pm0.270^{\dagger}$\\
 & ASPCR-DAG & $0.385\pm0.220^{\dagger\dagger}$ & $0.385\pm0.220^{\dagger\dagger\dagger}$ & $0.385\pm0.220$\\
\midrule
Asia (8) & ABA-PC & $0.310\pm0.095$ & $0.565\pm0.105$ & $0.109\pm0.107$\\
 & OptABA-PC & $\mathbf{0.665\pm0.009}^{***}$ & $\mathbf{0.754\pm0.045}^{***}$ & $\mathbf{0.613\pm0.032}^{***}$\\
 & MPC & $0.620\pm0.013^{\dagger\dagger\dagger}$ & $\mathbf{0.786\pm0.037}^{\dagger\dagger\dagger}$ & $0.454\pm0.014^{\dagger\dagger\dagger}$\\
 & FGS & $0.610\pm0.105^{\dagger\dagger}$ & $0.696\pm0.132^{\dagger\dagger}$ & $0.533\pm0.094^{\dagger\dagger\dagger}$\\
\midrule
ER (5) & ABA-PC & $0.567\pm0.215$ & $0.817\pm0.229$ & $0.348\pm0.264$\\
 & OptABA-PC & $\mathbf{0.620\pm0.188}^{**}$ & $0.847\pm0.187$ & $\mathbf{0.417\pm0.248}^{**}$\\
 & MPC & $\mathbf{0.768\pm0.177}^{\dagger\dagger\dagger}$ & $\mathbf{0.924\pm0.170}^{\dagger}$ & $\mathbf{0.623\pm0.242}^{\dagger\dagger\dagger}$\\
 & FGS & $\mathbf{0.710\pm0.196}^{\dagger\dagger}$ & $0.886\pm0.205$ & $\mathbf{0.547\pm0.255}^{\dagger\dagger}$\\
 & ASPCR-DAG & $0.426\pm0.251^{\dagger\dagger\dagger}$ & $0.426\pm0.251^{\dagger\dagger\dagger}$ & $0.426\pm0.251$\\
\midrule
ER (8) & ABA-PC & $0.496\pm0.181$ & $0.773\pm0.167$ & $0.261\pm0.216$\\
 & OptABA-PC & $\mathbf{0.650\pm0.137}^{***}$ & $0.814\pm0.154$ & $\mathbf{0.494\pm0.177}^{***}$\\
 & MPC & $\mathbf{0.727\pm0.140}^{\dagger\dagger}_{49}$ & $0.858\pm0.150_{49}$ & $\mathbf{0.606\pm0.175}^{\dagger\dagger\dagger}_{49}$\\
 & FGS & $\mathbf{0.774\pm0.139}^{\dagger\dagger\dagger}$ & $\mathbf{0.916\pm0.146}^{\dagger\dagger\dagger}$ & $\mathbf{0.643\pm0.175}^{\dagger\dagger\dagger}$\\
\midrule
SF (5) & ABA-PC & $0.594\pm0.213$ & $0.879\pm0.201$ & $0.370\pm0.289$\\
 & OptABA-PC & $\mathbf{0.650\pm0.197}^{**}$ & $\mathbf{0.915\pm0.181}^{*}$ & $\mathbf{0.440\pm0.273}^{**}$\\
 & MPC & $\mathbf{0.893\pm0.120}^{\dagger\dagger\dagger}$ & $\mathbf{0.968\pm0.080}^{\dagger}$ & $\mathbf{0.818\pm0.184}^{\dagger\dagger\dagger}$\\
 & FGS & $\mathbf{0.861\pm0.143}^{\dagger\dagger\dagger}$ & $0.939\pm0.127$ & $\mathbf{0.784\pm0.193}^{\dagger\dagger\dagger}$\\
 & ASPCR-DAG & $0.448\pm0.283^{\dagger\dagger\dagger}$ & $0.448\pm0.283^{\dagger\dagger\dagger}$ & $0.448\pm0.283$\\
\midrule
SF (8) & ABA-PC & $0.481\pm0.166$ & $0.769\pm0.179$ & $0.235\pm0.197$\\
 & OptABA-PC & $\mathbf{0.601\pm0.146}^{***}$ & $0.733\pm0.153$ & $\mathbf{0.499\pm0.182}^{***}$\\
 & MPC & $\mathbf{0.773\pm0.143}^{\dagger\dagger\dagger}$ & $\mathbf{0.871\pm0.130}^{\dagger\dagger\dagger}$ & $\mathbf{0.674\pm0.181}^{\dagger\dagger\dagger}$\\
 & FGS & $0.664\pm0.241$ & $0.784\pm0.212$ & $0.542\pm0.291$\\
\bottomrule
\end{tabular*}
\end{table}

\begin{table}[t]
\caption{Timing, structural output size, and CI compatibility for the matched primary experiments (mean $\pm$ standard deviation over 50 seeds for each reported method--dataset pair). Returned CPDAGs and DAGs are structural output/extension counts. CI-compatible counts additionally require all retained post-repair facts for ABA-PC/OptABA-PC or every recorded native test for MPC/ASPCR-DAG; compatibility is undefined for score-based FGS (`--'). ABA-PC and OptABA-PC times include construction, solving, and compatible-set evaluation and are directly paired. MPC/FGS time the learner call, while ASPCR-DAG reports its native end-to-end R run; their values and paired daggers describe implementations with different timing scopes and are not a controlled cross-family speed ranking. Stars identify the significant ABA-PC--OptABA-PC winner. Bold daggers favour the external method and unbold daggers favour OptABA-PC. One, two, and three symbols denote BH-adjusted $p<0.05$, $0.01$, and $0.001$.}
\label{tab:final-runtime-completion}
\centering
\scriptsize
\setlength{\tabcolsep}{3pt}
\resizebox{\linewidth}{!}{%
\begin{tabular}{llccccc}
\toprule
Dataset & Method & Time (s) & Returned CPDAGs & Returned DAGs & CI-compat. CPDAGs & CI-compat. DAGs\\
\midrule
Cancer (5) & ABA-PC & $1.51\pm0.48$ & $15.1\pm28.8$ & $31.9\pm57.1$ & $15.1\pm28.8$ & $31.9\pm57.1$\\
 & OptABA-PC & $\mathbf{1.28\pm0.38}^{***}$ & $1.0\pm0.0$ & $1.5\pm1.1$ & $\mathbf{1.0\pm0.0}^{***}$ & $\mathbf{1.5\pm1.1}^{***}$\\
 & MPC & $\mathbf{0.01\pm0.00}^{\dagger\dagger\dagger}$ & $1.0\pm0.0$ & $3.3\pm1.6$ & $0.0\pm0.0$ & $0.0\pm0.0$\\
 & FGS & $\mathbf{0.08\pm0.09}^{\dagger\dagger\dagger}$ & $1.0\pm0.0$ & $1.8\pm1.5$ & -- & --\\
 & ASPCR-DAG & $12.8\pm0.09^{\dagger\dagger\dagger}$ & $1.0\pm0.0$ & $1.0\pm0.0$ & $0.0\pm0.0$ & $0.0\pm0.0$\\
\midrule
Earthquake (5) & ABA-PC & $1.83\pm0.83$ & $1.0\pm0.3$ & $1.1\pm0.5$ & $1.0\pm0.3$ & $1.1\pm0.5$\\
 & OptABA-PC & $\mathbf{1.68\pm0.81}^{***}$ & $1.0\pm0.0$ & $1.1\pm0.2$ & $1.0\pm0.0$ & $1.1\pm0.2$\\
 & MPC & $\mathbf{0.01\pm0.00}^{\dagger\dagger\dagger}$ & $1.0\pm0.0$ & $1.1\pm0.2$ & $0.9\pm0.2$ & $1.0\pm0.3$\\
 & FGS & $\mathbf{0.10\pm0.09}^{\dagger\dagger\dagger}$ & $1.0\pm0.0$ & $7.6\pm6.2$ & -- & --\\
 & ASPCR-DAG & $12.7\pm0.08^{\dagger\dagger\dagger}$ & $1.0\pm0.0$ & $1.0\pm0.0$ & $0.4\pm0.5$ & $0.4\pm0.5$\\
\midrule
Survey (6) & ABA-PC & $15.7\pm30.2$ & $1278.5\pm2807.2$ & $4725.3\pm9759.0$ & $1278.5\pm2807.2$ & $4725.3\pm9759.0$\\
 & OptABA-PC & $2.94\pm1.75$ & $1.2\pm0.7$ & $2.9\pm2.0$ & $\mathbf{1.2\pm0.7}^{***}$ & $\mathbf{2.9\pm2.0}^{***}$\\
 & MPC & $\mathbf{0.02\pm0.00}^{\dagger\dagger\dagger}$ & $0.8\pm0.4$ & $2.9\pm2.4$ & $0.0\pm0.0$ & $0.0\pm0.0$\\
 & FGS & $\mathbf{0.08\pm0.09}^{\dagger\dagger\dagger}$ & $1.0\pm0.0$ & $4.2\pm1.6$ & -- & --\\
 & ASPCR-DAG & $39.2\pm0.72^{\dagger\dagger\dagger}$ & $1.0\pm0.0$ & $1.0\pm0.0$ & $0.0\pm0.0$ & $0.0\pm0.0$\\
\midrule
Asia (8) & ABA-PC & $95.4\pm48.8$ & $10822.7\pm7112.3$ & $27191.5\pm16181.0$ & $10822.7\pm7112.3$ & $27191.5\pm16181.0$\\
 & OptABA-PC & $\mathbf{66.5\pm84.9}^{*}$ & $1.2\pm0.4$ & $4.6\pm3.8$ & $\mathbf{1.2\pm0.4}^{***}$ & $\mathbf{4.6\pm3.8}^{***}$\\
 & MPC & $\mathbf{0.04\pm0.01}^{\dagger\dagger\dagger}$ & $1.0\pm0.0$ & $4.8\pm1.6$ & $0.0\pm0.0$ & $0.0\pm0.0$\\
 & FGS & $\mathbf{0.08\pm0.08}^{\dagger\dagger\dagger}$ & $1.0\pm0.0$ & $4.8\pm1.6$ & -- & --\\
\midrule
ER (5) & ABA-PC & $1.01\pm0.15$ & $8.2\pm13.0$ & $32.1\pm61.1$ & $8.2\pm13.0$ & $32.1\pm61.1$\\
 & OptABA-PC & $\mathbf{0.90\pm0.08}^{***}$ & $3.0\pm3.4$ & $11.6\pm15.6$ & $\mathbf{3.0\pm3.4}^{***}$ & $\mathbf{11.6\pm15.6}^{**}$\\
 & MPC & $\mathbf{0.01\pm0.00}^{\dagger\dagger\dagger}$ & $1.0\pm0.0$ & $3.3\pm2.5$ & $0.2\pm0.4$ & $0.6\pm1.3$\\
 & FGS & $\mathbf{0.06\pm0.07}^{\dagger\dagger\dagger}$ & $1.0\pm0.0$ & $3.6\pm2.6$ & -- & --\\
 & ASPCR-DAG & $15.5\pm0.23^{\dagger\dagger\dagger}$ & $1.0\pm0.0$ & $1.0\pm0.0$ & $0.1\pm0.3$ & $0.1\pm0.3$\\
\midrule
ER (8) & ABA-PC & $\mathbf{48.0\pm54.3}^{***}$ & $3840.0\pm5377.9$ & $11539.7\pm15313.7$ & $3840.0\pm5377.9$ & $11539.7\pm15313.7$\\
 & OptABA-PC & $139.1\pm139.4$ & $4.7\pm9.5$ & $32.0\pm64.1$ & $\mathbf{4.7\pm9.5}^{***}$ & $\mathbf{32.0\pm64.1}^{***}$\\
 & MPC & $\mathbf{0.02\pm0.01}^{\dagger\dagger\dagger}$ & $1.0\pm0.1$ & $5.2\pm4.2$ & $0.1\pm0.2$ & $0.2\pm1.0$\\
 & FGS & $\mathbf{0.07\pm0.07}^{\dagger\dagger\dagger}$ & $1.0\pm0.0$ & $5.8\pm4.5$ & -- & --\\
\midrule
SF (5) & ABA-PC & $1.85\pm2.03$ & $88.4\pm202.7$ & $372.6\pm856.2$ & $88.4\pm202.7$ & $372.6\pm856.2$\\
 & OptABA-PC & $\mathbf{1.76\pm2.05}^{***}$ & $85.1\pm203.9$ & $357.6\pm861.0$ & $\mathbf{85.1\pm203.9}^{**}$ & $\mathbf{357.6\pm861.0}^{**}$\\
 & MPC & $\mathbf{0.01\pm0.00}^{\dagger\dagger\dagger}$ & $1.0\pm0.0$ & $1.7\pm0.9$ & $0.3\pm0.5$ & $0.5\pm0.9$\\
 & FGS & $\mathbf{0.06\pm0.08}^{\dagger\dagger\dagger}$ & $1.0\pm0.0$ & $1.7\pm1.1$ & -- & --\\
 & ASPCR-DAG & $15.4\pm0.25^{\dagger\dagger\dagger}$ & $1.0\pm0.0$ & $1.0\pm0.0$ & $0.1\pm0.3$ & $0.1\pm0.3$\\
\midrule
SF (8) & ABA-PC & $\mathbf{42.3\pm53.4}^{**}$ & $3108.1\pm4287.4$ & $8716.5\pm11926.2$ & $3108.1\pm4287.4$ & $8716.5\pm11926.2$\\
 & OptABA-PC & $99.1\pm122.0$ & $170.3\pm1023.2$ & $969.8\pm6017.6$ & $\mathbf{170.3\pm1023.2}^{***}$ & $\mathbf{969.8\pm6017.6}^{***}$\\
 & MPC & $\mathbf{0.02\pm0.01}^{\dagger\dagger\dagger}$ & $1.0\pm0.0$ & $2.8\pm1.4$ & $0.0\pm0.2$ & $0.1\pm0.5$\\
 & FGS & $\mathbf{0.06\pm0.07}^{\dagger\dagger\dagger}$ & $1.0\pm0.0$ & $3.3\pm1.8$ & -- & --\\
\bottomrule
\end{tabular}
}%
\end{table}

\fi

\end{document}